\documentclass{article}

\usepackage[final]{neurips_2024}

\usepackage[utf8]{inputenc} 
\usepackage[T1]{fontenc}    
\usepackage[colorlinks,
            linkcolor=red,
            anchorcolor=blue,
            citecolor=blue
            ]{hyperref}
\usepackage{smile}
\usepackage{amsmath}
\usepackage{amssymb}
\usepackage{mathtools}
\usepackage{amsthm}

\usepackage[capitalize,noabbrev]{cleveref}
\usepackage{enumitem}

\usepackage[protrusion=true,expansion=true]{microtype}
\usepackage{microtype}
\usepackage{graphicx}
\usepackage{subfigure}
\usepackage{booktabs} 
\usepackage{xcolor}

\newcommand{\cN}{N}

\newcommand{\highlight}{\textcolor{black}}

\newcommand{\qvalue}{Q}
\newcommand{\vvalue}{V}

\newcommand{\regret}{\mathrm{Regret}}

\newcommand{\var}{\mathrm{Var}}

\newcommand{\polylog}{\text{polylog}}

\usepackage{colortbl}
\definecolor{LightCyan}{rgb}{1.0, 0.95, 0.9}

\title{A Nearly Optimal and Low-Switching Algorithm for Reinforcement Learning with General Function Approximation}

\allowdisplaybreaks

\author{%
Heyang Zhao\\
Department of Computer Science \\
University of California, Los Angeles \\
Los Angeles, CA 90095\\
\texttt{hyzhao@cs.ucla.edu} \\
\AND
Jiafan He \\
Department of Computer Science \\
University of California, Los Angeles \\
Los Angeles, CA 90095\\
\texttt{jiafanhe19@ucla.edu} \\
\And
Quanquan Gu \\
Department of Computer Science \\
University of California, Los Angeles \\
Los Angeles, CA 90095\\
\texttt{qgu@cs.ucla.edu} \\
}

\begin{document}

\maketitle

\begin{abstract}
   The exploration-exploitation dilemma has been a central challenge in reinforcement learning (RL) with complex model classes. In this paper, we propose a new algorithm, Monotonic  Q-Learning with Upper Confidence Bound (MQL-UCB) for RL with general function approximation. Our key algorithmic design includes (1) a general deterministic policy-switching strategy that achieves low switching cost, (2) a monotonic value function structure with carefully controlled function class complexity, and (3) a variance-weighted regression scheme that exploits historical trajectories with high data efficiency. MQL-UCB achieves minimax optimal regret of $\tilde{O}(d\sqrt{HK})$ when $K$ is sufficiently large and near-optimal policy switching cost of $\tilde{O}(dH)$, with $d$ being the eluder dimension of the function class, $H$ being the planning horizon, and $K$ being the number of episodes. 
   Our work sheds light on designing provably sample-efficient and deployment-efficient Q-learning with nonlinear function approximation. 
\end{abstract}

\section{Introduction}

In reinforcement learning (RL), a learner interacts with an unknown environment and aims to maximize the cumulative reward. 
As one of the most mainstream paradigms for sequential decision-making, RL has extensive applications in many real-world problems \citep{kober2013reinforcement, mnih2015human, lillicrap2015continuous, zoph2016neural, zheng2018drn}. Theoretically, the RL problem is often formulated as a Markov Decision Process (MDP) \citep{puterman2014markov}. Achieving the optimal regret bound for various MDP settings has been a long-standing fundamental problem in RL research. In tabular MDPs where the state space $\cS$ and the action space $\cA$ are finite, the optimal regret bound has been well-established, ranging from the episodic setting \citep{azar2017minimax, zanette2019tighter, simchowitz2019non, zhang2020almost}, average-reward setting \citep{zhang2019regret} to the discounted setting \citep{he2021nearly}. Nevertheless, these regret guarantees are intolerably large in many real-world applications, where the state space $\cS$ and the action space $\cA$ are often large or even infinite. 

As is commonly applied in applications, \emph{function approximation} has been widely studied by theorists to demonstrate the generalization across large state-action spaces, proving the performance guarantees of various RL algorithms for specific function classes. There are recent works on RL with linear function approximation under different assumptions such as linear MDPs \citep{yang2019sample, jin2020provably}, linear mixture MDPs \citep{modi2020sample, ayoub2020model, zhou2021nearly}. Among them, \citet{zhou2021nearly} achieved nearly optimal regret bounds for linear mixture MDPs through a model-based approach adopting variance-weighted linear regression. Later, \citet{hu2022nearly} proposed LSVI-UCB+ algorithm, making an attempt to improve the regret for linear MDP through an over-optimistic value function approach. However, their analysis was later discovered to suffer from a technical issue \citep{agarwal2022vo,he2022nearly}. To fix this issue, \citet{agarwal2022vo} introduced similar over-optimistic value functions to construct a monotonic variance estimator and a non-Markov policy, achieving the first statistically optimal regret for linear MDPs. They also proposed a novel algorithm dubbed VO$Q$L for RL with general function approximation. Concurrently, \citet{he2022nearly} proposed LSVI-UCB++, which takes a different approach and employs a rare-switching technique to obtain the optimal regret. In a parallel line of research, there has been a growing body of literature proposing more general frameworks to unify sample efficient RL algorithms, e.g., Bellman rank \citep{jiang2017contextual}, Witness rank \citep{sun2019model}, eluder dimension \citep{russo2013eluder}, Bellman eluder dimension \citep{jin2021bellman}, Bilinear Classes \citep{du2021bilinear}, Decision-Estimation Coefficient \citep{foster2021statistical}, Admissible Bellman Characterization \citep{chen2022general} and Decoupling Coefficient \citep{ agarwal2022model, agarwal2022non}. However, when applying these frameworks to linear MDPs, none of them can achieve the minimax optimal regret. To our knowledge, \citet{agarwal2022vo} is the only algorithmic framework achieving optimal regret for RL beyond linear function approximation. Since VO$Q$L requires a non-Markov planning procedure, where the resulting policy does not act greedily with respect to a single optimistic value function, it is natural to ask 
\begin{center}
\emph{Can we develop a RL algorithm with Markov policy\footnote{A Markov policy means that the action chosen by the policy only depends on the current state instead of the prefix trajectory. It is more aligned with the empirical RL approaches, since the estimated value function is not well-defined under non-Markov policy.} to solve MDPs with general function approximation and  achieve the optimal regret? }
\end{center} 
While the sample efficiency of RL algorithms for MDPs with nonlinear function classes has been comprehensively researched, 
\emph{deployment efficiency} \citep{matsushima2020deployment} is also a major concern in many real-world application scenarios. For example, in recommendation systems \citep{afsar2022reinforcement}, it may take several weeks to deploy a new recommendation policy. On the other hand, the system is capable of collecting an enormous amount of data every minute implementing a fixed policy. As a result, it is computationally inefficient to change the executed policy after each data point is collected as is demanded by most of the online RL algorithms in theoretical studies. To resolve this issue, \citet{bai2019provably} first introduced the concept of switching cost, defined as the number of policy updates. Following this concept, a series of RL algorithms have been proposed on the theoretical side with low switching cost guarantees \citep[e.g.,][]{zhang2020almost, wang2021provably, qiao2022sample, kong2021online, velegkas2022reinforcement, li2023low}. \citet{xiong2023general} considered low-switching RL with general function approximation, which achieves $\tilde{O}(dH)$ switching cost. However, their algorithm has an intractable planning phase and is not statistically optimal. \citet{kong2021online}, \citet{velegkas2022reinforcement}, and \citet{li2023low} considered RL with general function approximation, all of which achieved the switching cost of $O(d^2 H \polylog(K))$, where $d$ is the eluder dimension of the underlying function class, $H$ is the planning horizon, and $K$ is the number of episodes. \highlight{In contrast, \citet{gao2021provably} proved an $\Omega(dH /\log d)$ lower bound of switching cost for any deterministic algorithms in learning linear MDPs. }
Therefore, the following question remains open: \begin{center} \emph{Can we design an efficient algorithm with $\tilde{O}(dH)$ switching cost for MDPs with bounded eluder dimension?}
\end{center}

In this paper, we answer the above two questions simultaneously by proposing a novel algorithm \emph{Monotonic Q-Learning with UCB (MQL-UCB)} with all the aforementioned appealing properties. At the core of our algorithmic design are the following innovative techniques: 
\begin{itemize}[leftmargin = *]
   \item We propose a novel policy-switching strategy based on the cumulative sensitivity of historical data. To the best of our knowledge, this is the first computationally-tractable deterministic rare-switching algorithm for RL with general function approximation which achieves $\tilde{O}(dH)$ switching cost. \highlight{We also prove a nearly matching lower bound for any algorithm with
arbitrary policies including both deterministic and stochastic policies (See Theorem~\ref{lemma:lower}).} Previous approaches for low switching cost in RL with general function approximation are sampling-based \citep{kong2021online, velegkas2022reinforcement, li2023low} and highly coupled with a sub-sampling technique used for regression, making it less efficient. When restricted to the linear case, sampling-based rare-switching has a worse switching cost of $\tilde{O}(d^2 H)$ (Section 3.3 in \citealt{kong2021online}; Theorem~3.3 in \citealt{velegkas2022reinforcement}; Section C in \citealt{li2023low}) than that in \citet{wang2021provably}.
   \item With the novel policy-switching scheme, we illustrate how to reduce the complexity of value function classes while maintaining a series of monotonic value functions, strictly generalizing the LSVI-UCB++ algorithm \citep{he2022nearly} to general function class with bounded eluder dimension. Based on the structure of the value functions, we further demonstrate how MQL-UCB achieves a variance-dependent regret bound that can gracefully reduce to a nearly constant regret for deterministic MDPs. While in the worst case, our algorithm still attains a nearly minimax optimal regret guarantee. \highlight{Our work is the first work for RL with general function approximation that achieves the nearly minimax optimal regret when specialized to linear MDPs, while still enjoys simple Markov planning phase.}
\end{itemize}
\highlight{It is worth noting that recently, \citet{qiao2023logarithmic} also considered RL with low-switching cost beyond linear function approximation, i.e., MDPs with low inherent Bellman error \citep{zanette2020learning} and generalized linear MDPs \citep{wang2020optimism}. Their approach can be seen as a slight extension of RL with low-switching cost for linear MDPs, since in both settings, the covariance matrix still exists and they can still use the determinant of the covariance as a criterion for policy switching.}
\newcolumntype{g}{>{\columncolor{LightCyan}}c}
\begin{table}[t!]
\caption{A comparison of existing algorithms in terms of regret and switching cost for linear MDP and general function class with bounded eluder dimension and Bellman completeness. The results hold for in-homogeneous episodic RL with horizon length $H$, number of episodes $K$ where the total reward obtained in an episode is not larger than 1. For regret, we only present the leading term when $K$ is large enough compared to other variables and \highlight{ hide poly-logarithmic factors in $K$, $d$ or $\dim$, $H$ and the constant}. For linear MDPs, $d$ is the dimension of the feature vectors. For general function class, $\dim$ is a shorthand of the eluder dimension of the underlying function class, $\cN$ is the covering number of the value function class, and $\cN_{\cS, \cA}$ is the covering number of the state-action space. 
}\label{table:1}
\centering 
\renewcommand{\arraystretch}{1.0}
\resizebox*{0.99\columnwidth}{!}{
\begin{tabular}{ggggg}
\toprule 
\rowcolor{white} Algorithm & Regret & \# of Switches  & Model Class  \\
\midrule  
\rowcolor{white} LSVI-UCB & & & \\
\rowcolor{white} \citep{jin2020provably}  & \multirow{-2}{*}{$d^{3 / 2} H \sqrt{K}$} & \multirow{-2}{*}{$K$}  & \multirow{-2}{*}{}\\
\rowcolor{white} LSVI-UCB-RareSwitch & & & \\
\rowcolor{white} \citep{wang2021provably}  & \multirow{-2}{*}{$d^{3 / 2} H \sqrt{K}$} & \multirow{-2}{*}{$\tilde{O}(dH)$}  & \multirow{-2}{*}{}\\
\rowcolor{white} LSVI-UCB++ & & & \\
\rowcolor{white} \citep{he2022nearly}  & \multirow{-2}{*}{$d\sqrt{ H K}$} & \multirow{-2}{*}{$\tilde{O}(dH)$}  & \multirow{-6}{*}{Linear MDPs}\\
\hline
\rowcolor{white} $\cF$-LSVI &  & & \\
\rowcolor{white} \citep{wang2020reinforcement} & \multirow{-2}{*}{$\dim(\cF) \sqrt{\log \cN \log \cN_{\cS, \cA}} \cdot H \sqrt{K}$} & \multirow{-2}{*}{$K$} \\

\rowcolor{white} GOLF &  & & \\
\rowcolor{white} \citep{jin2021bellman} & \multirow{-2}{*}{$\sqrt{\dim(\cF) \log \cN} \cdot H \sqrt{K}$} & \multirow{-2}{*}{$K$} & Bounded eluder dimension\\

\rowcolor{white} VOQL &  & & + Completeness\\
\rowcolor{white} \citep{agarwal2022vo} & \multirow{-2}{*}{$\sqrt{\dim(\cF) \log \cN \cdot H K}$} & \multirow{-2}{*}{$\tilde O(\dim(\cF)^2H)$} \\
MQL-UCB &  &  & \\
\small{(Theorem \ref{thm:main})}& \multirow{-2}{*}{$\sqrt{\dim(\cF) \log \cN \cdot H K}$}& \multirow{-2}{*}{$\tilde O(\dim(\cF)H)$}  & \\
\bottomrule 
\end{tabular}}
\end{table}
\paragraph{Notation.} 
We use lower case letters to denote scalars and use lower and upper case bold face letters to denote vectors and matrices respectively. We denote by $[n]$ the set $\{1,\dots, n\}$. For two positive sequences $\{a_n\}$ and $\{b_n\}$ with $n=1,2,\dots$, we write $a_n=O(b_n)$ if there exists an absolute constant $C>0$ such that $a_n\leq Cb_n$ holds for all $n\ge 1$ and write $a_n=\Omega(b_n)$ if there exists an absolute constant $C>0$ such that $a_n\geq Cb_n$ holds for all $n\ge 1$. We use $\tilde O(\cdot)$ to further hide the polylogarithmic factors except log-covering numbers. We use $\ind\{\cdot\}$ to denote the indicator function.



\section{Preliminaries} \label{section:pre}

\subsection{Time-Inhomogeneous Episodic MDP}

We consider a time-inhomogeneous episodic Markov Decision Process (MDP), denoted by a tuple $\cM = M(\cS, \cA, H, \{\PP_h\}_{h = 1}^H, \{r_h\}_{h=1}^H)$. Here, $\cS$ and $\cA$ are the spaces of state and action, respectively, $H$ is the length of each episode, $\PP_h: \cS \times \cA \times \cS \to [0, 1]$ is the transition probability function at stage $h$ which denotes the probability for state $s$ to transfer to the next state $s'$ with current action $a$, and $r_h: \cS \times \cA \to [0, 1]$ is the deterministic reward function at stage $h$. A policy $\pi := \{\pi_h\}_{h=1}^{H}$ is a collection of mappings $\pi_h$ from an observed state $s\in \cS$ to the simplex of action space $\cA$. For any policy $\pi=\{\pi_h\}_{h=1}^{H}$ and stage $h\in[H]$, we define the value function $\vvalue_h^{\pi}(s)$ and the action-value function $\qvalue_h^{\pi}(s,a)$ as follows:
\begin{align}
\qvalue^{\pi}_h(s,a) &=r_h(s,a) + \EE\bigg[\sum_{h'=h+1}^H r_{h'}\big(s_{h'},
\pi_{h'}(s_{h'})\big)\bigg| s_h=s,a_h=a\bigg],  
\vvalue_h^{\pi}(s) = \qvalue_h^{\pi}(s, \pi_h(s)),\notag
\end{align}
where $s_{h'+1}\sim \PP_{h'}(\cdot|s_{h'},a_{h'})$.
Then, we further define the optimal value function $V_h^*$ and the optimal action-value function $\qvalue_h^*$ as $V_h^*(s) = \max_{\pi}\vvalue_h^{\pi}(s)$ and $\qvalue_h^*(s,a) = \max_{\pi}\qvalue_h^{\pi}(s,a)$. For simplicity, we assume the total reward for each possible trajectory $(s_1,a_1,...,s_H,a_H)$ satisfies $\sum_{h=1}^H r_h(s_h,a_h) \le 1$. Under this assumption, the value function $V_h^\pi(\cdot)$ and $Q_h^\pi(\cdot,\cdot)$ are bounded in $[0,1]$. For any function $V: \cS \to \RR$ and stage $h \in [H]$, we define the following first-order Bellman operator $\cT_h$ and second-order Bellman operator $\cT_h^2$ on function $V$: 
\begin{align*} 
   &\cT_h V(s_h, a_h) = \EE_{ s_{h + 1}}\big[r_h + V(s_{h + 1})|s_h, a_h\big], \cT_h^2 V(s_h, a_h) = \EE_{ s_{h + 1}}\Big[\big(r_h + V(s_{h + 1})\big)^2|s_h, a_h\Big],
\end{align*}
where $s_{h+1}\sim \PP_h(\cdot|s_h,a_h)$ and $r_h=r_h(s_h,a_h)$. For simplicity, we further define  $[\PP_h \vvalue](s,a)=\EE_{s' \sim \PP_h(\cdot|s,a)}\vvalue(s')$ and $[\VV_h \vvalue](s,a)=\cT_h^2 V(s_h, a_h)- \big(\cT_h V(s_h, a_h)\big)^2$.
Using this notation, for each stage $h\in[H]$, the Bellman equation and Bellman optimality equation take the following forms:
\begin{align}
   \qvalue_h^{\pi}(s,a) = \cT_h\vvalue_{h+1}^{\pi}(s,a), 
\quad \qvalue^{*}(s,a) = \cT_h \vvalue_{h+1}^{*}(s,a), \notag
\end{align}
where $\vvalue^{\pi}_{H+1}(\cdot)=\vvalue^{*}_{H+1}(\cdot)=0$. At the beginning of each episode $k\in[K]$, the agent selects a policy $\pi^k$ to be executed throughout the episode, and an initial state $s^k_1$ is arbitrarily selected by the environment. For each stage $h\in[H]$, the agent first observes the current state $s_h^k$, chooses an action following the policy $\pi_h^k$, then transits to the next state with $s_{h+1}^k \sim \PP_h(\cdot|s_h^k,a_h^k)$ and reward $r_h(s_h,a_h)$. Based on the protocol, we defined the suboptimality gap in episode $k$ as the difference between the value function for selected policy $\pi^k$ and the optimal value function $V_1^*(s_1^k) - V_1^{\pi^k}(s_1^k).$
Based on these definitions, we can define the regret in the first $K$ episodes as follows: 
\begin{definition}\label{def:Regret} 
For any RL algorithm $\mathtt{Alg}$, the regret in the first $K$ episodes is denoted by the sum of the suboptimality for episode $k = 1,\ldots, K$, 
\begin{align}
   \text{Regret}(K) = \sum_{k=1}^K \vvalue_1^*(s_1^k) - \vvalue_1^{\pi^k}(s_1^k),\notag
\end{align}
where $\pi^k$ is the agent's policy at episode $k$.
\end{definition}

\subsection{Function Classes and Covering Numbers}

\begin{assumption}[Completeness] \label{assumption:complete}
   Given $\cF := \{\cF_h\}_{h = 1}^H $ which is composed of bounded functions $f_h: \cS \times \cA \to [0, L]$. We assume that for any function $V: \cS \to [0, 1]$ there exists $f_1, f_2 \in \cF_h$ such that for any $(s, a) \in \cS \times \cA$, \begin{align*} 
       \EE_{s' \sim \PP_h(\cdot|s, a)}\big[r_h(s, a) + V(s')\big] = f_1(s, a),\ \text{and}\ 
       \EE_{s' \sim \PP_h(\cdot|s, a)}\Big[\big(r_h(s, a) + V(s')\big)^2\Big] = f_2(s, a).
   \end{align*}
   We assume that $L = O(1)$ throughout the paper. 
\end{assumption}
\begin{remark}
 Completeness is a fundamental assumption in RL with general function approximation, as recognized in \citet{wang2020optimism,jin2021bellman,agarwal2022vo}. Our assumption is the same as that in \citet{agarwal2022vo} and is slightly stronger than that in \citet{wang2020reinforcement} and \citet{jin2021bellman}.  More specifically, in \citet{wang2020reinforcement}, completeness is only required for the first-order Bellman operator. In contrast, we necessitate completeness with respect to the second-order Bellman operator, which becomes imperative during the construction of variance-based weights. \citet{jin2021bellman} only requires the completeness for the function class $\cF_{h+1}$ ($\cT_h \cF_{h+1}\subseteq \cF_h$). However, the GOLF algorithm \citep{jin2021bellman} requires solving an intricate optimization problem across the entire episode. In contrast, we employ pointwise exploration bonuses as an alternative strategy, which requires the completeness for function class $\cV=\{V: \cS \to [0, 1]\} $ , i.e., $\cT_h \cV\subseteq \cF_h$. 
 \highlight{The completeness assumption on the second moment is first introduced by \citet{agarwal2022vo}, and is crucial for obtaining a tighter regret bound. More specifically, making use of the variance of the value function at the next state is known to be crucial to achieve minimax-optimal regret bound in RL ranging from tabular MDPs \citep{azar2017minimax} to linear mixture MDPs \citep{zhou2021nearly} and linear MDPs \citep{he2022nearly}. In RL with general function approximation, the second-moment compleness assumption makes the variance of value functions computationally tractable.} 
\end{remark}

\begin{definition}[Generalized Eluder dimension, \citealt{agarwal2022vo}] \label{def:ged}
   Let $\lambda \ge 0$, a sequence of state-action pairs $\Zb = \{z_i\}_{i \in [T]}$ and a sequence of positive numbers $\bsigma = \{\sigma_i\}_{i \in [T]}$. The generalized Eluder dimension of a function class $\cF: \cS \times \cA \to [0, L]$ with respect to $\lambda$ is defined by $\dim_{\alpha, T}(\cF) := \sup_{\Zb, \bsigma:|Z| = T, \bsigma \ge \alpha} \dim(\cF, \Zb, \bsigma)$, \begin{small} \begin{align*} 
        &\dim(\cF, \Zb, \bsigma):= \sum_{i=1}^T \min \bigg(1, \frac{1}{\sigma_i^2} D_{\cF}^2(z_i; z_{[i - 1]}, \sigma_{[i - 1]})\bigg), \\
       &D_\cF^2(z; z_{[t - 1]}, \sigma_{[t - 1]}) := \sup_{f_1, f_2 \in \cF} \frac{(f_1(z) - f_2(z))^2}{\sum_{s \in [t - 1]} \frac{1}{\sigma_s^2}(f_1(z_s) - f_2(z_s))^2 + \lambda}.
   \end{align*} \end{small}
   We write $\dim_{\alpha, T}(\cF) := H^{-1} \cdot \sum_{h \in [H]}\dim_{\alpha, T}(\cF_h)$ for short when $\cF$ is a collection of function classes $\cF = \{\cF_h\}_{h = 1}^H$ in the context.
\end{definition}

\begin{remark} 
    The $D_\cF^2$ quantity has been introduced in \cite{agarwal2022vo} and \cite{ye2023corruption} to quantify the uncertainty of a state-action pair given a collected dataset with corresponding weights. It was inspired by \cite{gentile2022achieving} where an unweighted version of uncertainty has been defined for active learning. Prior to that, \cite{wang2020reinforcement} introduced a similar `sensitivity' measure to determine the sampling probability in their sub-sampling framework. As discussed in \cite{agarwal2022vo}, when specialized to linear function classes, $D_\cF^2(z_t; z_{[t - 1]}, \sigma_{[t - 1]})$ can be written as the elliptical norm $\|z_t\|_{\bSigma_{t - 1}^{-1}}^2$, where $\bSigma_{t - 1}$ is the weighted covariance matrix of the feature vectors $z_{[t - 1]}$. 
\end{remark}

\begin{definition}[Bonus oracle $\bar D_\cF^2$] \label{def:bonus-oracle}
In this paper, the bonus oracle is denoted by $\bar D_{\cF}^2$, which computes the estimated uncertainty of a state-action pair $z = (s, a) \in \cS \times \cA$ with respect to historical data $z_{[t - 1]}$ and corresponding weights $\sigma_{[t - 1]}$. In detail, we assume that a computable function $\bar D_{\cF}^2(z; z_{[t - 1]}, \sigma_{[t - 1]})$ satisfies 
$\frac{\bar D_{\cF}(z; z_{[t - 1]}, \sigma_{[t - 1]})}{D_{\cF}(z; z_{[t - 1]}, \sigma_{[t - 1]})} \in [1, C]$
, where $C$ is a fixed constant. 
\end{definition}

\begin{remark} 
\cite{agarwal2022vo} also assumed access to such a bonus oracle defined in Definition~\ref{def:bonus-oracle}, where they assume that the bonus oracle finds a proper bonus from a finite bonus class (Definition 3, \citealt{agarwal2022vo}). Our definition is slightly different in the sense that the bonus class is not assumed to be finite but with a finite covering number. Previous works by \cite{kong2021online} and \cite{wang2020reinforcement} proposed a sub-sampling idea to compute such a bonus function efficiently in general cases, which is also applicable in our framework. In a similar nonlinear RL setting, \cite{ye2023corruption} assumed that the uncertainly $D_\cF^2$ can be directly computed, which is a slightly stronger assumption. But essentially, these differences in bonus assumption only lightly affect the algorithm structure. 
\end{remark}

\begin{definition}[Covering numbers of function classes] \label{def:oracle}
   For any $\epsilon > 0$, we define the following covering numbers of the involved function classes: 
   \begin{enumerate}[leftmargin = *]
       \item For each $h \in [H]$, there exists an $\epsilon$-cover $\cC(\cF_h, \epsilon) \subseteq \cF_h$ with size $|\cC(\cF_h, \epsilon)| \le \cN(\cF_h, \epsilon)$, such that for any $f \in \cF$, there exists $f' \in \cC(\cF_h, \epsilon)$, such that $\|f - f'\|_\infty \le \epsilon$. For any $\epsilon > 0$, we define the uniform covering number of $\cF$ with respect to $\epsilon$ as $\cN_\cF(\epsilon) := \max_{h \in [H]} \cN(\cF_h, \epsilon)$. 
       \item There exists a bonus class $\cB: \cS \times \cA \to \RR$ such that for any $t \ge 0$, $z_{[t]} \in (\cS \times \cA)^t$, $\sigma_{[t]} \in \RR^t$, the oracle defined in Definition \ref{def:bonus-oracle} $\bar D_\cF(\cdot; z_{[t]}, \sigma_{[t]})$ is in $\cB$. 
       \item For bonus class $\cB$, there exists an $\epsilon$-cover $\cC(\cB, \epsilon) \subseteq \cB$ with size $|\cC(\cB, \epsilon)| \le \cN(\cB, \epsilon)$, such that for any $b \in \cB$, there exists $b' \in \cC(\cB, \epsilon)$, such that $\|b - b'\|_\infty \le \epsilon$.
   \end{enumerate}
\end{definition}
\begin{remark}
    In general function approximation, it is common to introduce the additional assumption on the covering number of bonus function classes. For example, in \citet{ye2023corruption}, \citet{agarwal2022model}, and \citet{di2023pessimistic}, the covering number of the bonus function class is  bounded. 
    
\end{remark}

\section{Algorithm and Key Techniques}

In this section, we will introduce our new algorithm, MQL-UCB. The detailed algorithm is provided in Algorithm \ref{algorithm1}. Our algorithm's foundational framework follows the Upper Confidence Bound (UCB) approach. In detail, for each episode $k\in [K]$, we construct an optimistic value function $Q_{k,h}(s,a)$ during the planning phase. Subsequently, in the \highlight{exploration} phase, the agent interacts with the environment, employing a greedy policy based on the current optimistic value function $Q_{k,h}(s,a)$. Once the agent obtains the reward \highlight{$r_h^k$} and transitions to the next state $s_{h+1}^k$, these outcomes are incorporated into the dataset, contributing to the subsequent planning phase. We will proceed to elucidate the essential components of our method.


\subsection{Rare Policy Switching}
For MQL-UCB algorithm, the value functions $Q_{k,h},\check{Q}_{k,h}$, along with their corresponding policy $\pi_k$, undergo updates when the agent collects a sufficient number of trajectories within the dataset that could significantly diminish the uncertainty associated with the Bellman operator $\cT_h V(\cdot,\cdot)$ through the weighted regression. In the context of linear bandits \citep{abbasi2011improved} or linear MDPs \citep{he2022nearly}, the uncertainty pertaining to the least-square regression is quantified by the covariance matrix $\bSigma_k$. In this scenario, the agent adjusts its policy once the determinant of the covariance matrix doubles, employing a determinant-based criterion. Nevertheless, in the general function approximation setting, such a method is not applicable in the absence of the covariance matrix which serves as a feature extractor in the linear setting. Circumventing this issue, \cite{kong2021online} proposed a sub-sampling-based method to achieve low-switching properties in nonlinear RL. Their subsampling technique is inspired by \cite{wang2021provably}, which showed that one only needs to maintain a small subset of historical data to obtain a sufficiently accurate least-square estimator. Such a subset can be generated sequentially according to the \emph{sensitivity} of a new coming data point. However, their approach leads to a switching cost of $\tilde O(d^2 H)$, which does not match the lower bound in linear MDPs proposed by \cite{gao2021provably}. 

To resolve this issue, we proposed a more general deterministic policy-updating framework for nonlinear RL. In detail, we use $\bar D_{\cF_h}^2(z_{i, h}; z_{[k_{last} - 1], h}, \bar \sigma_{[k_{last} - 1], h})$ to evaluate the information collected at the episode $i$, given the last updating $k_{last}$. Once the collected information goes beyond a threshold $\chi$ from last updating, i.e., \begin{small}
\begin{align}
       \sum_{i = k_{last}}^{k - 1} \frac{1}{\bar \sigma_{i, h}^2} \bar D_{\cF_h}^2(z_{i, h}; z_{[k_{last} - 1], h}, \bar \sigma_{[k_{last} - 1], h}) \ge \chi . \label{eq:def:switch}
   \end{align} \end{small}
the agent will perform updates on both the optimistic estimated value function and the pessimistic value function. Utilizing the $D^2_{\cF_h}$-based criterion, \highlight{we will show that the number of policy updates can be bounded by $O(H\cdot \dim_{\alpha, K}(\cF))$}. \highlight{This further reduces the complexity of the optimistic value function class and removes additional factors from a uniform convergence argument over the function class.} Specifically, we showcase under our rare-switching framework, the $\epsilon$-covering number of the optimistic and the pessimistic value function class at episode $k$ is bounded by \begin{align} 
\cN_\epsilon(k) := [\cN_{\cF}(\epsilon / 2) \cdot \cN(\cB, \epsilon / 2\hat \beta_K)]^{l_k + 1},
\end{align} \highlight{where $\hat \beta_K$ is the maximum confidence radius as shown in Algorithm \ref{algorithm1}, which will be specified in Lemmas \ref{lemma:opt-cover} and \ref{lemma:pes-cover}, {$l_k$ is the number of policy switches before the end of the $k$-th episode according to Algorithm \ref{algorithm1}. }}

\subsection{Weighted Regression}
\highlight{The estimation of $Q$ function in MQL-UCB extends LSVI-UCB proposed by \citet{jin2020provably} to general function classes. While the estimators in LSVI-UCB are computed from the classic least squares regression, }we construct the estimated value function $\hat f_{k, h}$ for general function classes by solving the following weighted regression: \begin{small}
\begin{align*}
    \hat f_{k, h} = \argmin\limits_{f_h \in \cF_h} \sum\limits_{i \in [k - 1]} \frac{1}{\bar \sigma_{i, h}^2} (f_h(s_h^i, a_h^i) - r_h^i - V_{k, h + 1}(s_{h + 1}^i))^2.
\end{align*} \end{small}
In the weighted regression, we set the weight $\bar \sigma_{k,h}$ as
\begin{align*}
    \bar\sigma_{k,h}= \max\big\{\sigma_{k,h}, \alpha,   \gamma \cdot \bar D^{1/2}_{\cF_h}(z; z_{[k - 1],h}, \bar\sigma_{[k - 1],h}) \big\},
\end{align*}
where $\sigma_{k,h}$ is the estimated variance for the stochastic transition process, $\bar D_{\cF_h}(z; z_{[k - 1],h}, \bar\sigma_{[k - 1],h})$ \highlight{denotes the uncertainty of the estimated function $\hat{f}_{k,h}$ conditioned on the historical observations} and \begin{small}\begin{align}
   &\gamma^2 := \log \Big(\Big(2HK^2 \big(2\log (L^2k /\alpha^4)+2\big)\cdot \big(\log(4L/\alpha^2)+2\big) \cdot \cN_{\cF}^4(\epsilon) \cdot \cN_{\epsilon}^2(K)\Big)\big/\delta\Big) \label{eq:def:gamma}
\end{align} \end{small} is used to properly balance the uncertainty across \highlight{various state-action pairs}. It is worth noting that \citet{ye2023corruption} also introduced the uncertainty-aware variance in the general function approximation with a distinct intention to deal with the adversarial corruption from the attacker. In addition, according to the weighted regression, with high probability, the Bellman operator $\cT_h V_{k,h}$ \highlight{satisfies:} \begin{small}
\begin{align*}
   \lambda +  {\sum_{i \in [k - 1]}} \bar \sigma_{i, h}^{-2}\left(\hat f_{k, h}(s_h^i, a_h^i) - \cT_h V_{k, h + 1}(s_h^i, a_h^i)\right)^2 \le  \hat \beta_{k}^2, 
\end{align*} \end{small}
{where $\hat \beta_k$ is the exploration radius for the $Q$ functions, which will be specified in Theorem \ref{thm:main}. }
According to the definition of Generalized Eluder dimension, the estimation error between the estimated function  $\hat f _{k,h}$ and the Bellman operator is upper bounded by:
\begin{small}
\begin{align*}
     &\big|\hat{f}_{k,h}(s,a)-\cT_h V_{k,h+1}(s,a)\big|\leq \hat{\beta}_k D_{\cF_h}(z; z_{[k - 1],h}, \bar\sigma_{[k - 1],h}).
\end{align*}\end{small}
Therefore, we introduce the exploration bonus $b_{k,h}$ and construct the optimistic value function $Q_{k,h}(s,a)$,i.e.,
$Q_{k,h}(s,a)\approx \hat f_{k,h}(s,a)+b_{k,h}(s,a), $ where $b_{k, h}(s, a) = \hat\beta_k \cdot \bar D_\cF\big((s, a); z_{[k - 1], h}, \sigma_{[k - 1], h}\big)$. 
Inspired by \citet{hu2022nearly,he2022nearly,agarwal2022vo}, in order to estimate the gap between the optimistic value function $V_{k,h}(s)$ and the optimal value function $V_{h}^*(s)$, we further construct the pessimistic estimator $\check f_{k,h}$ by the following weighted regression \begin{small}
\begin{align*}
    \check f_{k, h} =  \argmin\limits_{f_h \in \cF_h} \sum\limits_{i \in [k - 1]} \frac{1}{\bar \sigma_{i, h}^2} (f_h(s_h^i, a_h^i) - r_h^i - \check V_{k, h + 1}(s_{h + 1}^i))^2,
\end{align*} \end{small}
and introduce a negative exploration bonus when generating the pessimistic estimator.
   $\check{Q}_{k,h}(s,a)\approx \check f_{k,h}(s,a)-\check b_{k,h}(s,a)$, where $\check b_{k, h}(s, a) = \check\beta_k \cdot \bar D_\cF((s, a); z_{[k - 1], h}, \sigma_{[k - 1], h}$. 
Different from \citet{agarwal2022vo}, the pessimistic value function $\check f_{k,h}$ is computed from a similar weighted-regression scheme as in the case of the optimistic estimator, leading to a tighter confidence set.

\begin{algorithm*}[h!]
   \caption{Monotonic Q-Learning with UCB (MQL-UCB)}
   \begin{algorithmic}[1]\label{algorithm1}
  \REQUIRE Regularization parameter $\lambda$, confidence radius $\{\tilde \beta_k\}_{k \in [K]}, \{\hat \beta_k\}_{k \in [K]}$ and $\{\check \beta_k\}_{k \in [K]}$. 
  \STATE Initialize $k_\text{last}=0$. For each stage $h\in[H]$ and state-action $(s,a)\in \cS \times \cA$, set $ Q_{0,h}(s,a)\leftarrow H, \check{Q}_{0,h}(s,a)\leftarrow 0$. 
\FOR{episodes $k=1,\ldots,K$}
\STATE Received the initial state $s_1^k$. 
   \FOR{stage $h=H,\ldots,1$}
       \label{algorithm:det}
       \IF {there exists a stage $h'\in[H]$ \highlight{such that \eqref{eq:def:switch} holds}}
         \STATE $\hat f_{k, h} \gets \argmin_{f_h \in \cF_h} \sum_{i \in [k - 1]} \frac{1}{\bar \sigma_{i, h}^2} (f_h(s_h^i, a_h^i) - r_h^i - V_{k, h + 1}(s_{h + 1}^i))^2$. \label{algorithm:line:hat}
       \STATE $\check f_{k, h} \gets  \argmin_{f_h \in \cF_h} \sum_{i \in [k - 1]} \frac{1}{\bar \sigma_{i, h}^2} (f_h(s_h^i, a_h^i) - r_h^i - \check V_{k, h + 1}(s_{h + 1}^i))^2$. \label{algorithm:line:check}
       \STATE $\tilde{f}_{k, h} \gets  \argmin_{f_h \in \cF_h} \sum_{i \in [k - 1]} \big(f_h(s_h^i, a_h^i) - \left(r_h^i + V_{k, h + 1}(s_{h + 1}^i)\right)^2\big)^2 $. \label{algorithm:line:tilde}
       \STATE $\qvalue_{k,h}(s,a) \gets \min\Big\{\hat f_{k, h}(s, a)+b_{k, h}(s, a) ,{\qvalue}_{k-1,h}(s,a), 1\Big\}$.
       \STATE $\check{\qvalue}_{k,h}(s,a)\gets \max\Big\{\check f_{k, h}(s, a)-\check b_{k, h}(s, a) ,\check{\qvalue}_{k-1,h}(s,a),0\Big\}$.
       \STATE Set the last updating episode $k_{\text{last}}\gets k$ and number of policies as $l_k \gets l_{k - 1} + 1$. 
       \ELSE 
       \STATE $\qvalue_{k,h}(s,a)\gets \qvalue_{k-1,h}(s,a)$, $\check{\qvalue}_{k,h}(s,a) \gets \check{\qvalue}_{k-1,h}(s,a)$ and $l_k \gets l_{k - 1}$. 
       \ENDIF  
       \STATE Set the policy $\pi^k$ as $\pi_h^k(\cdot) \gets \argmax_{a \in \cA} \qvalue_{k,h}(\cdot,a)$. $\vvalue_{k,h}(s) \gets \max_{a}\qvalue_{k,h}(s,a)$, $\check{\vvalue}_{k,h}(s) \gets \max_{a}\check{\qvalue}_{k,h}(s,a)$. 
       \label{algorithm:line3}
   \ENDFOR
   \FOR{stage $h=1,\ldots,H$}
   \STATE Take action $a_h^k\leftarrow \pi_h^k(s_h^k)$ and receive next state $s_{h+1}^k$. \label{algorithm:line4}
   \STATE Set the estimated variance $\sigma_{k,h}$ as in \eqref{eq:variance} and set $\bar\sigma_{k,h}\leftarrow \max\big\{\sigma_{k,h}, \alpha,   \gamma \cdot D^{1/2}_{\cF_h}(z; z_{[k - 1],h}, \bar\sigma_{[k - 1],h}) \big\}$. \label{algorithm:def-variance}
   \ENDFOR
\ENDFOR
   \end{algorithmic}
\end{algorithm*}

\subsection{Variance Estimator}

In this subsection, we provide more details about the variance estimator $\sigma_{k,h}$, which measures the variance of the value function $V_{k,h+1}(s_{h+1}^k)$ \highlight{caused by the stochastic transition from state-action pair $(s_h^k, a_h^k)$}. \highlight{According to the definition of $\hat f_{k, h}$, the difference between the estimator $\hat f_{k,h}$ and $\cT_h V_{k,h+1}$ satisfies}\begin{small}
\begin{align*}
   &\lambda +  \sum_{i \in [k - 1]} \frac{1}{\bar \sigma_{i, h}^2}\big(\hat f_{k, h}(s_h^i, a_h^i) - \cT_h V_{k, h + 1}(s_h^i, a_h^i)\big)^2\le 2 \sum_{i \in [k - 1]} \frac{1}{\bar \sigma_{i, h}^2}\big(f(s_h^i, a_h^i) - \hat f_k^*(s_h^i, a_h^i)\big) \cdot \hat \eta_h^i(V_{k, h + 1}),
\end{align*}\end{small}
where the noise $\hat \eta_h^k(V) =r_h^k + V(s_{h + 1}^k) -  \EE_{s' \sim \PP_h(s_h^k, a_h^k)}[r_h(s_h^k, a_h^k, s')+ V(s')]$ denotes the stochastic transition noise for the value function $V$. However, the generation of the target function $V_{k,h+1}$ relies on previously collected data $z_{[k_{\text{last}}]}$, thus violating the conditional independence property. Consequently, the noise term $\hat \eta_h^k(V_{k,h+1})$ may not be unbiased. To address this challenge, it becomes imperative to establish a uniform convergence property over the potential function class, which is first introduced in linear MDPs by \citet{jin2020provably}.

Following the previous approach introduced by \citet{azar2017minimax,hu2022nearly,agarwal2022vo, he2022nearly}, 
we decompose the noise of \emph{optimistic} value function $\hat \eta_h^k(V_{k,h+1})$ into the noise of \emph{optimal} value function $\hat \eta_h^k(V_{h+1}^*)$ and the noise $\hat \eta_h^k\big(V_{k,h+1}-V_{h+1}^*\big)$ \highlight{to reduce the extra $\log \big(N_\epsilon(K)\big)$ dependency in the confidence radius. }With the noise decomposition, we evaluate the variances $[\VV_h V_{h+1}^*](s,a)$ and  $\big[\VV_h (V_{k,h+1}-V_{h+1}^*)\big](s,a)$ separately.

For the variance of the \emph{optimal} value function $[\VV_h V_{h+1}^*](s,a)$, since the optimal value function $V_{h+1}^*$ is independent with the collected data $z_{[k_{\text{last}}]}$, it prevents a uniform convergence-based argument over the function class. However, the optimal value function $V_{h+1}^*$ is unobservable, and it requires several steps to estimate the variance. In summary, we utilize the optimistic function $V_{k,h+1}$ to approximate the optimal value function $V_{h+1}^*$ and calculate the estimated variance $[\bar\VV_h \vvalue_{k,h}]$ as the difference between the second-order moment and the square of the first-order moment of $V_{k,h}$
\begin{align*}
     [\bar{\VV}_{k,h}\vvalue_{k,h+1}] =\tilde f_{k,h}-\hat{f}_{k,h}^2.
\end{align*}
Here, the approximate second-order moment $\tilde f_{k,h}$
and the approximate first-order moment $\hat{f}_{k,h}$ is generated by the least-square regression (Lines \ref{algorithm:line:hat} and \ref{algorithm:line:tilde}). In addition, we introduce the exploration bonus $E_{k,h}$ to control the estimation error between the estimated variance and the true variance of $V_{k,h+1}$ and $F_{k,h}$ to control the sub-optimality gap between $V_{k,h+1}$ and $V_{h+1}^*$:
\begin{small}
\begin{align} 
    &E_{k,h}=(2L{\beta}_k+\tilde \beta_k) \min\big(1, \bar D_{\cF_h}(z; z_{[k - 1],h}, \bar\sigma_{[k - 1],h})\big), \notag\\
   &F_{k,h}=\big(\log (\cN(\cF,  \epsilon)\cdot \cN_{\epsilon}(K))\big) \cdot \min\big(1,2 \hat{f}_{k,h}(s_h^k,a_h^k) - 2 \check{f}_{k,h}(s_h^k,a_h^k) + 4\beta_k \bar D_{\cF_h}(z; z_{[k - 1],h}, \bar\sigma_{[k - 1],h})\big), \notag
\end{align} \end{small}where \begin{small} \begin{align}
    &\tilde \beta_k = \sqrt{128 \log \frac{\cN_\epsilon(k) \cdot \cN(\cF,  \epsilon) \cdot H}{\delta} + 64L\epsilon \cdot k},  \beta_k = \sqrt{128 \cdot \log\frac{\cN_\epsilon(k) \cdot \cN(\cF,  \epsilon) H}{\delta}+ 64 L \epsilon \cdot k/ \alpha^2}.  \notag
\end{align} \end{small}
For the variance of the sub-optimality gap, $\big[\VV_h (V_{k,h+1}-V_{h+1}^*)\big](s, a)$, based on the structure of optimistic and pessimistic value function, it can be approximate and upper bounded by
\begin{align*}
    &[\VV_h(\vvalue_{k,h+1}-\vvalue_{h+1}^*)](s_h^k,a_h^k)\leq 2 [\PP_h(\vvalue_{k,h+1}-\vvalue_{h+1}^*)](s_h^k,a_h^k)\notag\\
        &\leq 2 [\PP_h(\vvalue_{k,h+1}-\check{\vvalue}_{k,h+1})](s_h^k,a_h^k) \approx 2 \hat{f}_{k,h}(s_h^k,a_h^k) - 2 \check{f}_{k,h}(s_h^k,a_h^k),
\end{align*}
where the approximate first-order moments $\hat{f}_{k,h}$, $\check{f}_{k,h}$ are generated by the least-square regression (Lines \ref{algorithm:line:hat} and \ref{algorithm:line:check}) and can be dominated by the exploration bonus $F_{k,h}$.

 In summary, we construct the estimated variance $\sigma_{k,h}$ as:
\begin{align}
\sigma_{k,h}=\sqrt{[\bar{\VV}_{k,h}\vvalue_{k,h+1}](s_h^k,a_h^k)+E_{k,h}+F_{k,h}}. \label{eq:variance}
\end{align}

\subsection{Monotonic Value Function}
As we discussed in the previous subsection, we decompose the value function $V_{k,h}$ and evaluate the variance $[\VV_h V_{h+1}^*](s,a)$,  $\big[\VV_h (V_{k,h+1}-V_{h+1}^*)\big](s,a)$ separately. However, for each state-action pair $(s_h^k,a_h^k)$ and any subsequent episode $i>k$, the value function $V_{i,h}$ and corresponding variance $\big[\VV_h (V_{i,h+1}-V_{h+1}^*)\big](s_h^k,a_h^k)$ may differ from the previous value function $V_{k,h}$ and variance $\big[\VV_h (V_{k,h+1}-V_{h+1}^*)\big](s_h^k,a_h^k)$. Extending the idea proposed by \citet{he2022nearly} for linear MDPs, we ensure that the pessimistic value function $\check Q_{k,h}$ maintains a monotonically increasing property during updates, while the optimistic value function $Q_{k,h}$ maintains a monotonically decreasing property. Leveraging these monotonic properties, we can establish an upper bound on the variance as follows:
\begin{small}
\begin{align*}
     &\big[\VV_h (V_{i,h+1}-V_{h+1}^*)\big](s_h^k,a_h^k)
        \leq 2 [\PP_h(\vvalue_{i,h+1}-\check{\vvalue}_{i,h+1})](s_h^k,a_h^k)\leq 2 [\PP_h(\vvalue_{k,h+1}-\check{\vvalue}_{k,h+1})](s_h^k,a_h^k)\leq F_{k,h}.
\end{align*} \end{small}
In this scenario, the previously employed variance estimator $\sigma_{k,h}$ offers a consistent and uniform upper bound for the variance across all subsequent episodes.

\section{Main Results} \label{sec-4-1}
In this section, we present the main theoretical results. In detail, we provide the regret upper bound of Algorithm MQL-UCB in Theorem \ref{thm:main}. As a complement, in Section~\ref{sec-4-2}, we provide a lower bound on the communication complexity for cooperative linear MDPs. Finally, in Section~\ref{sec:connection}, we discussed the connection between the generalized eluder dimension and the standard eluder dimension. 

 The following theorem provides the regret upper bound of Algorithm MQL-UCB. 
\begin{theorem}\label{thm:main}
    Suppose Assumption \ref{assumption:complete} holds for function classes $\cF := \{\cF_h\}_{h = 1}^H$ and Definition \ref{def:ged} holds with $\lambda = 1$. If we set $\alpha = 1 / \sqrt{KH}$, $\epsilon = (KLH)^{-1}$, and set $\hat \beta_k^2 = \check \beta_k^2 := O\big(\log\frac{2k^2 \left(2\log (L^2k /\alpha^4)+2\right)\cdot\left(\log(4L /\alpha^2)+2\right)}{\delta / H}\big) \cdot \left[\log(\cN_\cF( \epsilon)) + 1\right] + O(\lambda) + O(\epsilon k L / \alpha^2)$, then 
    with probability $1 - O(\delta)$, the regret of MQL-UCB is upper bounded as follows: 
    \begin{align*} 
       &\regret(K) = \tilde{O}\big(\sqrt{\dim(\cF) \log \cN \cdot H\var_K}\big)  \\&+\tilde O\big(H^{2.5}\dim^2(\cF) \sqrt{\log \cN} \log(\cN \cdot \cN_b)\big) \cdot \sqrt{H \log \cN + \dim(\cF) \log(\cN \cdot \cN_b)} 
   \end{align*}
   where $\var_K := \sum_{k = 1}^K \sum_{h = 1}^H [\VV_h V_{h + 1}^{\pi^k}] (s_h^k, a_h^k) = \tilde{O}(K)$, we denote the covering number of bonus function class by $\cN_b$, the covering number of function class $\cF$ by $\cN$, and the dimension $\dim_{\alpha, K}(\cF)$ by $\dim(\cF)$. Meanwhile, the switching cost of Algorithm \ref{algorithm1} is $O(\dim_{\alpha, K}(\cF)\cdot H)$. 
\end{theorem}

In the worst case, when the number of episodes $K$ is sufficiently large, the leading term in our regret bound is $\tilde O \big(\sqrt{\dim(\cF) \log \cN \cdot HK}\big)$. 
Our result matches the optimal regret achieved by \cite{agarwal2022vo}. While their proposed algorithm involves the execution of a complicated and non-Markovian policy with an action selection phase based on two series of optimistic value functions and the prefix trajectory, MQL-UCB only requires the knowledge of the current state and a single optimistic state-action value function $Q$ to make a decision over the action space. In addition, our theorem also provides a variance-dependent regret bound, which is adaptive to the randomness of the underlying MDP encountered by the agent. Our definition of $\var_K$ is inspired by \citet{zhao2023variance} and \citet{zhou2023sharp}, which studied variance-adaptive RL under tabular MDP setting and linear mixture MDP setting, respectively. 

As a direct application, we also present the regret guarantee of MQL-UCB for linear MDPs. 

\begin{corollary}\label{coro:linear MDP}
Under the same conditions of Theorem \ref{thm:main}, assume that the underlying MDP is a linear MDP such that $\cF:=\{\cF_h\}_{h \in [H]}$ is composed of linear function classes with a known feature mapping over the state-action space $\bphi : \cS \times \cA \to \RR^d$. If we set $\lambda = 1$, $\alpha = 1 / \sqrt{K}$, then with probability $1 - O(\delta)$, the following cumulative regret guarantee holds for MQL-UCB: 
\begin{align*} 
    \regret(K) &= \tilde{O} \big(d \sqrt{HK} + H^{2.5}d^5\sqrt{H + d^2}\big). 
\end{align*}
\end{corollary}

\begin{remark} 
    The leading term in our regret bound, as demonstrated in Corollary \ref{coro:linear MDP}, matches the lower bound proved in \citet{zhou2021nearly} for linear MDPs. Similar optimal regrets have also been accomplished  by \citet{he2022nearly} and \citet{agarwal2022vo} for linear MDPs. Since we also apply weighted regression to enhance the precision of our pessimistic value functions, the lower order term (i.e., $H^{2.5}d^5\sqrt{H + d^2}$) in our regret has a better dependency  on $H$ than VO$Q$L \citep{agarwal2022vo} and LSVI-UCB++ \citep{he2022nearly}, which may be of independent interest when considering long-horizon MDPs. \highlight{In addition, the switching cost of Algorithm \ref{algorithm1} is bounded by $\tilde O(dH)$, which matches the lower bound in \citet{gao2021provably} for deterministic algorithms and our new lower bound in Theorem~\ref{lemma:lower} for arbitrary algorithms up to logarithmic factors. For more details about our lower bound, please refer to Appendix \ref{sec-lower-bound}}. 
\end{remark}

\section{Conclusion and Future Work}
In this paper, we delve into the realm of RL with general function approximation.
We proposed the MQL-UCB algorithm with an innovative uncertainty-based rare-switching strategy in general function approximation. Notably, our algorithm only requires $\tilde O(d H)$ updating times, which matches with the lower bound established by \citet{gao2021provably} up to logarithmic factors, and obtains a $\tilde O(d\sqrt{HK})$ regret guarantee, which is near-optimal when restricted to the linear cases. 

\begin{ack}
We thank the anonymous reviewers for their helpful comments. HZ is partially supported by the research award from Cisco. JH and QG are partially supported by the research fund from UCLA-Amazon Science Hub. The views and conclusions contained in this paper are those of the authors and should not be interpreted as representing any funding agencies.
\end{ack}

\bibliography{ref.bib}
\bibliographystyle{ims.bst}

\appendix


\section{Additional Related Work}
\subsection{RL with Linear Function Approximation}
In recent years, a substantial body of research has emerged to address the challenges of solving Markov Decision Processes (MDP) with linear function approximation, particularly to handle the vast state and action spaces \citep{jiang2017contextual, dann2018oracle, yang2019sample, du2019good, sun2019model,  jin2020provably, wang2020optimism, zanette2020frequentist, yang2020reinforcement, modi2020sample, ayoub2020model, zhou2021nearly, he2021logarithmic, zhou2022computationally,he2022nearly,zhao2023variance}. These works can be broadly categorized into two groups based on the linear structures applied to the underlying MDP. One commonly employed linear structure is known as the linear MDP \citep{jin2020provably}, where the transition probability function $\PP_h$ and reward function $r_h$ are represented as linear functions with respect to a given feature mapping $\bphi: \cS\times \cA \rightarrow \RR^d$. Under this assumption, the LSVI-UCB algorithm \citep{jin2020provably} has been shown to achieve a regret guarantee of $\tilde{O}(\sqrt{d^3H^4K})$. Subsequently, \citet{zanette2020frequentist} introduced the RLSVI algorithm, utilizing the Thompson sampling method, to attain a regret bound of $\tilde{O}(\sqrt{d^4H^5K})$. More recently, \citet{he2022nearly} improved the regret guarantee to $\tilde{O}(\sqrt{d^2H^3K})$ with the LSVI-UCB++ algorithm, aligning with the theoretical lower bound in \citet{zhou2021nearly} up to logarithmic factors. Another line of research has focused on linear mixture MDPs \citep{modi2020sample, yang2020reinforcement, jia2020model,ayoub2020model,zhou2021nearly}, where the transition probability is expressed as a linear combination of basic models $\PP_{1},\PP_{2},..,\PP_{d}$. For linear mixture MDPs, \citet{jia2020model} introduced the UCRL-VTR algorithm, achieving a regret guarantee of $\tilde O(\sqrt{d^2H^4K})$. Subsequently, \citet{zhou2021nearly} enhanced this result to $\tilde O(\sqrt{d^2H^3K})$, reaching a nearly minimax optimal regret bound. Recently, several works focused on time-homogeneous linear mixture MDPs, and removed the dependency on the episode length (horizon-free) \citep{zhang2019regret,zhou2020provably,zhao2023variance}.

\subsection{RL with General Function Approximation}
Recent years have witnessed a flurry of progress on RL with nonlinear function classes. To explore the theoretical limits of RL algorithms, various complexity measures have been developed to characterize the hardness of RL instances such as Bellman rank \citep{jiang2017contextual}, Witness rank \citep{sun2019model}, eluder dimension \citep{russo2013eluder}, Bellman eluder dimension \citep{jin2021bellman}, Bilinear Classes \citep{du2021bilinear}, Decision-Estimation Coefficient \citep{foster2021statistical}, Admissible Bellman Characterization \citep{chen2022general}, generalized eluder dimension \citep{agarwal2022vo}. Among them, only \cite{agarwal2022vo} yields a near-optimal regret guarantee when specialized to linear MDPs. In their paper, they proposed a new framework named generalized eluder dimension to handle the weighted objects in weighted regression, which can be seen as a variant of eluder dimension. In their proposed algorithm VO$Q$L, they adopted over-optimistic and over-pessimistic value functions in order to bound the variance of regression targets, making it possible to apply a weighted regression scheme in the model-free framework. 

\subsection{RL with Low Switching Cost}

Most of the aforementioned approaches necessitate updating both the value function and the corresponding policy in each episode, a practice that proves to be inefficient when dealing with substantial datasets. To overcome this limitation, a widely adopted technique involves dividing the time sequence into several epochs and updating the policy only between different epochs. In the context of the linear bandit problem, \citet{abbasi2011improved} introduced the rarely-switching OFUL algorithm, where the agent updates the policy only when the determinant of the covariance matrix doubles. This method enables the algorithm to achieve near-optimal regret of $\tilde O(\sqrt{K})$ while maintaining policy updates to $\tilde O(d\log K)$ times. When the number of arms, denoted as $|\cD|$, is finite, \citet{ruan2021linear} proposed an algorithm with regret bounded by $\tilde O(\sqrt{dK})$ and a mere $\tilde O(d\log d\log K)$ policy-updating times.
In the realm of episodic reinforcement learning, \citet{bai2019provably} and \citet{zhang2021reinforcement} delved into tabular Markov Decision Processes, introducing algorithms that achieve a regret of $\tilde O(\sqrt{T})$ and $\tilde O(SA\log K)$ updating times. \citet{wang2021provably} later extended these results to linear MDPs, unveiling the LSVI-UCB-RareSwitch algorithm. LSVI-UCB-RareSwitch delivers a regret bound of $\tilde O(d^3H^4K)$ with a policy switching frequency of $\tilde O(d\log T)$, which matches the lower bound of the switching cost \citep{gao2021provably} up to logarithmic terms.
Furthermore, if the policy is updated within fixed-length epochs \citep{han2020sequential}, the method is termed batch learning model, but this falls beyond the scope of our current work. \highlight{In addition, with the help of stochastic policies, \citet{zhang2022near} porposed an algorithm with $\tilde O(\sqrt{K})$ regret guarantee and only $\tilde O(H)$ swithcing cost for learning tabular MDPs. Later, \citet{huang2022towards} employed the stochastic policy in learning linear MDPs, which is able to find an $\epsilon$-optimal policy with only $\tilde O(H)$ switching cost.}

\section{Additional Results}

\subsection{Lower Bound for the Switching Cost}\label{sec-4-2}
As a complement, we prove a new lower bound on the {switching cost} for RL with linear MDPs. Note that linear MDPs lie in the class of MDPs studied in our paper with bounded generalized eluder dimension. In particular, the generalized eluder dimension of linear MDPs is $\tilde O(d)$.

\begin{theorem}\label{lemma:lower}
    For any algorithm \textbf{Alg} with expected switching cost less than $dH/(16\log K)$, there exists a hard-to-learn linear MDP, such that the expected regret of \textbf{Alg}  is at least $\Omega(K)$.
\end{theorem}

\begin{remark}
Theorem \ref{lemma:lower} suggests that, to achieve a sublinear regret guarantee, an $\tilde \Omega(dH)$ switching cost is inevitable. This lower bound does not violate the minimax upper bound of $\tilde O(H)$ proved in \citet{zhang2022near,huang2022towards}, which additionally assume that the initial state $s_1^k$ is either fixed or sampled from a fixed distribution. In contrast, our work and \citet{gao2021provably} allow the initial state to be \textbf{adaptively} chosen by an adversarial environment, bringing more challenges to the learning of linear MDPs. When comparing our lower bound with the result in \citet{gao2021provably}, it is worth noting that their focus is solely on deterministic algorithms, and they suggest that an $\tilde \Omega(dH)$ switching cost is necessary. As a comparison, our result holds for any algorithm with arbitrary policies including both deterministic and stochastic policies.
\end{remark}


\subsection{Connection Between $D_\cF^2$-Uncertainty and Eluder Dimension} \label{sec:connection}

Previous work by \cite{agarwal2022vo} achieved the optimal regret bound $\tilde O (\sqrt{\dim(\cF) \log \cN \cdot H K})$, where $\dim(\cF)$ is defined as the generalized eluder dimension as stated in Definition \ref{def:ged}. However, the connection between generalized eluder dimension and eluder dimension proposed by \cite{russo2013eluder} is still under-discussed~\footnote{\citet{agarwal2022vo} (Remark 4) also discussed the relationship between the generalized eluder dimension and eluder dimension. However, there exists a technique flaw in the proof and we will discuss it in Appendix \ref{sec-proof}.
}. Consequently, their results could not be directly compared with the results based on the classic eluder dimension measure \citep{wang2020reinforcement} or the more general Bellman eluder dimension \citep{jin2021bellman}. 

In this section, we make a first attempt to establish a connection between generalized eluder dimension and eluder dimension in Theorem \ref{thm:connection}.

\begin{theorem} \label{thm:connection}
    For a function space $\cG$, $\alpha > 0$, let $\dim$ be defined in Definition \ref{def:ged}. 
    WLOG, we assume that for all $g \in \cG$ and $z \in \cZ$, $|g(z)| \le 1$. Then the following inequality between $\dim(\cF, \Zb, \bsigma)$ and $\dim_E(\cG, 1 / \sqrt{T})$ holds for all $\Zb := \{z_i\}_{i \in [T]}$ with $z_i \in \cZ$ and $\bsigma := \{\sigma_i\}_{i \in [T]}$ s.t. $\alpha \le \sigma_i \le M\ \  \forall i \in [T]$: \begin{small}
    \begin{align*}
        &\sum_{i \in [T]} \min\Big(1, \frac{1}{\sigma_i^2} D_{\cG}^2\big(z_i; z_{[i - 1]}, \sigma_{[i - 1]}\big)\Big) \le O(\dim_E(\cF, 1 / \sqrt{T}) \log T \log \lambda T \log(M / \alpha) + \lambda^{-1}).
    \end{align*} \end{small}
\end{theorem}
\highlight{
According to Theorem \ref{thm:connection}, the generalized eluder dimension is upper bounded by eluder dimension up to logarithmic terms. When the number of episodes $K$ is sufficiently large, the leading term in our regret bound in Theorem \ref{thm:main} is $\tilde O \big(\sqrt{\dim_E(\mathcal{F}) \log \mathcal{N} \cdot HK}\big)$, where $\dim_E(\mathcal{F})$ is the eluder dimension of the function class $\mathcal{F}$. }

\section{Proof of Theorem \ref{thm:connection}}
To start with, we first recall the concept of eluder dimension as follows. 

\begin{definition}[Definition of eluder dimension, \citealt{russo2013eluder}] \label{def:ed}
The eluder dimension of a function class $\cG$ with domain $\cZ$ is defined as follows: 
\begin{itemize}[leftmargin = *]
    \item A point $z \in \cZ$ is $\epsilon$-dependent on $z_1, z_2, \cdots, z_k \in \cZ$ with respect to $\cG$, if for all $g_1, g_2 \in \cG$ such that $\sum_{i = 1}^k \left[g_1(z_i) - g_2(z_i)\right]^2 \le \epsilon$, it holds that $|g_1(z) - g_2(z)| \le \epsilon$.  
    \item Further $z$ is said to be $\epsilon$-independent of $z_1, z_2, \cdots, z_k$ with respect to $\cG$, if $z$ is not dependent on $z_1, z_2, \cdots, z_k$. 
    \item The eluder dimension of $\cG$, denoted by $\dim_E(\cG, \epsilon)$, is the length of the longest sequence of elements in $\cZ$ such that every element is $\epsilon'$-independent of its predecessors for some $\epsilon' \ge \epsilon$. 
\end{itemize}
\end{definition}
With this definition, we can prove Theorem \ref{thm:connection}.
\begin{proof}[Proof of Theorem \ref{thm:connection}] 
    Let $\cI_j$ ($1 \le j \le \lceil \log_2 M / \alpha \rceil $) be the index set such that $$\cI_j = \left\{t \in [T] | \sigma_t \in [2^{j - 1} \cdot \alpha, 2^j \alpha]\right\}.$$ 

    Then we focus on the summation over $\cI_j$ for each $j$. For simplicity, we denote the subsequence $\{z_i\}_{i \in \cI_j}$ by $\{x_i\}_{i \in \left[|\cI_j|\right]}$. Then we have \begin{align*} 
        \sum_{i \in \cI_j} \min\left(1, \frac{1}{\sigma_i^2} D_{\cG}^2\left(z_i; z_{[i - 1]}, \sigma_{[i - 1]}\right)\right) &\le \sum_{i \in [|\cI_j|]} \min\left(1, 4 D_{\cG}^2\left(x_i; x_{[i - 1]}, 1_{[i - 1]}\right)\right).
    \end{align*}
    To bound $\sum_{i \in [|\cI_j|]} \min\left(1, 4 D_{\cG}^2\left(x_i; x_{[i - 1]}, 1_{[i - 1]}\right)\right)$, \begin{align} 
        &\sum_{i \in [|\cI_j|]} \min\left(1, 4 D_{\cG}^2\left(x_i; x_{[i - 1]}, 1_{[i - 1]}\right)\right) \notag
        \\&\ \le 4 /\lambda  + \sum_{i \in [|\cI_j|]} 4 \int_{1 / \lambda T}^1 \ind\left\{D_{\cG}^2\left(x_i; x_{[i - 1]}, 1_{[i - 1]}\right) \ge \rho\right\} \text{d} \rho \notag
        \\&\ \le 4 /\lambda +  4 \int_{1 / \lambda T}^1 \sum_{i \in [|\cI_j|]}\ind\left\{D_{\cG}^2\left(x_i; x_{[i - 1]}, 1_{[i - 1]}\right) \ge \rho\right\} \text{d} \rho. \label{eq:bound:D2}
    \end{align}
    Then we proceed by bounding the summation $\sum_{i \in [|\cI_j|]}\ind\left\{D_{\cG}^2\left(x_i; x_{[i - 1]}, 1_{[i - 1]}\right) \ge \rho\right\}$ for each $\rho \ge 1 / (\lambda T)$. For short, let $d:= \dim_E(\cG, 1 / \sqrt{T})$. Essentially, it suffices to provide an upper bound of the cardinality of the subset $\cJ_j:= \left\{i \in \cI_j | D_{\cG}^2\left(x_i; x_{[i - 1]}, 1_{[i - 1]}\right) \ge \rho\right\}$. 

    For each $i \in \cJ_j$, since $D_{\cG}^2\left(x_i; x_{[i - 1]}, 1_{[i - 1]}\right) \ge \rho$, there exists $g_1, g_2$ in $\cG$ such that $$\frac{(g_1(x_i) - g_2(x_i))^2}{\sum_{t \in [i - 1]} (g_1(x_t) - g_2(x_t))^2 + \lambda} \ge \rho / 2.$$ Here $(g_1(x_i) - g_2(x_i))^2 \ge \lambda \rho \ge 1 / T$. Denote such value of $(g_1(x_i) - g_2(x_i))^2$ by $L(x_i)$. Then we consider split $cJ_j$ into $\lceil \log_2 T \rceil$ layers such that in each layer $\cJ_j^k \ (k \in [\lceil \log_2 T \rceil])$, we have $1 / T \le \xi \le L(x_i) \le 2\xi$ for some $\xi$. 

    Denote the cardinality of $\cJ_j^k$ by $A$ and the subsequence $\{x_i\}_{i \in \cI_j^k}$ by $\{y_i\}_{i \in [A]}$. For the elements in $\{y\}$, we dynamically maintain $\lfloor A / d \rfloor$ queues of elements. We enumerate through $\{y\}$ in its original order and put the current element $y_i$ into the queue $Q$ such that $y_i$ is $\xi$-independent of the current elements in $Q$. By the Pigeonhole principle and the definition of eluder dimension $d$, we can find an element $y_i$ in $\{y\}$ such that $y_i$ is $\xi$-dependent of all the $\lfloor A / d \rfloor$ queues before $i$. 

    Then by the definition of $L(y_i)$ and $\xi$, we can choose $g_1, g_2$ such that \begin{align} 
    \frac{(g_1(y_i) - g_2(y_i))^2}{\sum_{t \in [i - 1]} (g_1(y_t) - g_2(y_t))^2 + \lambda} \ge \rho / 2, \quad 2\xi \ge (g_1(y_t) - g_2(y_t))^2 \ge \xi. \label{eq:loss:upperbound}
    \end{align}
By the $\xi$-dependencies, we have $\sum_{t \in [i - 1]} (g_1(y_t) - g_2(y_t))^2 \ge \lfloor A / d \rfloor \cdot \xi$. At the same time, $\sum_{t \in [i - 1]} (g_1(y_t) - g_2(y_t))^2 \le 4 \xi / \rho$ due to \eqref{eq:loss:upperbound}. Thus, we deduce that $A = |\cJ_j^k| \le O(d / \rho)$ for all $k \in [\lceil \log_2 T \rceil]$. Substituting it into \eqref{eq:bound:D2}, we have \begin{align*}
        \sum_{i \in [|\cI_j|]} \min\left(1, 4 D_{\cG}^2\left(x_i; x_{[i - 1]}, 1_{[i - 1]}\right)\right) &\le 4 / \lambda + 4 \int_{1 / \lambda T}^1 O(d / \rho \cdot \log_2 T) \text{d} \rho \\&= O(d \log_2 T \cdot \log \lambda T + \lambda^{-1}). 
    \end{align*}
    Here $j$ is chosen arbitrarily. Hence, the generalized eluder dimension can be further bounded as follows: 
    \begin{align*} 
        \sum_{i \in [T]} \min\left(1, \frac{1}{\sigma_i^2} D_{\cG}^2\left(z_i; z_{[i - 1]}, \sigma_{[i - 1]}\right)\right) \le O(\dim_E(\cF, 1 / \sqrt{T}) \log T \log \lambda T \log(M / \alpha) + \lambda^{-1}). 
    \end{align*}
\end{proof}
In the following part, we will also discuss the issue in \citet{agarwal2022vo} about the relationship between the standard eluder dimension $\dim_E$ and the generalized eluder dimension $\dim$. In detail, \citet{agarwal2022vo} proposed the concept of Generalized Eluder dimension and made the following claim:
\begin{lemma}[Remark 4, \citealt{agarwal2022vo}]\label{lemma:vo}
If we set the weight $\bsigma=1$, then the Generalized Eluder dimension $\dim=\sup_{\Zb, \bsigma:|Z| = T} \dim(\cF, \Zb, \one)$ satisfied
\begin{align*}
    \dim\leq \dim_{E}(\cF,\sqrt{\lambda/T})+1,
\end{align*}
where $\dim_{E}$ denotes the standard Eluder dimension proposed in \citet{russo2013eluder}. In the proof of Remark 4, \citet{agarwal2022vo} claims that given the standard Eluder dimension $\dim_{E}(\cF,\sqrt{\lambda})=n$, there are at most $n$ different (sorted) indices $\{t_1,t_2,..\}$ such that $D_\cF^2(z_{t_i}; z_{t_i - 1]}, \one)\ge \epsilon^2/\lambda$. However, according to the definition of $D_\cF^2$-uncertainty, we only have
\begin{align*} 
       D_\cF^2(z; z_{[t - 1]}, \one) &:= \sup_{f_1, f_2 \in \cF} \frac{(f_1(z) - f_2(z))^2}{\sum_{s \in [t - 1]}(f_1(z_s) - f_2(z_s))^2 + \lambda}.
   \end{align*}
However, $z_{t_i}$ is $\epsilon$-dependence with $z_1,..,z_{t_i-1}$ is only a sufficient condition for the uncertainty $D_\cF^2(z; z_{[t - 1]}, \one) $ rather than necessary condition.
In cases where both $(f_1(z) - f_2(z))^2 \geq \epsilon^2$ and $\sum_{s \in [t - 1]}(f_1(z_s) - f_2(z_s))^2 \geq \epsilon^2$ hold, the uncertainty $D_\cF^2(z; z_{[t - 1]}, \mathbf{1})$ may exceed $\epsilon^2/\lambda$, yet it will not be counted within the longest $\epsilon$-independent sequence for the standard Eluder dimension.
\end{lemma}

\section{Proof of Theorem \ref{thm:main}}\label{sec-proof}
In this section, we provide the proof of Theorem \ref{thm:main}. 
\subsection{High Probability Events}
In this subsection, we define the following high-probability events: 
\begin{align} 
   \underline\cE_{k, h}^{\tilde f} = \left\{\lambda + \sum_{i \in [k - 1]} \left(\tilde f_{k, h}(s_h^i, a_h^i) - \cT_h^2 V_{k, h + 1}(s_h^i, a_h^i)\right)^2 \le \tilde \beta_{k}^2\right\}, \label{eq:def:tilde} \\
   \underline \cE_{k, h}^{\hat f} = \left\{\lambda + \sum_{i \in [k - 1]} \frac{1}{\bar \sigma_{i, h}^2}\left(\hat f_{k, h}(s_h^i, a_h^i) - \cT_h V_{k, h + 1}(s_h^i, a_h^i)\right)^2 \le  \beta_{k}^2\right\}, \label{eq:def:hat:1} \\
   \underline \cE_{k, h}^{\check f} = \left\{\lambda + \sum_{i \in [k - 1]} \frac{1}{\bar \sigma_{i, h}^2}\left(\check f_{k, h}(s_h^i, a_h^i) - \cT_h \check V_{k, h + 1}(s_h^i, a_h^i)\right)^2 \le \beta_{k}^2\right\}, \label{eq:def:check:1}
\end{align} where \begin{align*} 
\tilde \beta_k := \sqrt{128 \log \frac{\cN_\epsilon(k) \cdot \cN_\cF( \epsilon) \cdot H}{\delta} + 64L\epsilon \cdot k}, \ \beta_k := \sqrt{128 \cdot \log\frac{\cN_\epsilon(k) \cdot \cN_\cF( \epsilon) H}{\delta}+ 64 L \epsilon \cdot k/ \alpha^2}.  
\end{align*}
$\underline\cE_{k, h}^{\tilde f}, \underline \cE_{k, h}^{\hat f}$ and $\underline \cE_{k, h}^{\check f}$ are the hoeffding-type concentration results for $\tilde f_{k, h}$, $\hat f_{k, h}$ and $\check f_{k, h}$ respectively. 

Then we define the following bellman-type concentration events for $\hat f_{k, h}$, $\check f_{k, h}$, which implies a tighter confidence set due to carefully designed variance estimators and an inverse-variance weighted regression scheme for the general function classes. 

\begin{align} 
   \cE_{h}^{\hat f} =\left\{\lambda + \sum_{i=1}^{k-1} \frac{1}{(\bar \sigma_{i, h'})^2}\left(\hat f_{k, h'}(s_{h'}^i, a_{h'}^i) - \cT_{h'} V_{k, h' + 1}(s_{h'}^i, a_{h'}^i)\right)^2 \le \hat \beta_{k}^2,\ \forall h\leq h'\leq H, k\in[K]\right\}, \notag\\
   \cE_{h}^{\check f} =\left\{\lambda + \sum_{i=1}^{k-1} \frac{1}{(\bar \sigma_{i, h'})^2}\left(\check f_{k, h'}(s_{h'}^i, a_{h'}^i) - \cT_{h'} V_{k, h' + 1}(s_{h'}^i, a_{h'}^i)\right)^2 \le \check \beta_{k}^2,\ \forall h\leq h'\leq H, k\in[K]\right\}, 
\end{align}
where  \begin{align*}&\hat \beta_k^2 = \check \beta_k^2 :=  O\left(\log\frac{2k^2 \left(2\log (L^2k /\alpha^4)+2\right)\cdot\left(\log(4L /\alpha^2)+2\right)}{\delta / H}\right) \cdot \left[\log(\cN_\cF( \epsilon)) + 1\right] \\&\ + O(\lambda) + O(\epsilon k L / \alpha^2). \end{align*} 
We also define the following events which are later applied to prove the concentration of $\hat f_{k, h}$ and $\check f_{k, h}$ by induction. 

\begin{align} 
   \bar \cE_{k, h}^{\hat f} = \bigg\{\lambda + \sum_{i \in [k - 1]} \frac{\hat \ind_{i, h}}{(\bar \sigma_{i, h})^2}\left(\hat f_{k, h}(s_h^i, a_h^i) - \cT_h V_{k, h + 1}(s_h^i, a_h^i)\right)^2 \le \hat \beta_{k}^2\bigg\}, \label{eq:def:hat:2}
\end{align}
where $\hat \ind_{i, h}$ is the shorthand for the following product of indicator functions \begin{align} 
\hat \ind_{i, h} &:= \ind \left([\VV_h V_{h + 1}^*](s_h^i, a_h^i) \le \bar\sigma_{i, h}^2\right) \cdot \ind\big([\cT_h V_{i, h + 1} - \cT V_{h + 1}^*] \le (\log \cN_\cF( \epsilon) +  \log \cN_{\epsilon}(K))^{-1} \bar \sigma_{i, h}^2\big) \notag \\&\ \cdot \ind\left(V_{i, h + 1}(s) \ge V_{h + 1}^*(s) \quad \forall s \in \cS\right). \label{eq:def:hatind}
\end{align}
\begin{align} 
   \bar \cE_{k, h}^{\check f} = \bigg\{\lambda + \sum_{i \in [k - 1]} \frac{\check \ind_{i, h}}{(\bar \sigma_{i, h})^2}\left(\check f_{k, h}(s_h^i, a_h^i) - \cT_h \check V_{k, h + 1}(s_h^i, a_h^i)\right)^2 \le \check \beta_{k}^2\bigg\}, \label{eq:def:check:2}
\end{align}
where $\check \ind_{i, h}$ is the shorthand for the following product of indicator functions \begin{align} 
\check \ind_{i, h} &:= \ind \left([\VV_h V_{h + 1}^*](s_h^i, a_h^i) \le \bar\sigma_{i, h}^2\right) \cdot \ind\big([\cT V_{h + 1}^* - \cT_h \check V_{i, h + 1}] \le (\log \cN_\cF( \epsilon) +  \log \cN_{\epsilon}(K))^{-1} \bar \sigma_{i, h}^2\big) \notag \\&\ \cdot \ind\left( \check V_{i, h + 1}(s) \le V_{h + 1}^*(s) \quad \forall s \in \cS\right). \label{eq:def:checkind}
\end{align}
The following Lemmas suggest that previous events hold with high probability.
\begin{lemma} \label{lemma:tilde-event}
   Let $\tilde f_{k, h}$ be defined as in line \ref{algorithm:line:tilde} of Algorithm \ref{algorithm1}, we have $\underline \cE^{\tilde f}:= \bigcap_{k \ge 1, h \in [H]} \underline \cE_{k, h}^{\tilde f}$ holds with probability at least $1 - \delta$, where $\underline \cE_{k, h}^{\tilde f}$ is defined in \eqref{eq:def:tilde}. 
\end{lemma}
\begin{lemma} \label{lemma:hat-event-hoeffding}
   Let $\hat f_{k, h}$ be defined as in line \ref{algorithm:line:hat} of Algorithm \ref{algorithm1}, we have $\underline \cE^{\hat f}:= \bigcap_{k \ge 1, h \in [H]} \underline \cE_{k, h}^{\hat f}$ holds with probability at least $1 - \delta$, where $\underline \cE_{k, h}^{\hat f}$ is defined in \eqref{eq:def:hat:1}. 
\end{lemma}
\begin{lemma}\label{lemma:hat-event-hoeffding1}
   Let $\check f_{k, h}$ be defined as in line \ref{algorithm:line:check} of Algorithm \ref{algorithm1}, we have $\underline \cE^{\check f}:= \bigcap_{k \ge 1, h \in [H]} \underline \cE_{k, h}^{\check f}$ holds with probability at least $1 - \delta$, where $\underline \cE_{k, h}^{\check f}$ is defined in \eqref{eq:def:check:1}. 
\end{lemma}

\begin{lemma} \label{lemma:hat-event-hoeffding2}
   Let $\hat f_{k, h}$ be defined as in line \ref{algorithm:line:hat} of Algorithm \ref{algorithm1}, we have $\bar \cE^{\hat f} := \bigcap_{k\ge 1, h \in [H]} \bar \cE_{k, h}^{\hat f}$ holds with probability at least $1 - 2\delta$, where $\cE_{k, h}^{\hat f}$ is defined in \eqref{eq:def:hat:2}. 
\end{lemma}
\begin{lemma} \label{lemma:hat-event-hoeffding3}
   Let $\check f_{k, h}$ be defined as in line $\cdot$ of Algorithm 1, we have $\bar \cE^{\check f}$ holds with probability at least $1 - 2\delta$. 
\end{lemma}

\subsection{Optimism and Pessimism}
Based on the high probability events, we have the following lemmas for the function $\hat{f}_{k,h}(s,a)$ and $\check{f}_{k,h}(s,a)$
\begin{lemma}\label{lemma:bern-error}
   On the events $\cE_{h}^{\hat f}$ and $\cE_{h}^{\check f}$, for each episode $k\in[K]$, we have
   \begin{align*}
       &\big|\hat{f}_{k,h}(s,a)-\cT_h V_{k,h+1}(s,a)\big|\leq \hat{\beta}_k D_{\cF_h}(z; z_{[k - 1],h}, \bar\sigma_{[k - 1],h}),\notag\\
       &\big|\check{f}_{k,h}(s,a)-\cT_h \check{V}_{k,h+1}(s,a)\big|\leq \check{\beta}_k D_{\cF_h}(z; z_{[k - 1],h}, \bar\sigma_{[k - 1],h}),
   \end{align*}
   where $z=(s,a)$ and $z_{[k - 1],h}=\{z_{1,h},z_{2,h},..,z_{k-1,h}\}$.
\end{lemma}

\begin{lemma}\label{lemma:opt-pess}
     On the events  $\cE_{h+1}^{\hat f}$ and $\cE_{h+1}^{\check f}$, for each stage $h \leq h' \leq H$ and episode $k\in[K]$, we have $\qvalue_{k,h}(s,a)\ge \qvalue_{h}^*(s,a) \ge \check{\qvalue}_{k,h}(s,a)$. Furthermore, for the value functions $\vvalue_{k,h}(s)$ and $\check{\vvalue}_{k,h}(s)$, we have $\vvalue_{k,h}(s)\ge \vvalue_{h}^*(s) \ge \check{\vvalue}_{k,h}(s)$. 
\end{lemma}

\subsection{Monotonic Variance Estimator}
In this subsection, we introduce the following lemmas to construct the monotonic variance estimator.

\begin{lemma}\label{lemma:hoff-error}
   On the events $\underline\cE^{\tilde f}$, $\underline\cE^{\hat f}$ and $\underline\cE^{\check f}$, for each episode $k\in[K]$ and stage $h\in[H]$, we have
   \begin{align*}
       &\big|\hat{f}_{k,h}(s,a)-\cT_h V_{k,h+1}(s,a)\big|\leq {\beta}_k D_{\cF_h}(z; z_{[k - 1],h}, \bar\sigma_{[k - 1],h}),\notag\\
       &\big|\check{f}_{k,h}(s,a)-\cT_h \check{V}_{k,h+1}(s,a)\big|\leq {\beta}_k D_{\cF_h}(z; z_{[k - 1],h}, \bar\sigma_{[k - 1],h}),\notag\\
       &\big|\tilde{f}_{k,h}(s,a)-\cT_h^2 {V}_{k,h+1}(s,a)\big|\leq \tilde{\beta}_k D_{\cF_h}(z; z_{[k - 1],h}, \bar\sigma_{[k - 1],h}),
   \end{align*}
   where $z=(s,a)$ and $z_{[k - 1],h}=\{z_{1,h},z_{2,h},..,z_{k-1,h}\}$.
\end{lemma}
\begin{lemma}\label{lemma:variance-estimator-E}
   On the events $\underline\cE^{\tilde f}$, $\underline\cE^{\hat f}$, $\underline\cE^{\check f}$, $\cE_{h+1}^{\hat f}$, $\cE_{h+1}^{\check f}$, for each episode $k\in[K]$, the empirical variance $[\bar\VV_h\vvalue_{k,h+1}](s_h^k,a_h^k)$ satisfies the following inequalities:
     \begin{align}
         &\big|[\bar\VV_h\vvalue_{k,h+1}](s_h^k,a_h^k)         -[\VV_h\vvalue_{k,h+1}](s_h^k,a_h^k)\big|\leq E_{k,h},\notag\\
         &\big|[\bar\VV_h\vvalue_{k,h+1}](s_h^k,a_h^k)         -[\VV_h\vvalue_{h+1}^*](s_h^k,a_h^k)\big|\leq E_{k,h}+F_{k,h}.\notag
     \end{align}
\end{lemma}
\begin{lemma}\label{lemma:variance-estimator-D}
   On the events  $\underline\cE^{\hat f}$, $\underline\cE^{\check f}$, $\cE_{h+1}^{\hat f}$, $\cE_{h+1}^{\check f}$, for each episode $k\in[K]$ and $i>k$, we have
   \begin{align*}
       \left(\log \cN_\cF( \epsilon) +  \log \cN_{\epsilon}(K)\right)\cdot \big[\VV_h (V_{i,h + 1} -  V_{h + 1}^*)\big](s_h^k,a_h^k) &\leq F_{i,h},\notag\\
       \left(\log \cN_\cF( \epsilon) +  \log \cN_{\epsilon}(K)\right)\cdot \big[\VV_h (V_{h + 1}^* - \check V_{i, h + 1})\big](s_h^k,a_h^k) &\leq F_{i,h},
   \end{align*}
\end{lemma}
Based on previous lemmas, we can construct an optimistic estimator $\sigma_{k,h}$ for the transition variance. Under this situation, the weighted regression have the following guarantee.
\begin{lemma}\label{lemma:final-concenetration}
   If the events $\underline\cE^{\tilde f}$, $\underline\cE^{\hat f}$, $\underline\cE^{\check f}$, $\bar\cE^{\hat f}$ and $\bar\cE^{\check f}$ hold, then events $\cE^{\hat f}= \cE_{ 1}^{\hat f}$ and $\cE^{\check f}=\cE_1^{\check f}$ hold.
\end{lemma}

\subsection{Proof of Regret Bound}\label{sub-proof}
In the subsection, we first define the following high proSbability events to control the stochastic noise from the transition process:
\begin{align*}
   &\cE_1=\bigg\{\forall h'\in[H], \sum_{k=1}^K\sum_{h=h'}^{H}  [\PP_h(\vvalue_{k,h+1}-\vvalue_{h+1}^{\pi^k})\big](s_h^k,a_h^k) - \sum_{k=1}^K\sum_{h=h'}^{H}\big(\vvalue_{k,h+1}(s_{h+1}^{k})-\vvalue_{h+1}^{\pi^k}(s_{h+1}^{k})\big) \\& \leq 2\sqrt{\sum_{k = 1}^K \sum_{h=h'}^H [\VV_h (\vvalue_{k,h+1}-\vvalue_{h+1}^{\pi^k})](s_h^k, a_h^k) \log(2K^2 H /\delta)} + 2\sqrt{\log(2K^2 H /\delta)} + 2 \log(2K^2 H /\delta) \bigg\},\notag\\
   &\cE_2=\bigg\{\forall h'\in[H], \sum_{k=1}^K\sum_{h=h'}^{H}[\PP_h(\vvalue_{k,h+1}-\check{\vvalue}_{k,h+1})\big](s_h^k,a_h^k) -\sum_{k=1}^K\sum_{h=h'}^{H}\big(\vvalue_{k,h+1}(s_{h+1}^{k})-\check{\vvalue}_{k,h+1}(s_{h+1}^{k})\big)\\&\leq 2\sqrt{\sum_{k = 1}^K \sum_{h=h'}^H [\VV_h (\vvalue_{k,h+1}-\check V_{k, h + 1})](s_h^k, a_h^k) \log(2K^2 H /\delta)} + 2\sqrt{\log(2K^2 H /\delta)} + 2 \log(2K^2 H /\delta)\bigg\}.\notag
\end{align*}
According to Freedman inequality (Lemma \ref{lemma:freedman}), we directly have the following results.
\begin{lemma}\label{lemma:transition-noise}
   Events $\cE_1$ and $\cE_2$ hold with probability at least $1-2\delta$.
\end{lemma}

\begin{proof} 
Fix an arbitrary $h \in [H]$. Applying Lemma \ref{lemma:freedman}, we have the following inequality: 
\begin{align} 
&\sum_{k = 1}^K \sum_{h' = h}^H \PP_h(\vvalue_{k,h+1}-\vvalue_{h+1}^{\pi^k})\big](s_h^k,a_h^k)- \sum_{k=1}^K\sum_{h'=h}^{H}\big(\vvalue_{k,h+1}(s_{h+1}^{k})-\vvalue_{h+1}^{\pi^k}(s_{h+1}^{k})\big) \notag
\\&\le 2\sqrt{\sum_{k = 1}^K \sum_{h' = h}^H [\VV_h (\vvalue_{k,h+1}-\vvalue_{h+1}^{\pi^k})](s_h^k, a_h^k) \log(2K^2 H /\delta)} + 2\sqrt{\log(2K^2 H /\delta)} + 2 \log(2K^2 H /\delta) \label{eq:prob:e1}
\end{align}
Applying a union bound for \eqref{eq:prob:e1} across all $h \in [H], k \ge 0$, we have Event $\cE_1$ holds with probability at least $1 - 2\delta$. 

Similarly, we also have the corresponding high-probability bound for $\cE_2$. 
\end{proof}
Next, we need the following lemma to control the summation of confidence radius.
\begin{lemma}\label{lemma:sum-bonus}
   For any parameters $\beta \ge 1$ and stage $h\in [H]$, the summation of confidence radius over episode $k\in[K]$ is upper bounded by
   \begin{align*}
&\sum_{k=1}^K\min\Big(\beta D_{\cF_h}(z; z_{[k - 1],h}, \bar\sigma_{[k - 1],h}),1\Big)\\&\ \leq (1+\beta\gamma^2) \dim_{\alpha, K}(\cF_h) + 2 \beta \sqrt{\dim_{\alpha, K}(\cF_h)} \sqrt{\sum_{k=1}^K (\sigma_{k,h}^2+\alpha^2)},
   \end{align*}
   where $z=(s,a)$ and $z_{[k - 1],h}=\{z_{1,h},z_{2,h},..,z_{k-1,h}\}$.
\end{lemma}

Then, we can decompose the total regret in the first $K$ episode to the summation of variance $\sum_{k=1}^K \sum_{h=1}^H \sigma_{k,h}^2$ as the following lemma.

\begin{lemma}\label{lemma:transition}
    On the events $\cE^{\hat f}=\cE_1^{\hat f}$, $\cE^{\check f}=\cE_1^{\check f}$ and $\cE_1$, for each stage $h\in[H]$, the regret in the first $K$ episodes can be decomposed and controlled as:
   \begin{align}
       &\sum_{k=1}^K\big(\vvalue_{k,h}(s_h^k)-\vvalue_{h}^{\pi^k}(s_h^k)\big)\leq 2CH(1+\chi)(1+\hat \beta_{k}\gamma^2)  \dim_{\alpha, K}(\cF)\notag\\
       &\qquad + 4C(1+\chi) \hat \beta_{k} \sqrt{\dim_{\alpha, K}(\cF)} \sqrt{H \sum_{k=1}^K \sum_{h=1}^H (\sigma_{k,h}^2+\alpha^2)} \notag \\
       &\qquad + 2\sqrt{\sum_{k = 1}^K \sum_{h' = h}^H [\VV_h (\vvalue_{k,h+1}-\vvalue_{h+1}^{\pi^k})](s_h^k, a_h^k) \log(2K^2 H /\delta)} + 2\sqrt{\log(2K^2 H /\delta)} + 2 \log(2K^2 H /\delta) \notag
   \end{align}
   and for all stage $h\in[H]$, we further have
\begin{align}
   &\sum_{k=1}^K\sum_{h=1}^H \big[\PP_h(\vvalue_{k,h+1}-\vvalue_{h+1}^{\pi^k})\big](s_h^k,a_h^k)\notag\\
   &\leq 2CH^2(1+\chi)(1+\hat \beta_{k}\gamma^2) \dim_{\alpha, K}(\cF) + 4CH(1+\chi) \hat \beta_{k} \sqrt{\dim_{\alpha, K}(\cF_h)} \sqrt{H \sum_{k=1}^K \sum_{h=1}^H (\sigma_{k,h}^2+\alpha^2)}\notag\\
  &\qquad +H \left(2\sqrt{\sum_{k = 1}^K \sum_{h = 1}^H [\VV_h (\vvalue_{k,h+1}-\vvalue_{h+1}^{\pi^k})](s_h^k, a_h^k) \log(2K^2 H /\delta)} + 2\sqrt{\log(2K^2 H /\delta)} + 2 \log(2K^2 H /\delta)\right).\notag
\end{align}
\end{lemma}
In addition, the gap between the optimistic value function $V_{k,h}(s)$ and pessimistic value function $\check{V}_{k,h}(s)$ can be upper bounded by the following lemma.
\begin{lemma}\label{lemma:transition-OP}
    On the events $\cE^{\hat f}=\cE_1^{\hat f}$, $\cE^{\check f}=\cE_1^{\check f}$ and $\cE_2$, for each stage $h\in[H]$, the regret in the first $K$ episodes can be decomposed and controlled as:
   \begin{align}
       &\sum_{k=1}^K\big(\vvalue_{k,h}(s_h^k)-\check \vvalue_{k,h}(s_h^k)\big))\leq 4CH(1+\chi)(1+\hat \beta_{k}\gamma^2)  \dim_{\alpha, K}(\cF_{h}) \\&\quad+ 8C(1+\chi) \hat \beta_{k} \sqrt{\dim_{\alpha, K}(\cF)} \sqrt{H \sum_{k=1}^K \sum_{h=1}^H (\sigma_{k,h}^2+\alpha^2)}\notag\\
  &\qquad +2\sqrt{\sum_{k = 1}^K \sum_{h' = h}^H [\VV_h (\vvalue_{k,h+1}-\check V_{k, h + 1})](s_h^k, a_h^k) \log(2K^2 H /\delta)} + 2\sqrt{\log(2K^2 H /\delta)} + 2 \log(2K^2 H /\delta) \notag
   \end{align}
   and for all stage $h\in[H]$, we further have
\begin{align}
   &\sum_{k=1}^K\sum_{h=1}^H \big[\PP_h(\vvalue_{k,h+1}-\check{\vvalue}_{k,h+1})\big](s_h^k,a_h^k)\notag\\
   &\leq 4CH^2(1+\chi)(1+\hat \beta_{k}\gamma^2)  \dim_{\alpha, K}(\cF_{h}) + 8CH(1+\chi) \hat \beta_{k} \sqrt{\dim_{\alpha, K}(\cF)} \sqrt{H \sum_{k=1}^K \sum_{h=1}^H (\sigma_{k,h}^2+\alpha^2)}\notag\\
  &\qquad +2H\left(2\sqrt{\sum_{k = 1}^K \sum_{h' = h}^H [\VV_h (\vvalue_{k,h+1}-\check V_{k, h + 1})](s_h^k, a_h^k) \log(2K^2 H /\delta)} + 2\sqrt{\log(2K^2 H /\delta)} + 2 \log(2K^2 H /\delta)\right).\notag
\end{align}
\end{lemma}
We define the following high probability event $\cE_3$ to control the summation of variance.
\begin{align*}
   \cE_3=\bigg\{\sum_{k=1}^K\sum_{h=1}^H [\VV_h \vvalue_{h+1}^{\pi^k}](s_h^k,a_h^k)\leq 3K+3H\log(1/\delta)\bigg\}.
\end{align*}
According to \citet{jin2018q} (Lemma C.5)\footnote{\citet{jin2018q} showed that $\sum_{k=1}^K\sum_{h=1}^H [\VV_h \vvalue_{h+1}^{\pi^k}](s_h^k,a_h^k)=\tilde O(KH^2+H^3)$ when $\sum{h=1}^H r_h(s_h,a_h)\leq H$. In this work, we assume the total reward satisfied $\sum_{h=1}^H r_h(s_h,a_h)\leq 1$, and the summation of variance is upper bounded by $\tilde O(K+H)$ }, with probability at least $1-\delta$. Condition on this event, the summation of variance $\sum_{k=1}^K \sum_{h=1}^H \sigma_{k,h}^2$ can be upper bounded by the following lemma.
\begin{lemma}\label{lemma:total-variance}
   On the events $\cE_1$, $\cE_2$, $\cE_3$, $\cE^{\hat f}=\cE_1^{\hat f}$ and $\cE^{\check f}=\cE_1^{\check f}$, the total estimated variance is upper bounded by:
\begin{align*}
\sum_{k=1}^K \sum_{h=1}^H \sigma_{k,h}^2&\leq  \left(\log \cN_\cF( \epsilon) +  \log \cN_{\epsilon}(K)\right) \times \tilde O\big((1+\gamma^2)(\beta_k+ H \hat\beta_k+\tilde \beta_k) H\dim_{\alpha, K}(\cF)\big)\notag\\
&\qquad +\left(\log \cN_\cF( \epsilon) +  \log \cN_{\epsilon}(K)\right)^2 \times \tilde O\big((\beta_k+ H \hat\beta_k+\tilde \beta_k)^2 H\dim_{\alpha, K}(\cF)\big)\notag\\
&\qquad + \tilde O(\var_K+KH \alpha^2).
\end{align*}
\end{lemma}
With all previous lemmas, we can prove Theorem \ref{thm:main}
\begin{proof}[Proof of Theorem \ref{thm:main}]
The low switching cost result is given by Lemma \ref{lemma:switch}. 

After taking a union bound, the high probability events $\cE_1$, $\cE_2$, $\cE_3$, $\underline\cE^{\tilde f}$, $\underline\cE^{\hat f}$, $\underline\cE^{\check f}$, $\bar\cE^{\hat f}$ and $\bar\cE^{\check f}$ hold with probability at least $1-10\delta$. Conditioned on these events, the regret is upper bounded by
   \begin{align}
&\regret(K) \notag\\&=\sum_{k=1}^K\big(\vvalue_{1}^*(s_1^k)-\vvalue_{k,1}^{\pi^k}(s_1^k)\big)\notag\\
       &\leq\sum_{k=1}^K\big(\vvalue_{k,1}(s_1^k)-\vvalue_{k,1}^{\pi^k}(s_1^k)\big)\notag\\
       & \leq 2CH(1+\chi)(1+\hat \beta_{K}\gamma^2)  \dim_{\alpha, K}(\cF) +\tilde{O} \left(\sqrt{\sum_{k = 1}^K \sum_{h = 1}^H [\VV_h (\vvalue_{k,h+1}-\vvalue_{h+1}^{\pi^k})](s_h^k, a_h^k)}\right)\notag\\
       &\qquad + 4C(1+\chi) \hat \beta_{K} \sqrt{\dim_{\alpha, K}(\cF)} \sqrt{H \sum_{k=1}^K \sum_{h=1}^H (\sigma_{k,h}^2+\alpha^2)}\notag\\
       &\le 2CH(1+\chi)(1+\hat \beta_{K}\gamma^2)  \dim_{\alpha, K}(\cF) + \tilde{O} \left(\sqrt{\sum_{k = 1}^K \sum_{h = 1}^H [\PP_h (\vvalue_{k,h+1}-\vvalue_{h+1}^{\pi^k})](s_h^k, a_h^k)}\right)\notag\\
       &\qquad + 4C(1+\chi) \hat \beta_{K} \sqrt{\dim_{\alpha, K}(\cF)} \sqrt{H \sum_{k=1}^K \sum_{h=1}^H (\sigma_{k,h}^2+\alpha^2)}\notag\\
       &\le 2CH(1+\chi)(1+\hat \beta_{K}\gamma^2)  \dim_{\alpha, K}(\cF) + \tilde{O} \left(\hat \beta_K \sqrt{\dim_{\alpha, K}(\cF)} \sqrt{H \sum_{k=1}^K \sum_{h=1}^H (\sigma_{k,h}^2+\alpha^2)}\right) \notag\\
       & \leq  \tilde O\left(H \dim_{\alpha, K}(\cF) \cdot(1 + \hat \beta_K \gamma^2)\right)  + \tilde O\left(\hat \beta_K  \sqrt{\dim_{\alpha, K}(\cF)} \cdot \sqrt{H \var_K}\right) \notag
       \\&\ + \tilde O \left(H \dim_{\alpha, K}(\cF) \hat\beta_K (\beta_K + H \hat \beta_K + \tilde \beta_K) \cdot \log[\cN_\cF(\epsilon) \cdot \cN_\epsilon(K)]\right) + \tilde O\left(H \dim_{\alpha, K}(\cF) \hat\beta_K (1 + \gamma^2)\right)\notag\\
       & = \tilde{O}\left(H^3 \dim_{\alpha, K}^2(\cF) \cdot \log \cN_\cF\left(\frac{1}{2KHL}\right)\cdot \log \left[\cN_\cF\left(\frac{1}{2KHL}\right) \cdot \cN\left(\cB,  \frac{1}{2KHL \hat\beta_K}\right)\right]\right) \notag
       \\& \quad + \tilde{O}\left(H^{2.5} \dim_{\alpha, K}^{2.5}(\cF) \cdot \sqrt{\log \cN_\cF\left(\frac{1}{2KHL}\right)} \cdot \log^{1.5} \left[\cN_\cF\left(\frac{1}{2KHL}\right) \cdot \cN\left(\cB,  \frac{1}{2KHL \hat\beta_K}\right)\right]\right) \notag
       \\& \quad + \tilde O \left(\sqrt{\dim_{\alpha, K}(\cF) \log \cN_\cF\left(\frac{1}{2KHL}\right) \cdot H \var_K} \right), \label{eq:final-regret:0}
   \end{align}
   where the first inequality holds due to Lemma \ref{lemma:opt-pess}, the second inequality holds due to Lemma \ref{lemma:transition}, the third inequality follows from Lemma \ref{lemma:opt-pess} and $\var[X] \le \EE[X] \cdot M $ if random variable $X$ is in $[0, M]$, the fourth inequality holds due to \eqref{eq:0023} and $2ab \le a^2 + b^2$, the last inequality holds due to Lemma \ref{lemma:total-variance}. Thus, we complete the proof of Theorem \ref{thm:main}. 

   We can reorganize \eqref{eq:final-regret:0} into the following upper bound for regret, \begin{align*} 
       \regret(K) &= \tilde{O}(\sqrt{\dim(\cF) \log \cN \cdot H \var_K}) \\&\quad + \tilde O\left(H^{2.5}\dim^2(\cF) \sqrt{\log \cN} \log(\cN \cdot \cN_b) \cdot \sqrt{H \log \cN + \dim(\cF) \log(\cN \cdot \cN_b)}\right), 
   \end{align*}
   where we denote the covering number of bonus function class $ \cN\left(\cB,  \frac{1}{2KHL \hat\beta_K}\right)$ by $\cN_b$, the covering number of function class $\cF$ by $\cN$ and the dimension $\dim_{\alpha, K}(\cF)$ by $\dim(\cF)$. Since $\cE_3$ occurs with high probability according to \citet{jin2018q} (Lemma C.5), we have $\var_K = \tilde{O}(K)$, which matches the worst-case minimax optimal regret bound of MDPs with general function approximation. 
\end{proof}

\subsection{Proof Sketch of Regret Bound} \label{sketch}

In this subsection, we provide a proof sketch on the regret bound and show how our policy-switching strategy works. 

First, we introduce the following lemma which illustrates the stability of the bonus function after using the uncertainty-based policy-switching condition. 

\begin{lemma}[Restatement of Lemma~\ref{lemma:double}]
If the policy is not updated at episode $k$, the uncertainty of all state-action pair $z = (s, a) \in \cS \times \cA$ and stage $h \in [H]$ satisfies the following stability property: \begin{align*} 
   D_{\cF_h}^2(z; z_{[k - 1], h}, \bar \sigma_{[k - 1], h}) \ge \frac{1}{1 + \chi} D_{\cF_h}^2(z; z_{[k_{last} - 1], h}, \bar\sigma_{[k_{last} - 1], h}). 
\end{align*}
\end{lemma}

With Lemma \ref{lemma:double}, we can then convert$D_{\cF_h}^2(z; z_{[k_{last} - 1], h}, \bar\sigma_{[k_{last} - 1], h}$ in the bonus term into $D_{\cF_h}^2(z; z_{[k - 1], h}, \bar \sigma_{[k - 1], h})$ with a cost of an additional $1 + \chi$ constant. 

As a direct consequence of Lemma~\ref{lemma:double}, We have 
the following lemma, which controls the gap between optimistic value functions and real value functions resulting from the exploration bonuses. 

\begin{lemma}[Restatement of Lemma \ref{lemma:transition}]
    On the events $\cE^{\hat f}=\cE_1^{\hat f}$, $\cE^{\check f}=\cE_1^{\check f}$ and $\cE_1$, for each stage $h\in[H]$, the regret in the first $K$ episodes can be decomposed and controlled as:
   \begin{align}
       &\sum_{k=1}^K\big(\vvalue_{k,h}(s_h^k)-\vvalue_{h}^{\pi^k}(s_h^k)\big)\leq 2CH(1+\chi)(1+\hat \beta_{k}\gamma^2)  \dim_{\alpha, K}(\cF)\notag\\
       &\qquad + 4C(1+\chi) \hat \beta_{k} \sqrt{\dim_{\alpha, K}(\cF)} \sqrt{H \sum_{k=1}^K \sum_{h=1}^H (\sigma_{k,h}^2+\alpha^2)} \notag \\
       &\qquad + 2\sqrt{\sum_{k = 1}^K \sum_{h' = h}^H [\VV_h (\vvalue_{k,h+1}-\vvalue_{h+1}^{\pi^k})](s_h^k, a_h^k) \log(2K^2 H /\delta)} + 2\sqrt{\log(2K^2 H /\delta)} + 2 \log(2K^2 H /\delta) \notag
   \end{align}
   and for all stage $h\in[H]$, we further have
\begin{align}
   &\sum_{k=1}^K\sum_{h=1}^H \big[\PP_h(\vvalue_{k,h+1}-\vvalue_{h+1}^{\pi^k})\big](s_h^k,a_h^k)\notag\\
   &\leq 2CH^2(1+\chi)(1+\hat \beta_{k}\gamma^2) \dim_{\alpha, K}(\cF) + 4CH(1+\chi) \hat \beta_{k} \sqrt{\dim_{\alpha, K}(\cF_h)} \sqrt{H \sum_{k=1}^K \sum_{h=1}^H (\sigma_{k,h}^2+\alpha^2)}\notag\\
  &\qquad +H \left(2\sqrt{\sum_{k = 1}^K \sum_{h = 1}^H [\VV_h (\vvalue_{k,h+1}-\vvalue_{h+1}^{\pi^k})](s_h^k, a_h^k) \log(2K^2 H /\delta)} + 2\sqrt{\log(2K^2 H /\delta)} + 2 \log(2K^2 H /\delta)\right).\notag
\end{align}
\end{lemma}

As a need to bound the sum of bonuses corresponding to a weighted regression target, we derive the following results on the sum of inverse weights. 

\begin{lemma}[Informal version of Lemma~\ref{lemma:total-variance}]
   With high probability, the total estimated variance is upper bounded by:
\begin{align*}
\sum_{k=1}^K \sum_{h=1}^H \sigma_{k,h}^2&\leq  \left(\log \cN_\cF( \epsilon) +  \log \cN_{\epsilon}(K)\right) \times \tilde O\big((1+\gamma^2)(\beta_k+ H \hat\beta_k+\tilde \beta_k) H\dim_{\alpha, K}(\cF)\big)\notag\\
&\qquad +\left(\log \cN_\cF( \epsilon) +  \log \cN_{\epsilon}(K)\right)^2 \times \tilde O\big((\beta_k+ H \hat\beta_k+\tilde \beta_k)^2 H\dim_{\alpha, K}(\cF)\big)\notag\\
&\qquad + \tilde O(\var_K+KH \alpha^2).
\end{align*}
\end{lemma}

\begin{proof}[Proof sketch of Theorem \ref{thm:main}]
As the result of optimism, we have shown that with high probability, for all $k, h \in [K] \times [H]$, $V_{k, h}$ is an upper bound for $V_{h}^*$. Hence, we further have, \begin{align*} 
\regret(K) =\sum_{k=1}^K\big(\vvalue_{1}^*(s_1^k)-\vvalue_{k,1}^{\pi^k}(s_1^k)\big)
       \leq\sum_{k=1}^K\big(\vvalue_{k,1}(s_1^k)-\vvalue_{k,1}^{\pi^k}(s_1^k)\big). 
\end{align*}

To further proceed, we apply Lemma \ref{lemma:transition} and obtain \begin{align*} 
\regret(K) &\le 2CH(1+\chi)(1+\hat \beta_{K}\gamma^2)  \dim_{\alpha, K}(\cF) +\tilde{O} \left(\sqrt{\sum_{k = 1}^K \sum_{h = 1}^H [\VV_h (\vvalue_{k,h+1}-\vvalue_{h+1}^{\pi^k})](s_h^k, a_h^k)}\right) \\&\quad + 4C(1+\chi) \hat \beta_{K} \sqrt{\dim_{\alpha, K}(\cF)} \sqrt{H \sum_{k=1}^K \sum_{h=1}^H (\sigma_{k,h}^2+\alpha^2)}, 
\end{align*}
where the second term can be bounded by applying Lemma \ref{lemma:transition} repeatedly and the third term can be controlled by Lemma \ref{lemma:total-variance}. 

Then after substituting the value of $\hat beta_K, \gamma$ and $\alpha$, we conclude that with high probability, \begin{align*} 
\regret(K) &= \tilde{O}(\sqrt{\dim(\cF) \log \cN \cdot H \var_K}) \\&\quad + \tilde O\left(H^{2.5}\dim^2(\cF) \sqrt{\log \cN} \log(\cN \cdot \cN_b) \cdot \sqrt{H \log \cN + \dim(\cF) \log(\cN \cdot \cN_b)}\right). 
\end{align*} Please refer to Subsection \ref{sub-proof} for the detailed calculation of this part. 
\end{proof}

\section{Proof of Theorem \ref{lemma:lower}}\label{sec-lower-bound}

In this section, we provide the proof of Theorem \ref{lemma:lower}. To prove the lower bound, we create a series of hard-to-learn MDPs as follows. Each hard-to-learn MDP comprises $d/4$ distinct sub-MDPs denoted as $\cM_{1},..,\cM_{d/4}$. Each sub-MDP $\cM_{i}$ is characterized by two distinct states, initial state $s_{i,0}$ and absorbing state $s_{i,1}$, and shares the same action set $\cA=\{a_0,a_1\}$. Since the state and action spaces are finite, these tabular MDPs can always be represented as linear MDPs with dimension $|\cS|\times |\cA|=d$.

To generate each sub-MDP $\cM_{i}$, for all stage $h\in [H]$, a special action $a_{i,h}$ is uniformly randomly selected from the action set $\{a_0,a_1\}$. Given the current state $s_{i,0}$, the agent transitions to the state $s_{i,0}$ if it takes the special action $a_{i,h}$. Otherwise, the agent transitions to the absorbing state $s_{i,1}$ and remains in that state in subsequent stages. The agent will receive the reward $1$ if it takes the special action $a_{i,H}$ at the state $s_{i,0}$ during the last stage $H$. Otherwise, the agent always receives reward $0$. In this scenario, for sub-MDP $\cM_{i}$, the optimal policy entails following the special action sequence $(a_{i,1}, a_{i,2}, ..., a_{i,H})$ to achieve a total reward of $1$. In contrast, any other action sequence fails to yield any reward.

Now, we partition the $K$ episodes to $d/4$ different distinct epochs. For each epoch (ranging from episodes $4(i-1)K/d+1$ to episode $4iK/d$), we initialize the state as $s_{i,0}$ and exclusively focus on the sub-MDP $\cM_{i}$. The regret in each epoch can be lower bounded separately as follows:

\begin{lemma}\label{lemma:lower-epoch}
    For each epoch $i\in[d/4]$ and any algorithm $\textbf{Alg}$ capable of deploying arbitrary policies, if the expected switching cost in epoch $i$ is less than $H/(2\log K)$, the expected regret of $\textbf{Alg}$ in the $i$-th epoch is at least $\Omega(K/d)$.
\end{lemma} 
\begin{proof}[Proof of Lemma \ref{lemma:lower-epoch}]
     Given that each sub-MDP is independently generated, policy updates before epoch $i$ only offer information for the sub-MDPs $\cM_1,...,\cM_{i-1}$ and do not provide any information for the current epoch $i$. In this scenario, there is no distinction between epochs and for simplicity, we only focus on the first epoch, encompassing episodes $1$ to $4K/d$.
     
     Now, let $k_0=0$ and we denote $\cK=\{k_1,k_2,...\}$ as the set of episodes where the algorithm $\textbf{Alg}$ updates the policy. If $\textbf{Alg}$ does not update the policy $i$ times, we set $k_i=4K/d+1$. For simplicity, we set $C=2\log K$ and for each $i\leq H/C$, we define the events $\cE$ as the algorithm $\textbf{Alg}$ has not reached the state $s_{1,0}$ at the stage $iC$ before the episode $k_i$. 
     
    Conditioned on the events $\cE_1,...,\cE_{i-1}$, the algorithm $\textbf{Alg}$ does not gather any information about the special action $a_{1,h}$ for stage $h\ge (i-1)C +1$. In this scenario, the special actions can still be considered as uniformly randomly selected from the action set ${a_0, a_1}$. For each episode between $k_{i-1}$ and $k_i$, the probability that a policy $\pi$ arrives at state $s_{1,0}$ at stage $iC$ is upper-bounded by:
     \begin{align*}
         &\EE_{a_{1,(i-1)C +1},...,a_{1,iC}}\big[\Pi_{h=1}^{iC} \Pr(\pi_h(s_{1,0})=a_{1,h}) \big]
         \leq \EE_{a_{1,(i-1)C +1},...,a_{1,iC}}\big[\Pi_{h=(i-1)C +1}^{iC} \Pr(\pi_h(s_{1,0})=a_{1,h}) \big]= \frac{1}{2^{C}},
     \end{align*}
where the first inequality holds due to $\pi(s_{1,0})=a_{1,h})\leq 1$ and the second equation holds due to the random generation process of the special actions. Notice that there are at most $K$ episodes between the $k_{i-1}$ and $k_i$ and after applying an union bound for all episodes, we have
    \begin{align*}
        \Pr (\cE_i | \cE_1,...,\cE_{i-1})\leq 1-\frac{K}{2^C}=1-\frac{1} K.
    \end{align*}
Furthermore, we have
\begin{align}
    \Pr(\cE_{H/C})&\leq \Pr(\cE_{1}\cap \cE_{2} \cap \cE_{3} \cap ... \cap \cE_{H/C})= \Pi_{i=1}^{H/C} \Pr (\cE_i | \cE_1,...,\cE_{i-1}) \leq (1 - \frac{1}{1K})^{H/C}\leq 1 -\frac{H}{CK}.\label{eq-low-1}
\end{align}
Notice that the agent cannot receive any reward unless it has reached the state $s_{i,0}$ at the last stage $H$. Therefore, the expected regret for the algorithm $\textbf{Alg}$ for the sub-MDP $\cM$ can be bounded by the switching cost $\delta$:
\begin{align}
    \EE_{\cM}[\text{Regret}(\textbf{Alg})]&\ge \EE_{\cM}[\ind (\cE_{H/C}) \times (k_{H/C}-1)]\notag\\
    &\ge \EE_{\cM}[k_{H/C}-1]-\frac{H}C\notag\\
    &\ge \EE_{\cM}\big[\ind(\delta< H/C)\cdot K/d\big]-\frac{H}C\notag\\
    &\ge 4K/d- \EE_{\cM}\big[\delta\big]\cdot \frac{4KC}{dH}-\frac{H}C,\label{eq-low-2}
\end{align}
where the first inequality holds due to the fact that the agent receives no reward before $k_{H/C}$ conditioned on the event $\cE_{H/C}$, the second inequality holds due to \eqref{eq-low-1} with $k_{H/C}-1\leq K$, the third inequality holds due to the definition of $k_{H/C}$ and the last inequality holds due to $\EE[\ind(x\ge a)]\le \EE[x]/a$ for any non-negative random variable $x$. According to the result in \eqref{eq-low-1}, if the expected switching $\EE[\delta]\leq H/(2C)$, the expected regret is at least $\Omega(K)$, when $K$ is large enough compared to $H$. 
\end{proof}

With Lemma \ref{lemma:lower}, we can prove Theorem \ref{lemma:lower}.

\begin{proof}[Proof of Theorem \ref{lemma:lower}]
    For these constructed hard-to-learn MDPs and any given algorithm \textbf{Alg}, we denote the expected switching cost for sub-MDP $\cM_i$ as $\delta_i$. 
    
    According to Lemma \ref{lemma:lower-epoch}, we have
    \begin{align}
        \EE_{\cM_i}[\text{Regret}(\textbf{Alg})]&\ge \Pr \big(\delta_i < H/(2\log K)\big) \cdot K/d\notag\\
        &\ge \big(1-\EE_{\cM_i}[\delta_i]\cdot 2 \log K/H\big)\cdot K/d, \label{lower:03}
    \end{align}
    where the first inequality holds due to lemma \ref{lemma:lower} and last inequality holds due to $\EE[\ind(x\ge a)]\le \EE[x]/a$ for any non-negative random variable $x$. Taking a summation of \eqref{lower:03} for all sub-Mdps, we have
    \begin{align*}
        \EE_{\cM}[\text{Regret}(\textbf{Alg})]&=\sum_{i=1}^{d/4} \EE_{\cM_i}[\text{Regret}(\textbf{Alg})]\notag\\
        &\ge \sum_{i=1}^{d/4} \big(1-\EE_{\cM_i}[\delta_i]\cdot 2 \log K/H\big)\cdot K/d\notag\\
        &=(d/4- \EE_{\cM}[\delta] \cdot 2 \log K/H )\cdot K/d,
    \end{align*}
    where $\delta$ is the total expected switching cost. Therefore, for any algorithm with total expected switching cost less than $dH/(16\log K)$, the expected is lower bounded by
    \begin{align*}
        \EE_{\cM}[\text{Regret}(\textbf{Alg})] \ge (d/4- \EE_{\cM}[\delta] \cdot 2 \log K/H )\cdot K/d \ge K/8.
    \end{align*}
    Thus, we finish the proof of Theorem \ref{lemma:lower}.
\end{proof}

\color{black}

\section{Proof of Lemmas in Appendix \ref{sec-proof}}
In this section, we provide the detailed proof of lemmas in Appendix \ref{sec-proof}.

\subsection{Proof of High Probability Events}

\subsubsection{Proof of Lemma \ref{lemma:tilde-event}}

\begin{proof} [Proof of Lemma \ref{lemma:tilde-event}]
   We first prove that for an arbitrarily chosen $h \in [H]$, $\underline \cE_{k, h}^{\tilde f}$ holds with probability at least $1 - \delta / H$ for all $k$.

   For simplicity, in this proof, we denote $\cT_h^2 V_{k, h + 1}(s_h^i, a_h^i)$ as $\tilde f_k^*(s_h^i, a_h^i)$ where $\tilde f_k^* \in \cF_h$ exists due to our assumption. 
   For any function $V:\cS \to [0, 1]$, let $\tilde \eta_h^k(V) =\left(r_h^k + V(s_{h + 1}^k)\right)^2 -  \EE_{s' \sim s_h^k, a_h^k}\left[\left(r(s_h^k, a_h^k, s')+ V(s')\right)^2\right]$. 
   
   By simple calculation, for all $f \in \cF_h$, we have  \begin{align*} 
       &\sum_{i \in [k - 1]} \left(f(s_h^i, a_h^i) - \tilde f_k^*(s_h^i, a_h^i)\right)^2 
        + 2 \underbrace{\sum_{i \in [k - 1]} \left(f(s_h^i, a_h^i) - \tilde f_k^*(s_h^i, a_h^i)\right) \cdot \tilde \eta_h^i(V_{k, h + 1})}_{I(f, \tilde f_k^*, V_{k, h + 1})}
       \\&=  \sum_{i \in [k - 1]} \left[\left(r_h^i + V_{k, h + 1}(s_h^i)\right)^2 -  f(s_h^i, a_h^i)\right]^2 - \sum_{i \in [k - 1]} \left[\left(r_h^i + V_{k, h + 1}(s_h^i)\right)^2 -  \tilde f_k^*(s_h^i, a_h^i)\right]^2. 
   \end{align*} 
   Due to the definition of $\tilde f_{k, h}$, we have \begin{align} 
       \sum_{(i, j) \in [k - 1] \times [H]} \left(\tilde f_{k, h}(s_j^i, a_j^i) - \tilde f_k^*(s_j^i, a_j^i)\right)^2 
        + 2 I(\tilde f_{k, h}, \tilde f_k^*, V_{k, h + 1}) \le 0. \label{eq:tilde-f:1}
   \end{align}
   Then we give a high probability bound for $-I(f, \tilde f_k^*, V_{k, h + 1})$ through the following calculation. 
   
   Applying Lemma \ref{lemma:hoeffding-variant}, for fixed $f$, $\bar f$ and $V$, with probability at least $1 - \delta$, \begin{align*} 
       -I(f, \bar f, V) &:= -\sum_{i \in [k - 1]} \left(f(s_h^i, a_h^i) - \bar f(s_h^i, a_h^i)\right) \cdot \tilde \eta_h^i(V)  \\&\le 8\lambda\sum_{i \in [k - 1]} \left(f(s_h^i, a_h^i) - \bar f(s_h^i, a_h^i)\right)^2 + \frac{1}{\lambda} \cdot \log\frac{1}{\delta}.
   \end{align*}
   By the definition of $V$ in line $\cdot$ of Algorithm 1, $V_{k, h+1}$ lies in the optimistic value function class $\cV_{k}$ defined in 
   \eqref{def:opt_value_class}. Applying a union bound for all the value functions $V^c$ in the corresponding $\epsilon$-net $\cV^c$, we have \begin{align*} 
       -I(f, \bar f, V^c) \le \frac{1}{4}\sum_{i \in [k - 1]} \left(f(s_h^i, a_h^i) - \bar f(s_h^i, a_h^i)\right)^2 + 32 \cdot \log\frac{\cN_\epsilon(k)}{\delta}
   \end{align*} holds for all $k$ with probability at least $1 - \delta$. 

   For all $V$ such that $\|V - V^c\|_\infty \le \epsilon$, we have $|\eta_h^i(V) - \eta_h^i (V^c)| \le 4 \epsilon$. Therefore, with probability $1 - \delta$, the following bound holds for $I(f, \bar f, V_{k, h + 1})$: 
   \[ -I(f, \bar f, V_{k, h + 1}) \le \frac{1}{4}\sum_{i \in [k - 1]} \left(f(s_h^i, a_h^i) - \bar f(s_h^i, a_h^i)\right)^2 + 32 \cdot \log\frac{\cN_\epsilon(k)}{\delta} + 4\epsilon \cdot k. \]
   To further bound $I(\tilde f_{k, h}, \tilde f_k^*, V_{k, h + 1})$ in \eqref{eq:tilde-f:1}, we apply an $\epsilon$-covering argument on $\cF_h$ and show that with probability at least $1 - \delta$, \begin{align} 
       -I(\tilde f_{k, h}, \tilde f_k^*,  V_{k, h + 1}) &\le \frac{1}{4}\sum_{i \in [k - 1]} \left(\tilde f_{k, h}(s_h^i, a_h^i) - \tilde f_k^*(s_h^i, a_h^i)\right)^2 + 32 \cdot \log\frac{\cN_\epsilon(k) \cdot \cN_\cF( \epsilon)}{\delta} \notag \\&\ + 16 L \epsilon \cdot k\label{eq:tilde-f:3}
   \end{align} for probability at least $1 - \delta$. 

   Substituting \eqref{eq:tilde-f:3} into \eqref{eq:tilde-f:1} and rearranging the terms, 
   \[ \sum_{i \in [k - 1]} \left(\tilde f_{k, h}(s_h^i, a_h^i) - \tilde f_k^*(s_h^i, a_h^i)\right)^2 \le 128 \log \frac{\cN_\epsilon(k) \cdot \cN_\cF( \epsilon) \cdot H}{\delta} + 64L\epsilon \cdot k \] for all $k$ with probability at least $1 - \delta / H$. 
   
   Finally, we apply a union bound over all $h \in [H]$ to conclude that $\cE^{\tilde f}$ holds with probability at least $1 - \delta$.

\end{proof}
\subsubsection{Proof of Lemmas \ref{lemma:hat-event-hoeffding} and \ref{lemma:hat-event-hoeffding1}}

\begin{proof} [Proof of Lemma \ref{lemma:hat-event-hoeffding}]
    Similar to the proof of Lemma \ref{lemma:tilde-event}, we first prove that for an arbitrarily chosen $h \in [H]$, $\cE_{k, h}^{\hat f}$ holds with probability at least $1 - \delta / H$ for all $k$. 

    In this proof, we denote $\cT_h V_{k, h + 1}(s_h^i, a_h^i)$ as $\hat f_k^*(s_h^i, a_h^i)$ where $\hat f_k^* \in \cF_h$ exists due to our assumption. For any function $V:\cS \to [0, 1]$, let $\hat \eta_h^k(V) =\left(r_h^k + V(s_{h + 1}^k)\right) -  \EE_{s' \sim s_h^k, a_h^k}\left[r_h(s_h^k, a_h^k, s')+ V(s')\right]$. 

   For all $f \in \cF_h$, we have  \begin{align*} 
       &\sum_{i \in [k - 1]} \frac{1}{(\bar \sigma_{i, h})^2}\left(f(s_h^i, a_h^i) - \hat f_k^*(s_h^i, a_h^i)\right)^2 
        + 2 \underbrace{\sum_{i \in [k - 1]} \frac{1}{(\bar \sigma_{i, h})^2}\left(f(s_h^i, a_h^i) - \hat f_k^*(s_h^i, a_h^i)\right) \cdot \hat \eta_h^i(V_{k, h + 1})}_{I(f, \hat f_k^*, V_{k, h + 1})}
       \\&=  \sum_{i \in [k - 1]} \frac{1}{(\bar \sigma_{i, h})^2}\left[r_h^i + V_{k, h + 1}(s_h^i)^2 -  f(s_h^i, a_h^i)\right]^2 - \sum_{i \in [k - 1]} \frac{1}{(\bar \sigma_{i, h})^2}\left[r_h^i + V_{k, h + 1}(s_h^i)-  \hat f_k^*(s_h^i, a_h^i)\right]^2. 
   \end{align*} 
   By the definition of $\hat f_{k, h}$, we have \begin{align} 
       \sum_{i \in [k - 1]} \frac{1}{(\bar \sigma_{i, h})^2}\left(\hat f_{k, h} (s_h^i, a_h^i) - \hat f_k^*(s_h^i, a_h^i)\right)^2
        + 2 I(\hat f_{k, h}, V_{k, h + 1}) \le 0. \label{eq:hat-h-f:0}
   \end{align}
   Then we give a high probability bound for $-I(\hat f_{k, h}, \hat f_k^*, V_{k, h + 1})$ through the following calculation. 
   
   Applying Lemma \ref{lemma:hoeffding-variant}, for fixed $f$, $\bar f$ and $V$, with probability at least $1 - \delta$, \begin{align*} 
       -I(f, \bar f, V) &:= -\sum_{i \in [k - 1]} \frac{1}{(\bar \sigma_{i, h})^2}\left(f(s_h^i, a_h^i) - \bar f(s_h^i, a_h^i)\right) \cdot \hat \eta_h^i(V)  \\&\le 8\lambda \frac{1}{\alpha^2}\sum_{i \in [k - 1]} \frac{1}{(\bar \sigma_{i, h})^2} \left(f(s_h^i, a_h^i) - \bar f(s_h^i, a_h^i)\right)^2 + \frac{1}{\lambda} \cdot \log\frac{1}{\delta}.
   \end{align*}
   By the definition of $V$ in line $\cdot$ of Algorithm 1, $V_{k, h+1}$ lies in the optimistic value function class $\cV_{k}$ defined in \eqref{def:opt_value_class}. Applying a union bound for all the value functions $V^c$ in the corresponding $\epsilon$-net $\cV^c$, we have \begin{align*} 
       -I(f, \bar f, V^c) \le \frac{1}{4}\sum_{i \in [k - 1]} \frac{1}{(\bar \sigma_{i, h})^2}\left(f(s_h^i, a_h^i) - \bar f(s_h^i, a_h^i)\right)^2 + \frac{32}{\alpha^2} \cdot \log\frac{\cN_\epsilon(k)}{\delta}
   \end{align*} holds for all $k$ with probability at least $1 - \delta$. 
   
   For all $V$ such that $\|V - V^c\|_\infty \le \epsilon$, we have $|\eta_h^i(V) - \eta_h^i (V^c)| \le 4 \epsilon$. Therefore, with probability $1 - \delta$, the following bound holds for $I(f, \bar f, V_{k, h + 1})$: \[ -I(f, \bar f, V_{k, h + 1}) \le \frac{1}{4}\sum_{i \in [k - 1]} \frac{1}{(\bar \sigma_{i, h})^2}\left(f(s_h^i, a_h^i) - \bar f(s_h^i, a_h^i)\right)^2 + \frac{32}{\alpha^2} \cdot \log\frac{\cN_\epsilon(k)}{\delta} + 4\epsilon \cdot k / \alpha^2. \]
To further bound $I(\hat f_{k, h}, \hat f_k^*, V_{k, h + 1})$ in \eqref{eq:hat-h-f:0}, we apply an $\epsilon$-covering argument on $\cF_h$ and show that with probability at least $1 - \delta$, \begin{align} 
       -I(\hat f_{k, h}, \hat f_k^*,  V_{k, h + 1}) &\le \frac{1}{4}\sum_{i \in [k - 1]} \frac{1}{(\bar \sigma_{i, h})^2}\left(\hat f_{k, h}(s_h^i, a_h^i) - \hat f_k^*(s_h^i, a_h^i)\right)^2 + 32 \cdot \log\frac{\cN_\epsilon(k) \cdot \cN_\cF( \epsilon)}{\delta} \notag \\&\ + 16 L \epsilon \cdot k/ \alpha^2. \label{eq:hat-h-f:3}
   \end{align}
   Substituting \eqref{eq:hat-h-f:3} into \eqref{eq:hat-h-f:0} and rearranging the terms, we have \begin{align*} 
       \sum_{i \in [k - 1]} \frac{1}{(\bar \sigma_{i, h})^2}\left(\hat f_{k, h} (s_h^i, a_h^i) - \hat f_k^*(s_h^i, a_h^i)\right)^2 \le 128 \cdot \log\frac{\cN_\epsilon(k) \cdot \cN_\cF( \epsilon) H}{\delta}+ 64 L \epsilon \cdot k/ \alpha^2
   \end{align*} for all $k$ with probability at least $1 - \delta / H$. 
   Then we can complete the proof by using a union bound over all $h \in [H]$. 
\end{proof}

\begin{proof} [Proof of Lemma \ref{lemma:hat-event-hoeffding1}]
The proof is almost identical to the proof of Lemma \ref{lemma:hat-event-hoeffding}. 
\end{proof}

\subsubsection{Proof of Lemmas \ref{lemma:hat-event-hoeffding2} and \ref{lemma:hat-event-hoeffding3}}

\begin{proof} [Proof of Lemma \ref{lemma:hat-event-hoeffding2}]
   Similar to the proof of Lemma \ref{lemma:tilde-event}, we first prove that for an arbitrary $h \in [H]$, $\bar \cE_{k, h}^{\hat f}$ holds with probability at least $1 - \delta / H$ for all $k$.

   In this proof, we denote $\cT_h V_{k, h + 1}(s_h^i, a_h^i)$ as $\hat f_k^*(s_h^i, a_h^i)$ where $\hat f_k^* \in \cF$ exists due to our assumption. For any function $V:\cS \to [0, 1]$, let $\hat \eta_h^k(V) =\left(r_h^k + V(s_{h + 1}^k)\right) -  \EE_{s' \sim s_h^k, a_h^k}\left[r_h(s_h^k, a_h^k, s')+ V(s')\right]$. 

   For all $f \in \cF_h$, we have  \begin{align*} 
       &\sum_{i \in [k - 1]} \frac{\hat \ind_{i, h}}{(\bar \sigma_{i, h})^2}\left(f(s_h^i, a_h^i) - \hat f_k^*(s_h^i, a_h^i)\right)^2 
        + 2 \underbrace{\sum_{i \in [k - 1]} \frac{\hat \ind_{i, h}}{(\bar \sigma_{i, h})^2}\left(f(s_h^i, a_h^i) - \hat f_k^*(s_h^i, a_h^i)\right) \cdot \hat \eta_h^i(V_{k, h + 1})}_{I(f, \hat f_k^*, V_{k, h + 1})}
       \\&=  \sum_{i \in [k - 1]} \frac{\hat \ind_{i, h}}{(\bar \sigma_{i, h})^2}\left[r_h^i + V_{k, h + 1}(s_h^i)^2 -  f(s_h^i, a_h^i)\right]^2 - \sum_{i \in [k - 1]} \frac{\hat \ind_{i, h}}{(\bar \sigma_{i, h})^2}\left[r_h^i + V_{k, h + 1}(s_h^i)-  \hat f_k^*(s_h^i, a_h^i)\right]^2. 
   \end{align*} 
Due to the definition of $\hat f_{k, h}$, we have \begin{align} 
       \sum_{i \in [k - 1]} \frac{\hat \ind_{i, h}}{(\bar \sigma_{i, h})^2}\left(\hat f_{k, h} (s_h^i, a_h^i) - \hat f_k^*(s_h^i, a_h^i)\right)^2
        + 2 I(\hat f_{k, h}, V_{k, h + 1}) \le 0. \label{eq:hat-f:0}
   \end{align}
Then it suffices to bound the value of $I(f, \bar f, V_{k, h + 1})$ for all $f, \bar f \in \cF$. 

   Unlike the proof for Lemma \ref{lemma:tilde-event}, we decompose $I(f, \bar f, V_{k, h + 1})$ into two parts: \begin{align} 
       I(f, \bar f, V_{k, h + 1}) &= \sum_{i \in [k - 1]} \frac{\hat \ind_{i, h}}{(\bar \sigma_{i, h})^2}\left(f(s_h^i, a_h^i) - \bar f(s_h^i, a_h^i)\right) \cdot \hat \eta_h^i(V_{h + 1}^*) \notag \\& + \sum_{i \in [k - 1]} \frac{\hat \ind_{i, h}}{(\bar \sigma_{i, h})^2}\left(f(s_h^i, a_h^i) - \bar f(s_h^i, a_h^i)\right) \cdot \hat \eta_h^i(V_{k, h + 1} - V_{h + 1}^*).  \label{eq:hat-f:1}
   \end{align}
   Then we bound the two terms separately. 

   For the first term, we first check the following conditions before applying Lemma \ref{lemma:freedman-variant}, which is a variant of Freedman inequality. 
   For fixed $f$ and $\bar f$, we have \begin{align*} 
       \EE\left[\frac{\hat \ind_{i, h}}{(\bar \sigma_{i, h})^2}\left(f(s_h^i, a_h^i) - \bar f(s_h^i, a_h^i)\right) \cdot \hat \eta_h^i(V_{h + 1}^*)\right] = 0, 
       \end{align*}
       since $s_{h + 1}^i$ is sampled from $\PP_h(\cdot | s_h^i, a_h^i)$. 
       
       Next, we need to derive a bound for the maximum absolute value of each `weighted' transition noise:
       \begin{align}
        &\max_{i \in [k - 1]}\left|\frac{\hat \ind_{i, h}}{(\bar \sigma_{i, h})^2}\left(f(s_h^i, a_h^i) - \bar f(s_h^i, a_h^i)\right) \cdot \hat \eta_h^i(V_{h + 1}^*)\right| \notag
        \\&\le 4\max_{i \in [k - 1]}\left|\frac{\hat\ind_{i, h}}{(\bar \sigma_{i, h})^2}\left(f(s_h^i, a_h^i) - \bar f(s_h^i, a_h^i)\right)\right| \notag
       \\&\le 4 \max_{i \in [k - 1]} \frac{1}{(\bar \sigma_{i, h})^2} \sqrt{D_{\cF_h}^2\big(z_{i, h}; z_{[i - 1], h}, \bar\sigma_{[i-1], h}\big)\left(\sum_{\tau\in[k - 1]} \frac{\hat \ind_{\tau, h}}{(\bar \sigma_{\tau, h})^2} \left(f(s_h^\tau, a_h^\tau) - \bar f(s_h^\tau, a_h^\tau)\right)^2 + \lambda\right)} \notag
       \\&\le 4 \cdot \gamma^{-2} \sqrt{\sum_{\tau\in[k - 1]} \frac{\hat \ind_{\tau, h}}{(\bar \sigma_{\tau, h})^2} \left(f(s_h^\tau, a_h^\tau) - \bar f(s_h^\tau, a_h^\tau)\right)^2 + \lambda}, \label{eq:hat-f:1.0}
       \end{align}
       where the second inequality follows from the definition of $\cD_{\cF_h}$ in Definition \ref{def:ged}, the last inequality holds since $\bar \sigma_{i, h} \ge \gamma \cdot D_{\cF_h}^{1/2}\big(z_{i, h}; z_{[i - 1], h}, \bar\sigma_{[i-1], h}\big)$ according to line \ref{algorithm:def-variance} in Algorithm \ref{algorithm1}. 
       From the definition of $\hat \ind_{i, h}$ in \eqref{eq:def:hatind} we directly obtain the following upper bound of the variance:
       \begin{align*}
       &\sum_{i \in [k - 1]} \EE\left[\frac{\hat \ind_{i, h}}{(\bar \sigma_{i, h})^4}\left(f(s_h^i, a_h^i) - \bar f(s_h^i, a_h^i)\right)^2 \cdot \hat \eta_h^i(V_{h + 1}^*)^2\right] \\&\ \le 4 \sum_{i \in [k - 1]} \frac{\hat \ind_{i, h}}{(\bar \sigma_{i, h})^2} \left(f(s_h^i, a_h^i) - \bar f(s_h^i, a_h^i)\right)^2  \le 4 L^2 k /\alpha^2 := V. 
       \end{align*}

   Applying Lemma \ref{lemma:freedman-variant} with $V = 4 L^2 k /\alpha^2$, $M = 2L / \alpha^2$ and $v = m = 1$, for fixed $f$, $\bar f$, $k$, with probability at least $1 - \delta / (2k^2 \cdot \cN_\cF( \epsilon)\cdot H)$, 
   \begin{align} 
       &-\sum_{i \in [k - 1]} \frac{\hat \ind_{i, h}}{(\bar \sigma_{i, h})^2}\left(f(s_h^i, a_h^i) - \bar f(s_h^i, a_h^i)\right) \cdot \hat \eta_h^i(V_{h + 1}^*) \label{eq:hat-f:1.1}\\&\le \iota \sqrt{2\left(8\sum_{i \in [k - 1]} \frac{\hat \ind_{i, h}}{(\bar \sigma_{i, h})^2} \left(f(s_h^i, a_h^i) - \bar f(s_h^i, a_h^i)\right)^2 + 1 \right)} \notag \\&\quad + \frac{2}{3} \iota^2 \left(8 \gamma^{-2} \sqrt{\sum_{\tau\in[k - 1]} \frac{\hat\ind_{\tau, h}}{(\bar \sigma_{\tau, h})^2} \left(f(s_h^\tau, a_h^\tau) - \bar f(s_h^\tau, a_h^\tau)\right)^2 + \lambda} + 1\right)
       \\&\le \bigg(4\iota + \frac{16}{3} \iota^2 \gamma^{-2}\bigg) \sqrt{\sum_{\tau\in[k - 1]} \frac{\hat\ind_{\tau, h}}{(\bar \sigma_{\tau, h})^2} \left(f(s_h^\tau, a_h^\tau) - \bar f(s_h^\tau, a_h^\tau)\right)^2 + \lambda} + \sqrt{2} \iota + \frac{2}{3} \iota^2, \label{eq:hat-f:2}
   \end{align}
   where $\iota := \iota_1(k, h, \delta) = \sqrt{\log\frac{2k^2 \left(2\log( \frac{4L^2k}{\alpha^2})+2\right)\cdot\left(\log(\frac{4L}{\alpha^2})+2\right) \cdot \cN_\cF( \epsilon)}{\delta / H}}$. 

   Using a union bound across all $f, \bar f \in \cC(\cF_h, \epsilon)$ and $k \ge 1$, with probability at least $1 - \delta / H$, \begin{align} 
       &-\sum_{i \in [k - 1]} \frac{\hat\ind_{i, h}}{(\bar \sigma_{i, h})^2}\left(f(s_h^i, a_h^i) - \bar f(s_h^i, a_h^i)\right) \cdot \hat \eta_h^i(V_{h + 1}^*) \notag\\&\le \bigg(4\iota + \frac{16}{3} \iota^2 \gamma^{-2}\bigg) \sqrt{\sum_{\tau\in[k - 1]} \frac{\hat\ind_{\tau, h}}{(\bar \sigma_{\tau, h})^2} \left(f(s_h^\tau, a_h^\tau) - \bar f(s_h^\tau, a_h^\tau)\right)^2 + \lambda} + \sqrt{2} \iota + \frac{2}{3} \iota^2
   \end{align} holds for all $f, \bar f \in \cC(\cF_h, \epsilon)$ and $k$. 
   Due to the definition of $\epsilon$-net, we deduce that for $\hat f_{k, h}$ and $\hat f_k^*$, \begin{align} 
       &-\sum_{i \in [k - 1]} \frac{\hat \ind_{i, h}}{(\bar \sigma_{i, h})^2}\left(\hat f_{k, h}(s_h^i, a_h^i) - \hat f_k^*(s_h^i, a_h^i)\right) \cdot \hat \eta_h^i(V_{h + 1}^*) \notag\\&\le \bigg(4\iota + \frac{16}{3} \iota^2 \gamma^{-2}\bigg) \sqrt{\sum_{\tau\in[k - 1]} \frac{\hat \ind_{\tau, h}}{(\bar \sigma_{\tau, h})^2} \left(\hat f_{k, h}(s_h^\tau, a_h^\tau) - \hat f_k^*(s_h^\tau, a_h^\tau)\right)^2 + \lambda} + \sqrt{2} \iota + \frac{2}{3} \iota^2 \notag
       \\&\quad + \bigg(4\iota + \frac{16}{3} \iota^2 \gamma^{-2}\bigg) \sqrt{16k \epsilon L / \alpha^2} + 8\epsilon k / \alpha^2. \label{eq:hat-f:3}
   \end{align}
   For the second term in \eqref{eq:hat-f:1}, the following inequality holds for all $V \in \cV_{k, h + 1}$, \begin{align*} 
       &\EE\left[\frac{1}{(\bar \sigma_{i, h})^2}\left(f(s_h^i, a_h^i) - \bar f(s_h^i, a_h^i)\right) \cdot \hat \eta_h^i(V - V_{h + 1}^*)\right] = 0, \\
       &\max_{i \in [k - 1]}\left|\frac{1}{(\bar \sigma_{i, h})^2}\left(f(s_h^i, a_h^i) - \bar f(s_h^i, a_h^i)\right) \cdot \hat \eta_h^i(V - V_{h + 1}^*)\right| \\&\le 4\max_{i \in [k - 1]}\left|\frac{1}{(\bar \sigma_{i, h})^2}\left(f(s_h^i, a_h^i) - \bar f(s_h^i, a_h^i)\right)\right|
       \\&\le 4 \max_{i \in [k - 1]} \frac{1}{(\bar \sigma_{i, h})^2} \sqrt{D_{\cF_h}^2\big(z_{i, h}; z_{[i - 1], h}, \bar\sigma_{[i - 1], h}\big)\sum_{\tau\in[k - 1]} \frac{1}{(\bar \sigma_{\tau, h})^2} \left(f(s_h^\tau, a_h^\tau) - \bar f(s_h^\tau, a_h^\tau)\right)^2 + \lambda}
       \\&= 4 \cdot \gamma^{-2} \sqrt{\sum_{\tau\in[k - 1]} \frac{1}{(\bar \sigma_{\tau, h})^2} \left(f(s_h^\tau, a_h^\tau) - \bar f(s_h^\tau, a_h^\tau)\right)^2 + \lambda}, 
   \end{align*} 
   where the calculation is similar to that in \eqref{eq:hat-f:1.0}. 
   
   We denote the sum of variance by $\var(V - V_{h + 1}^*)$ as follows for simplicity: 
   \begin{small} 
       \begin{align}
           \var(V - V_{h + 1}^*) &:= \sum_{i \in [k - 1]} \EE\left[\frac{\hat \ind_{i, h}}{(\bar \sigma_{i, h})^4}\left(f(s_h^i, a_h^i) - \bar f(s_h^i, a_h^i)\right)^2 \cdot \hat \eta_h^i(V - V_{h + 1}^*)^2\right]
           \\&\le k \cdot L^2 / \alpha^4. 
       \end{align}
   \end{small}
   For $V_{k, h + 1}$, we have
       \begin{align*}
           \var(V_{k, h + 1} - V_{h + 1}^*)&:=\sum_{i \in [k - 1]} \EE\left[\frac{\hat\ind_{i, h}}{(\bar \sigma_{i, h})^4}\left(f(s_h^i, a_h^i) - \bar f(s_h^i, a_h^i)\right)^2 \cdot \hat \eta_h^i(V_{k, h + 1} - V_{h + 1}^*)^2\right] \\&\le \frac{4}{\log \cN_\cF( \epsilon) +  \log \cN_{\epsilon}(K)} \sum_{i \in [k - 1]} \frac{\hat \ind_{i, h}}{(\bar \sigma_{i, h})^2} \left(f(s_h^i, a_h^i) - \bar f(s_h^i, a_h^i)\right)^2, 
   \end{align*}
   where the second inequality holds due to the definition of $\hat\ind_{i, h}$ in \eqref{eq:def:hatind}. 
   
   With a similar argument as shown in \eqref{eq:hat-f:1.1}$\sim$\eqref{eq:hat-f:3}, we have with probability at least $1 - \delta / (2k^2 \cdot \cN_\cF( \epsilon) \cdot \cN_\epsilon(k - 1) \cdot H)$, for a fixed $f$, $\bar f$, $k$ and $V$, (applying Lemma \ref{lemma:freedman-variant}, with $V = k \cdot L^2 / \alpha^4$ and $M = 2L / \alpha^2$, $v = (\log \cN_\cF( \epsilon) +  \log \cN_{\epsilon}(K))^{-1/2}$, $m = v^2$. )
   \begin{align*} 
       &-\sum_{i \in [k - 1]} \frac{\hat \ind_{i, h}}{(\bar \sigma_{i, h})^2}\left(f(s_h^i, a_h^i) - \bar f(s_h^i, a_h^i)\right) \cdot \hat \eta_h^i(V_{k, h + 1} - V_{h + 1}^*)
       \\&\le \iota \sqrt{2 \left(2\var(V - V_{h + 1}^*) + \left(\log \cN_\cF( \epsilon) +  \log \cN_{\epsilon}(K)\right)^{-1}\right)} \\&\ + \frac{2}{3} \iota^2 \left(8 \gamma^{-2} \sqrt{\sum_{\tau\in[k - 1]} \frac{\hat\ind_{\tau, h}}{(\bar \sigma_{\tau, h})^2} \left(f(s_h^\tau, a_h^\tau) - \bar f(s_h^\tau, a_h^\tau)\right)^2 + \lambda} + \left(\log \cN_\cF( \epsilon) +  \log \cN_{\epsilon}(K)\right)^{-1}\right). 
   \end{align*}
   where  \begin{align*} &\log\frac{2k^2 \left(2\log \frac{L^2k \left(\log \cN_\cF( \epsilon)\cdot  \cN_{\epsilon}(K)\right)^{1 / 2}}{\alpha^4}+2\right)\cdot\left(\log(\frac{4L\left(\log \cN_\cF( \epsilon)\cdot\cN_{\epsilon}(K)\right)}{\alpha^2})+2\right) \cdot \cN_\cF( \epsilon) \cdot \cN_{\epsilon}(k)}{\delta / H} \\&\le \log\frac{2k^2 \left(2\log \frac{L^2k }{\alpha^4}+2\right)\cdot\left(\log(\frac{4L}{\alpha^2})+2\right) \cdot \cN_{\cF}^4(\epsilon) \cdot \cN_{\epsilon}^2(K)}{\delta / H}  := \iota_2^2(k, h, \delta).
   \end{align*} 
Using a union bound over all $(f, \bar f, V) \in \cC(\cF_h, \epsilon)\times \cC(\cF_h, \epsilon) \times \cV_{k, h + 1}^c$ and all $k \ge 1$, we have the inequality above holds for all such $f, \bar f, V, k$ with probability at least $1 - \delta / H$. 

   There exists a $V_{k, h + 1}^c$ in the $\epsilon$-net such that $\|V_{k, h + 1} - V_{k, h + 1}^c\|_\infty \le \epsilon$. Then we have 
       \begin{align} 
           &-\sum_{i \in [k - 1]} \frac{\hat \ind_{i, h}}{(\bar \sigma_{i, h})^2}\left(\hat f_{k, h}(s_h^i, a_h^i) - \hat f_k^*(s_h^i, a_h^i)\right) \cdot \hat \eta_h^i(V_{k, h + 1} - V_{h + 1}^*) \notag
           \\&\le O\left(\frac{\iota_2(k, h, \delta)}{\sqrt{\log \cN_\cF( \epsilon) +  \log \cN_{\epsilon}(K)}} + \frac{\iota_2(k, h, \delta)^2}{\gamma^{-2}}\right)  \notag\\&\  \cdot\sqrt{\sum_{\tau\in[k - 1]} \frac{\hat \ind_{\tau, h}}{(\bar \sigma_{\tau, h})^2} \left(\hat f_{k, h}(s_h^\tau, a_h^\tau) - \hat f_k^*(s_h^\tau, a_h^\tau)\right)^2 + \lambda} \notag
           \\&\  + O(\epsilon k L / \alpha)^2) + O\left(\frac{\iota_2^2(k, h, \delta)}{\log \cN_\cF( \epsilon) +  \log \cN_{\epsilon}(K)}\right), \label{eq:hat-f:6}
       \end{align}
   for all $k$ with at least probability $1 - \delta / H$. 

   Substituting \eqref{eq:hat-f:6} and \eqref{eq:hat-f:3} into \eqref{eq:hat-f:0}, we can conclude that \begin{align*} 
       &\sum_{i \in [k - 1]} \frac{\hat \ind_{i, h}}{(\bar \sigma_{i, h})^2}\left(\hat f_{k, h} (s_h^i, a_h^i) - \hat f_k^*(s_h^i, a_h^i)\right)^2 \\&\le O\left(\left(\frac{\iota_2(k, h, \delta)}{\sqrt{\log \cN_\cF( \epsilon) +  \log \cN_{\epsilon}(K)}} + \iota_2(k, h, \delta)^2 \cdot \gamma^{-2}\right)^2\right)+ O\left(\left(\iota_1(k, h, \delta) + \iota_1(k, h, \delta)^2 / \gamma^2\right)^2\right).
   \end{align*}
   From the definition of $\gamma$ in \eqref{eq:def:gamma}, we can rewrite the upper bound of the squared loss \\$\sum_{i \in [k - 1]} \frac{\hat \ind_{i, h}}{(\bar \sigma_{i, h})^2}\left(\hat f_{k, h} (s_h^i, a_h^i) - \hat f_k^*(s_h^i, a_h^i)\right)^2$ as follows: \begin{align*}
       &\lambda + \sum_{i \in [k - 1]} \frac{\hat \ind_{i, h}}{(\bar \sigma_{i, h})^2}\left(\hat f_{k, h} (s_h^i, a_h^i) - \hat f_k^*(s_h^i, a_h^i)\right)^2 \\&\le O\left(\log\frac{2k^2 \left(2\log \frac{L^2k }{\alpha^4}+2\right)\cdot\left(\log(\frac{4L}{\alpha^2})+2\right)}{\delta / H}\right) \cdot \left[\log(\cN_\cF( \epsilon)) + 1\right] + O(\lambda) + O(\epsilon k L / \alpha^2). 
   \end{align*}
\end{proof}

\begin{proof} [Proof of Lemma \ref{lemma:hat-event-hoeffding3}]
The proof is almost identical to the proof of Lemma \ref{lemma:hat-event-hoeffding2}. 
\end{proof}

\subsection{Proof of Optimism and Pessimism}

\subsubsection{Proof of Lemma \ref{lemma:bern-error}}

\begin{proof}[Proof of Lemma \ref{lemma:bern-error}]
   According to the definition of $D_\cF^2$ function, we have
\begin{align}
    &\big(\hat{f}_{k,h}(s,a)-\cT_h V_{k,h+1}(s,a)\big)^2\notag\\
    &\leq D_{\cF_h}^2(z; z_{[k - 1],h}, \bar\sigma_{[k - 1],h})\times \left(\lambda + \sum_{i=1}^{k-1} \frac{1}{(\bar\sigma_{i, h})^2}\left(\hat f_{k, h}(s_{h}^i, a_{h}^i) - \cT_{h} V_{k, h + 1}(s_{h}^i, a_{h}^i)\right)^2\right)\notag\\
    &\leq \hat{\beta}^2_k \times D_{\cF_h}^2(z; z_{[k - 1],h}, \bar\sigma_{[k - 1],h}),\notag
\end{align}
   where the first inequality holds due the definition of $D_\cF^2$ function with the Assumption \ref{assumption:complete} and the second inequality holds due to the events $\cE_{h}^{\hat f}$. Thus, we have
   \begin{align*}
       \big|\hat{f}_{k,h}(s,a)-\cT_h V_{k,h+1}(s,a)\big|\leq \hat{\beta}_k D_{\cF_h}(z; z_{[k - 1],h}, \bar\sigma_{[k - 1],h}).
   \end{align*}
    With a similar argument, for the pessimistic value function $\check{f}_{k,h}$, we have 
   \begin{align}
    &\big(\check{f}_{k,h}(s,a)-\cT_h \check V_{k,h}(s,a)\big)^2\notag\\
    &\leq D_{\cF_h}^2(z; z_{[k - 1],h}, \bar\sigma_{[k - 1],h})\times \left(\lambda + \sum_{i=1}^{k-1} \frac{1}{(\bar\sigma_{i, h})^2}\left(\check f_{k, h}(s_{h}^i, a_{h}^i) - \cT_{h} \check V_{k, h + 1}(s_{h}^i, a_{h}^i)\right)^2\right)\notag\\
    &\leq \check{\beta}^2_k \times D_{\cF_h}^2(z; z_{[k - 1],h}, \bar\sigma_{[k - 1],h}),\notag
\end{align}
where the first inequality holds due to the definition of $D_\cF^2$ function with the Assumption \ref{assumption:complete} and the second inequality holds due to the events $\cE_{h}^{\check f}$. In addition, we have,
   \begin{align*}
       \big|\check{f}_{k,h}(s,a)-\cT_h \check{V}_{k,h+1}(s,a)\big|\leq \hat{\beta}_k D_{\cF_h}(z; z_{[k - 1],h}, \bar\sigma_{[k - 1],h}).
   \end{align*}
   Thus, we complete the proof of Lemma \ref{lemma:bern-error}.
\end{proof}
\subsubsection{Proof of Lemma \ref{lemma:opt-pess}}

\begin{proof}[Proof of Lemma \ref{lemma:opt-pess}]
   We use induction to prove the optimistic and pessimistic property. First, we study the basic case with the last stage $H+1$. In this situation, $Q_{k,H+1}(s,a)=\qvalue_{h}^*(s,a)=\check{\qvalue}_{k,h}(s,a)=0$ and $\vvalue_{k,h}(s)= \vvalue_{h}^*(s) = \check{\vvalue}_{k,h}(s)=0$ hold for all state $s\in \cS$ and action $a\in \cA$. Therefore, Lemma \ref{lemma:opt-pess} holds for the basic case (stage $H+1$).

   Second, if Lemma \ref{lemma:opt-pess} holds for stage $h+1$, then we focus on the stage $h$. Notice that the event $\tilde{\cE}_h$ directly implies the event $\tilde{\cE}_{h+1}$. Therefore, according to the induction assumption, the following inequality holds for all state $s\in \cS$ and episode $k\in[K]$.
   \begin{align}
        \vvalue_{k,h+1}(s)\ge \vvalue_{h+1}^*(s) \ge \check{\vvalue}_{k,h}(s).\label{eq:0001}
   \end{align}
   Thus, for all episode $k\in[K]$ and state-action pair $(s,a)\in \cS\times \cA$, we have
   \begin{align}
&\hat{f}_{k,h}(s,a)+b_{k,h}(s,a)-Q_h^*(s,a)\notag\\
& \ge  \cT_h V_{k,h+1}(s,a)- \hat{\beta}_k \cdot  D_{\cF_h}(z; z_{[k - 1],h}, \bar\sigma_{[k - 1],h}) +b_{k,h}(s,a)-Q_h^*(s,a)\notag\\
& \ge \cT_h V_{k,h+1}(s,a) - Q_h^*(s,a)\notag\\
& = \PP_h V_{k,h+1}(s,a) - \PP_h V_h^*(s,a)\notag\\
&\ge 0,\label{eq:0002}
   \end{align}
   where the first inequality holds due to Lemma \ref{lemma:bern-error}, the second inequality holds due to the definition of the exploration bonus $b_{k,h}$ and the last inequality holds due to the \eqref{eq:0001}. Therefore, the optimal value function $Q_{h}^*(s,a)$ is upper bounded by
   \begin{align}
       Q_h^*(s,a)\leq \min\Big\{\min_{1\leq i\leq k} \hat{f}_{i,h}(s,a)+b_{i,h}(s,a) ,1\Big\}\leq Q_{k,h}(s,a), \label{eq:0003}
   \end{align}
   where the first inequality holds due to \eqref{eq:0002} with the fact that $Q_h^*(s,a)\leq 1$ and the second inequality holds due to the update rule of value function $Q_{k,h}$.

    With a similar argument, for the pessimistic estimator $\check{f}_{k,h}$, we have  \begin{align}
&\check{f}_{k,h}(s,a)-b_{k,h}(s,a)-Q_h^*(s,a)\notag\\
& \leq  \cT_h \check V_{k,h+1}(s,a)+ \check{\beta}_k \cdot  D_{\cF_h}(z; z_{[k - 1],h}, \bar\sigma_{[k - 1],h}) -b_{k,h}(s,a)-Q_h^*(s,a)\notag\\
& \leq \cT_h \check V_{k,h+1}(s,a) - Q_h^*(s,a)\notag\\
& = \PP_h \check V_{k,h+1}(s,a) - \PP_h V_h^*(s,a)\notag\\
&\leq 0,\label{eq:0004}
   \end{align}
   where the first inequality holds due to Lemma \ref{lemma:bern-error}, the second inequality holds due to the definition of the exploration bonus $b_{k,h}$ and the last inequality holds due to the \eqref{eq:0004}. Therefore, the optimal value function $Q_{h}^*(s,a)$ is lower bounded by
   \begin{align}
       Q_h^*(s,a)\ge \max\Big\{\max_{1\leq i\leq k} \check{f}_{i,h}(s,a)-b_{i,h}(s,a) ,0\Big\}\ge \check Q_{k,h}(s,a), \label{eq:0005}
   \end{align}
   where the first inequality holds due to \eqref{eq:0004} with the fact that $Q_h^*(s,a)\ge 0$ and the second inequality holds due to the update rule of value function $\check{Q}_{k,h}$.

   Furthermore, for the value functions $V_{k,h}$ and $\check V_{k,h}$, we have
\begin{align*}
   V_{k,h}(s)&=\max_a Q_{k,h}(s,a)\ge \max_a Q_h^*(s,a)=V_h^*(s),\notag\\
   \check{V}_{k,h}(s)&=  \max_a \check Q_{k,h}(s,a)\leq   \max_a Q_h^*(s,a)=V_h^*(s),\
\end{align*}
where the first inequality holds due to \eqref{eq:0003} and the second inequality holds due to \eqref{eq:0005}. Thus, by induction, we complete the proof of Lemma \ref{lemma:opt-pess}.
\end{proof}

\subsection{Proof of Monotonic Variance Estimator}
\subsubsection{Proof of Lemma \ref{lemma:hoff-error}}

\begin{proof}[Proof of Lemma \ref{lemma:hoff-error}]
   According to the definition of $D_\cF^2$ function, we have
\begin{align}
    &\big(\hat{f}_{k,h}(s,a)-\cT_h V_{k,h+1}(s,a)\big)^2\notag\\
    &\leq D_{\cF_h}^2(z; z_{[k - 1],h}, \bar\sigma_{[k - 1],h})\times \left(\lambda + \sum_{i=1}^{k-1} \frac{1}{(\bar\sigma_{i, h})^2}\left(\hat f_{k, h}(s_{h}^i, a_{h}^i) - \cT_{h} V_{k, h + 1}(s_{h}^i, a_{h}^i)\right)^2\right)\notag\\
    &\leq {\beta}^2_k \times D_{\cF_h}^2(z; z_{[k - 1],h}, \bar\sigma_{[k - 1],h}),\notag
\end{align}
   where the first inequality holds due the definition of $D_\cF^2$ function with the Assumption \ref{assumption:complete} and the second inequality holds due to the events $\underline \cE^{\hat f}$. Thus, we have
   \begin{align*}
       \big|\hat{f}_{k,h}(s,a)-\cT_h V_{k,h+1}(s,a)\big|\leq {\beta}_k D_{\cF_h}(z; z_{[k - 1],h}, \bar\sigma_{[k - 1],h}).
   \end{align*}
    For the pessimistic value function $\check{f}_{k,h}$, we have 
   \begin{align}
    &\big(\check{f}_{k,h}(s,a)-\cT_h \check V_{k,h}(s,a)\big)^2\notag\\
    &\leq D_{\cF_h}^2(z; z_{[k - 1],h}, \bar\sigma_{[k - 1],h})\times \left(\lambda + \sum_{i=1}^{k-1} \frac{1}{(\bar\sigma_{i, h})^2}\left(\check f_{k, h}(s_{h}^i, a_{h}^i) - \cT_{h} \check V_{k, h + 1}(s_{h}^i, a_{h}^i)\right)^2\right)\notag\\
    &\leq {\beta}^2_k \times D_{\cF_h}^2(z; z_{[k - 1],h}, \bar\sigma_{[k - 1],h}),\notag
\end{align}
where the first inequality holds due the definition of $D_\cF^2$ function with the Assumption \ref{assumption:complete} and the second inequality holds due to the events $\underline \cE^{\check f}$. In addition, we have
   \begin{align*}
       \big|\check{f}_{k,h}(s,a)-\cT_h \check{V}_{k,h+1}(s,a)\big|\leq {\beta}_k D_{\cF_h}(z; z_{[k - 1],h}, \bar\sigma_{[k - 1],h}).
   \end{align*}
With a similar argument, for the second-order estimator $\tilde f_{k,h}$, we have
   \begin{align}
    &\big(\tilde{f}_{k,h}(s,a)-\cT_h^2 V_{k,h}(s,a)\big)^2\notag\\
    &\leq D_{\cF_h}^2(z; z_{[k - 1],h}, \bar\sigma_{[k - 1],h})\times \left(\lambda + \sum_{i=1}^{k-1} \frac{1}{(\bar\sigma_{i, h})^2}\left(\tilde f_{k, h}(s_{h}^i, a_{h}^i) - \cT_{h}^2 V_{k, h + 1}(s_{h}^i, a_{h}^i)\right)^2\right)\notag\\
    &\leq \tilde{\beta}^2_k \times D_{\cF_h}^2(z; z_{[k - 1],h}, \bar\sigma_{[k - 1],h}),\notag
\end{align}
where the first inequality holds due the definition of $D_\cF^2$ function with the Assumption \ref{assumption:complete} and the second inequality holds due to the events $\underline \cE^{\tilde f}$. Therefore, we have
   \begin{align*}
       \big|\tilde{f}_{k,h}(s,a)-\cT_h^2 {V}_{k,h}(s,a)\big|\leq \tilde{\beta}_k D_{\cF_h}(z; z_{[k - 1],h}, \bar\sigma_{[k - 1],h}).
   \end{align*}
   Now, we complete the proof of Lemma \ref{lemma:hoff-error}.
\end{proof}
\subsubsection{Proof of Lemma \ref{lemma:variance-estimator-E}}

\begin{proof}[Proof of Lemma \ref{lemma:variance-estimator-E}]
   First, according to Lemma \ref{lemma:hoff-error}, we have
   \begin{align}
       &\big|[\bar\VV_h\vvalue_{k,h+1}](s_h^k,a_h^k)         -[\VV_h\vvalue_{k,h+1}](s_h^k,a_h^k)\big| \notag\\
       &=\Big|\tilde f_{k,h}-\hat{f}_{k,h}^2 - [\PP_h\vvalue_{k,h+1}^2](s_h^k,a_h^k)+\big([\PP_h\vvalue_{k,h+1}](s_h^k,a_h^k)\big)^2\Big|\notag\\
       &\leq \Big|\hat{f}_{k,h}^2 - \big([\PP_h\vvalue_{k,h+1}](s_h^k,a_h^k)\big)^2\Big| + \big| \tilde f_{k,h} - [\PP_h\vvalue_{k,h+1}^2](s_h^k,a_h^k)\big|\notag\\
       &\leq 2L \Big|\hat{f}_{k,h}- \big([\PP_h\vvalue_{k,h+1}](s_h^k,a_h^k)\big)\Big| + \big| \tilde f_{k,h} - [\PP_h\vvalue_{k,h+1}^2](s_h^k,a_h^k)\big|\notag\\
       &\leq (2L{\beta}_k+\tilde \beta_k) D_{\cF_h}(z; z_{[k - 1],h}, \bar\sigma_{[k - 1],h})\notag\\
       &=E_{k,h},\label{eq:0006}
   \end{align}
   where the first inequality holds due to $|a+b|\leq |a|+|b|$, the second inequality holds due to $|a^2-b^2|=|a-b|\cdot |a+b|\leq |a-b|\cdot 2 \max(|a|,|b|)$ and the last inequalirtnholds due to Lemma \ref{lemma:hoff-error}.

   For the difference between variances $[\VV_h\vvalue_{k,h+1}](s_h^k,a_h^k)$ and $[\VV_h\vvalue_{h+1}^*](s_h^k,a_h^k)$, it can be upper bounded by
\begin{align}
    &\big|[\VV_h\vvalue_{k,h+1}](s_h^k,a_h^k)         -[\VV_h\vvalue_{h+1}^*](s_h^k,a_h^k)\big|\notag\\
    &=\Big|  [\PP_h\vvalue_{k,h+1}^2](s_h^k,a_h^k)-\big([\PP_h\vvalue_{k,h+1}](s_h^k,a_h^k)\big)^2- [\PP_h(\vvalue_{h+1}^*)^2](s_h^k,a_h^k)+\big([\PP_h\vvalue_{h+1}^*](s_h^k,a_h^k)\big)^2\Big|\notag\\
    &\leq \big|[\PP_h\vvalue_{k,h+1}^2](s_h^k,a_h^k)-[\PP_h(\vvalue_{h+1}^*)^2](s_h^k,a_h^k)
    \big|+\Big|\big([\PP_h\vvalue_{k,h+1}](s_h^k,a_h^k)\big)^2-\big([\PP_h\vvalue_{h+1}^*](s_h^k,a_h^k)\big)^2\Big|\notag\\
    &\leq 4 \big([\PP_h\vvalue_{k,h+1}](s_h^k,a_h^k)-[\PP_h\vvalue_{h+1}^*](s_h^k,a_h^k)\big)\notag\\
    &\leq \big([\PP_h\vvalue_{k,h+1}](s_h^k,a_h^k)-[\PP_h\check{\vvalue}_{k,h+1}](s_h^k,a_h^k)\big)\notag\\
    &\leq \hat f_{k,h}(s_h^k,a_h^k) - \check f_{k,h}(s_h^k,a_h^k)+ 2 \beta_k D_{\cF_h}(z; z_{[k - 1],h}, \bar\sigma_{[k - 1],h}),\label{eq:0007}
\end{align}
where the first inequality holds due to $|a+b|\leq |a|+|b|$, the second inequality holds due to $1\ge V_{k,h+1}(\cdot)\ge V_{h+1}^*(\cdot)\ge 0$ (Lemma \ref{lemma:opt-pess}), the third inequality holds due to $V_{h+1}^*(\cdot)\ge \check{V}_{k,h+1}(\cdot)$  (Lemma \ref{lemma:opt-pess}) and the last inequality holds due to Lemma \ref{lemma:hoff-error}. Combining the results in \eqref{eq:0006} and \eqref{eq:0007} with the fact that $0\leq V_{k,h+1}(\cdot),V_{h+1}^*(\cdot)\leq 1$, we have 
   \begin{align*}
       &\big|[\bar\VV_h\vvalue_{k,h+1}](s_h^k,a_h^k)         -[\VV_h\vvalue_{h+1}^*](s_h^k,a_h^k)\big|\notag\\
   &\leq \big|[\bar\VV_h\vvalue_{k,h+1}](s_h^k,a_h^k)         -[\VV_h\vvalue_{k,h+1}](s_h^k,a_h^k)\big|+\big|[\VV_h\vvalue_{k,h+1}](s_h^k,a_h^k)         -[\VV_h\vvalue_{h+1}^*](s_h^k,a_h^k)\big|\notag\\
   &\leq E_{k,h}+F_{k,h}.\notag
   \end{align*}
Thus, we complete the proof of Lemma \ref{lemma:variance-estimator-E}.
\end{proof}

\subsubsection{Proof of Lemma \ref{lemma:variance-estimator-D}}

\begin{proof}[Proof of Lemma \ref{lemma:variance-estimator-D}]
   On the events $\underline\cE^{\hat f}$ and $\cE_{h+1}^{\hat f}$, we have
   \begin{align}
       &\big[\VV_h(\vvalue_{i,h+1}-\vvalue_{h+1}^*)\big](s_h^k,a_h^k)\notag\\
       &=[\PP_h(\vvalue_{i,h+1}-\vvalue_{h+1}^*)^2](s_h^k,a_h^k) - \big([\PP_h(\vvalue_{i,h+1}-\vvalue_{h+1}^*)](s_h^k,a_h^k)\big)^2\notag\\
       &\leq [\PP_h(\vvalue_{i,h+1}-\vvalue_{h+1}^*)^2](s_h^k,a_h^k)\notag\\
       &\leq 2\big[\PP_h(\vvalue_{i,h+1}-\vvalue_{h+1}^*)\big](s_h^k,a_h^k)\notag\\
       &\leq 2\big([\PP_h\vvalue_{i,h+1}](s_h^k,a_h^k)-[\PP_h\check{\vvalue}_{k,h+1}](s_h^k,a_h^k)\big)\notag\\
       &\leq 2\big([\PP_h\vvalue_{k,h+1}](s_h^k,a_h^k)-[\PP_h\check{\vvalue}_{k,h+1}](s_h^k,a_h^k)\big),\label{eq:0008}
   \end{align}
   where the first equation holds since the reward function is deterministic, the second inequality holds due to Lemma \ref{lemma:opt-pess} with the fact that $0\leq V_{i,h+1}(s), V_{h+1}^*(s) \leq 1$ and the third inequality holds due to Lemma \ref{lemma:opt-pess} and the last inequality holds due to the fact that $\vvalue_{k,h+1}\ge \vvalue_{i,h+1}$.
   
   With a similar argument, on the events $\underline\cE^{\check f}$ and $\cE_{h+1}^{\check f}$, we have
   \begin{align}
       &\big[\VV_h(V_{h + 1}^* - \check V_{i, h + 1})\big](s_h^k,a_h^k)\notag\\
       &=[\PP_h(V_{h + 1}^* - \check V_{i, h + 1})^2](s_h^k,a_h^k) - \big([\PP_h(V_{h + 1}^* - \check V_{i, h + 1})](s_h^k,a_h^k)\big)^2\notag\\
       &\leq [\PP_h(V_{h + 1}^* - \check V_{i, h + 1})^2](s_h^k,a_h^k)\notag\\
       &\leq 2\big[\PP_h(V_{h + 1}^* - \check V_{i, h + 1})\big](s_h^k,a_h^k)\notag\\
       &\leq 2\big([\PP_h\vvalue_{k,h+1}](s_h^k,a_h^k)-[\PP_h\check{\vvalue}_{i,h+1}](s_h^k,a_h^k)\big)\notag\\
       &\leq 2\big([\PP_h\vvalue_{k,h+1}](s_h^k,a_h^k)-[\PP_h\check{\vvalue}_{k,h+1}](s_h^k,a_h^k)\big),\label{eq:0009}
   \end{align}
   where the first inequality holds since the reward function is deterministic, the second and third inequality holds due to Lemma \ref{lemma:opt-pess} with the fact that $0\leq \check{V}_{i,h+1}(s),\vvalue_{h+1}^*(s)\leq H$, the last inequality the fact $\check \vvalue_{k,h+1}(s)\leq \check \vvalue_{i,h+1}(s)$. 
For both variances $\big[\VV_h(\vvalue_{i,h+1}-\vvalue_{h+1}^*)\big](s_h^k,a_h^k)$ and $\big[\VV_h(V_{h + 1}^* - \check V_{i, h + 1})\big](s_h^k,a_h^k)$, they are upper bounded by
\begin{align}
   &2\big([\PP_h\vvalue_{k,h+1}](s_h^k,a_h^k)-[\PP_h\check{\vvalue}_{k,h+1}](s_h^k,a_h^k)\big)\notag\\
   &=2 \cT_h \vvalue_{k,h+1}(s_h^k,a_h^k)-2\cT_h \check{\vvalue}_{k,h+1}(s_h^k,a_h^k)\notag\\
   &\leq 2 \hat{f}_{k,h}(s_h^k,a_h^k) - 2 \check{f}_{k,h}(s_h^k,a_h^k) + 4\beta_k D_{\cF_h}(z; z_{[k - 1],h}, \bar\sigma_{[k - 1],h}),\label{eq:0010}
\end{align}
where the first inequality holds due to Lemma \ref{lemma:hoff-error} and the second inequality holds due to the definition of $F_{k,h}$. Substituting the result in \eqref{eq:0010} into \eqref{eq:0008}, \eqref{eq:0009} and combining the fact that $\big[\VV_h (V_{i,h + 1} -  V_{h + 1}^*)\big](s_h^k,a_h^k),\big[\VV_h (V_{h + 1}^*-\check V_{i,h + 1})\big](s_h^k,a_h^k)\leq 1$, we finish the proof of Lemma \ref{lemma:variance-estimator-D}.
\end{proof}
\subsection{Proof of Lemmas in Section \ref{sub-proof}}
\subsubsection{Proof of Lemma \ref{lemma:final-concenetration}}
\begin{proof}[Proof of Lemma \ref{lemma:final-concenetration}]
We use induction to shows that the conclusion in Lemma \ref{lemma:opt-pess} and events $\cE_h^{\hat f}$, $\cE_H^{\check f}$ hold for each stage $h\in[H]$. First, for the basic situation (stage $H+1$), $Q_{k,H+1}(s,a)=\qvalue_{h}^*(s,a)=\check{\qvalue}_{k,h}(s,a)=0$ and $\vvalue_{k,h}(s)= \vvalue_{h}^*(s) = \check{\vvalue}_{k,h}(s)=0$ hold for all state $s\in \cS$ and action $a\in \cA$. Therefore, Lemma \ref{lemma:opt-pess} holds for the basic case (stage $H+1$)

Second, if Lemma \ref{lemma:opt-pess} holds for stage $h+1$, then we focus on the stage $h$. According to Lemmas \ref{lemma:variance-estimator-D} and Lemma \ref{lemma:variance-estimator-E}, we have the following inequalities:
\begin{align*}
   \sigma^2_{i,h}&=  [\bar{\VV}_{h}\vvalue_{i,h+1}](s_h^i,a_h^i)+E_{i,h}+D_{i,h}\ge [\bar{\VV}_{h}\vvalue_{h+1}^*](s_h^i,a_h^i),\notag\\
   \sigma^2_{i,h}&\ge F_{i,h}\ge \left(\log \cN_\cF( \epsilon) +  \log \cN_{\epsilon}(K)\right)\cdot \big[\VV_h (V_{k,h + 1} -  V_{h + 1}^*)\big](s_h^i,a_h^i), \notag\\
   \sigma^2_{i,h}&\ge F_{i,h}\ge \left(\log \cN_\cF( \epsilon) +  \log \cN_{\epsilon}(K)\right)\cdot \big[\VV_h ( V_{h + 1}^*-\check V_{k,h + 1}  )\big](s_h^i,a_h^i),
\end{align*}
where the first inequality holds due to Lemma \ref{lemma:variance-estimator-E}, the second and third inequality holds due to Lemma \ref{lemma:variance-estimator-D}. Thus, the indicator function in events $\bar\cE_{h}^{\hat f}$ and $\bar\cE_h^{\check f}$ hold, which implies events $\cE_h^{\hat f}$, $\cE_H^{\check f}$ hold. Furthermore, when events $\cE_h^{\hat f}$, $\cE_H^{\check f}$ hold, then Lemma \ref{lemma:opt-pess} holds for stage $h$. Thus, we complete the proof of Lemma \ref{lemma:final-concenetration} by induction.
\end{proof}

\subsubsection{Proof of Lemma \ref{lemma:sum-bonus}}
\begin{proof}[Proof of Lemma \ref{lemma:sum-bonus}]
For each stage $h$, we divide the episodes $\{1,2,..,K\}$ to the following sets:
\begin{align*}
   \cI_1&=\{k\in[K]: D_{\cF_h}(z; z_{[k - 1],h}, \bar\sigma_{[k - 1],h})/\bar\sigma_{k,h} \ge 1\},\\
   \cI_2&=\{k\in[K]: D_{\cF_h}(z; z_{[k - 1],h}, \bar\sigma_{[k - 1],h})/\bar\sigma_{k,h} < 1, \bar \sigma_{k,h}=\sigma_{k,h}\},\notag\\
   \cI_3&=\{k\in[K]: D_{\cF_h}(z; z_{[k - 1],h}, \bar\sigma_{[k - 1],h})/\bar\sigma_{k,h} < 1, \bar \sigma_{k,h}=\alpha \},\notag\\
   \cI_4&=\{k\in[K]: D_{\cF_h}(z; z_{[k - 1],h}, \bar\sigma_{[k - 1],h})/\bar\sigma_{k,h} < 1, \bar \sigma_{k,h}=\gamma \times D^{1/2}_{\cF_h}(z; z_{[k - 1],h}, \bar\sigma_{[k - 1],h}) \}.\notag
\end{align*}
The number of episode in set $\cI_1$ is upper bounded by
\begin{align*}
   |\cI_1|=\sum_{k\in \cI_1}\min\Big( D^2_{\cF_h}(z; z_{[k - 1],h}, \bar\sigma_{[k - 1],h})/{\bar\sigma}^2_{k,h},1\Big)\leq \dim_{\alpha, K}(\cF_h),
\end{align*}
where the equation holds due to $D^2_{\cF_h}(z; z_{[k - 1],h}, \bar\sigma_{[k - 1],h})/{\bar\sigma}^2_{k,h} \ge 1$ and the inequality holds due to the definition of Generalized Eluder dimension. Thus, for set $\cI_1$, the summation of confidence radius is upper bounded by
\begin{align}
   &\sum_{k\in \cI_1}\min\Big(\beta D_{\cF_h}(z; z_{[k - 1],h}, \bar\sigma_{[k - 1],h}),1\Big)\leq  |\cI_1| \leq \dim_{\alpha, K}(\cF_h). \label{eq:0011}
\end{align}

For set $\cI_2$, the summation of confidence radius is upper bounded by
\begin{align}
   &\sum_{k\in \cI_2}\min\Big(\beta D_{\cF_h}(z; z_{[k - 1],h}, \bar\sigma_{[k - 1],h}),1\Big)\notag\\
   &\leq \sum_{k\in \cI_2} \beta D_{\cF_h}(z; z_{[k - 1],h}, \bar\sigma_{[k - 1],h})\notag\\
   &\leq  \beta \sqrt{\sum_{k\in \cI_2} \sigma_{k,h}^2}\cdot \sqrt{\sum_{k\in \cI_2}D^2_{\cF_h}(z; z_{[k - 1],h}, \bar\sigma_{[k - 1],h})/{\bar\sigma}^2_{k,h}}\notag\\
   &\leq \beta \sqrt{\dim_{\alpha, K}(\cF_h)} \sqrt{\sum_{k\in \cI_2} \sigma_{k,h}^2},\label{eq:0012}
\end{align}
where the second inequality holds due to Cauchy-Schwartz inequality with $\sigma_{k,h}=\bar \sigma_{k,h}$ and the last inequality holds due to the definition of Generalized Eluder dimension with the fact that $D_{\cF_h}(z; z_{[k - 1],h}, \bar\sigma_{[k - 1],h})/\bar\sigma_{k,h} < 1$. 

With a similar argument, the summation of confidence radius over set $\cI_3$ is upper bounded by
\begin{align}
   &\sum_{k\in \cI_3}\min\Big(\beta D_{\cF_h}(z; z_{[k - 1],h}, \bar\sigma_{[k - 1],h}),1\Big)\notag\\
   &\leq \sum_{k\in \cI_3} \beta D_{\cF_h}(z; z_{[k - 1],h}, \bar\sigma_{[k - 1],h})\notag\\
   &\leq  \beta \sqrt{\sum_{k\in \cI_3} \alpha^2}\cdot \sqrt{\sum_{k\in \cI_3}D^2_{\cF_h}(z; z_{[k - 1],h}, \bar\sigma_{[k - 1],h})/{\bar\sigma}^2_{k,h}}\notag\\
   &\leq \beta \sqrt{\dim_{\alpha, K}(\cF_h)} \sqrt{\sum_{k\in \cI_3} \alpha^2},\label{eq:0013}
\end{align}
where the second inequality holds due to Cauchy-Schwartz inequality with $\bar \sigma_{k,h}=\alpha$ and the last inequality holds due to the definition of Generalized Eluder dimension with the fact that $D_{\cF_h}(z; z_{[k - 1],h}, \bar\sigma_{[k - 1],h})/\bar\sigma_{k,h} < 1$.

Finally, the summation of confidence radius over set $\cI_4$ is upper bounded by
With a similar argument, the summation of confidence radius over set $\cI_3$ is upper bounded by
\begin{align}
   &\sum_{k\in \cI_4}\min\Big(\beta D_{\cF_h}(z; z_{[k - 1],h}, \bar\sigma_{[k - 1],h}),1\Big)\notag\\
   &\leq \sum_{k\in \cI_4} \beta D_{\cF_h}(z; z_{[k - 1],h}, \bar\sigma_{[k - 1],h})\notag\\
   &= \sum_{k\in \cI_4}\beta\gamma^2  D^2_{\cF_h}(z; z_{[k - 1],h}, \bar\sigma_{[k - 1],h})/\bar\sigma^2_{k,h}\notag\\
   &\leq \beta \gamma^2 \dim_{\alpha, K}(\cF_h), \label{eq:0014}
\end{align}
where the first equation holds due to $\bar \sigma_{k,h}=\gamma \times D^{1/2}_{\cF_h}(z; z_{[k - 1],h}, \bar\sigma_{[k - 1],h})$ and the 
last inequality holds due to the definition of Generalized Eluder dimension with $D_{\cF_h}(z; z_{[k - 1],h}, \bar\sigma_{[k - 1],h})/\bar\sigma_{k,h} < 1$.

Combining the results in \eqref{eq:0011}, \eqref{eq:0012}, \eqref{eq:0013} and \eqref{eq:0014}, we have
\begin{align*}
   &\sum_{k=1}^K\min\Big(\beta D_{\cF_h}(z; z_{[k - 1],h}, \bar\sigma_{[k - 1],h}),1\Big)\\&\ \leq (1+\beta\gamma^2) \dim_{\alpha, K}(\cF_h) + 2 \beta \sqrt{\dim_{\alpha, K}(\cF_h)} \sqrt{\sum_{k=1}^K (\sigma_{k,h}^2+\alpha^2)}.
\end{align*}
Thus, we complete the proof of Lemma \ref{lemma:sum-bonus}.
   \end{proof}

\subsubsection{Proof of Lemma \ref{lemma:transition}}
\begin{proof}[Proof of Lemma \ref{lemma:transition}]
   First, for each stage $h\in[H]$ and episode $k\in[K]$, the gap between $\vvalue_{k,h}(s_h^k)$ and $\vvalue_{h}^{\pi^k}(s_h^k)$ can be decomposed as:
   \begin{align}
   &\vvalue_{k,h}(s_h^k)-\vvalue_{h}^{\pi^k}(s_h^k) \notag\\
  &= \qvalue_{k,h}(s_h^k,a_h^k)-\qvalue^{\pi^k}_{h}(s_h^k,a_h^k)\notag\\
&\leq\min\Big(\hat f_{k_{\text{last}}, h}(s, a)+b_{k_{\text{last}}, h}(s, a),1\Big)-\cT_h\vvalue_{k,h+1}(s_h^k,a_h^k)\notag\\
&\qquad + \cT_h\vvalue_{k,h+1}(s_h^k,a_h^k)-\cT_h\vvalue^{\pi^k}_{h+1}(s_h^k,a_h^k)\notag\\
  &\leq \big[\PP_h(\vvalue_{k,h+1}-\vvalue_{h+1}^{\pi^k})\big](s_h^k,a_h^k)+\min\big(\hat \beta_{k_{\text{last}}} D_{\cF_h}(z; z_{[k_{\text{last}} - 1],h}, \bar\sigma_{[k_{\text{last}} - 1],h}),1\big) + \min \big(b_{k_{\text{last}}, h}(s_h^k, a_h^k),1\big)
 \notag\\
 &\leq \big[\PP_h(\vvalue_{k,h+1}-\vvalue_{h+1}^{\pi^k})\big](s_h^k,a_h^k)+2C\cdot \min\big(\hat \beta_{k_{\text{last}}} D_{\cF_h}(z; z_{[k_{\text{last}} - 1],h}, \bar\sigma_{[k_{\text{last}} - 1],h}),1\big)\notag\\
 &\leq \big[\PP_h(\vvalue_{k,h+1}-\vvalue_{h+1}^{\pi^k})\big](s_h^k,a_h^k)+2C(1+\chi)\cdot  \min\big(\hat \beta_{k} D_{\cF_h}(z; z_{[k - 1],h}, \bar\sigma_{[k - 1],h}),1\big)\notag\\
  &=\vvalue_{k,h+1}(s_{h+1}^{k})-\vvalue_{h+1}^{\pi^k}(s_{h+1}^{k})+\big[\PP_h(\vvalue_{k,h+1}-\vvalue_{h+1}^{\pi^k})\big](s_h^k,a_h^k)-\big(\vvalue_{k,h+1}(s_{h+1}^{k})-\vvalue_{h+1}^{\pi^k}(s_{h+1}^{k})\big) \notag\\
  &\qquad +2C(1+\chi)\cdot \min\big(\hat \beta_{k} D_{\cF_h}(z; z_{[k - 1],h}, \bar\sigma_{[k - 1],h}),1\big),\label{eq:0015}
\end{align}
where the first inequality holds due to the definition of value function $Q_{k,h}(s_h^k,a_h^k)$, the second inequality holds due to Lemma \ref{lemma:bern-error}, the third inequality holds due to $b_{k_{\text{last}}, h}(s_h^k, a_h^k)\leq C \cdot  D_{\cF_h}(z; z_{[k_{\text{last}} - 1],h}, \bar\sigma_{[k_{\text{last}} - 1],h})$ and the last inequality holds due to Lemma \ref{lemma:double}.
Taking a summation of \eqref{eq:0015} over all episode $k\in[K]$ and stage $h' \ge h$, we have
\begin{align}
  &\sum_{k=1}^K\big(\vvalue_{k,h}(s_h^k)-\vvalue_{h}^{\pi^k}(s_h^k)\big)\notag\notag\\
  &\leq \sum_{k=1}^K\sum_{h'=h}^{H}\Big(\big[\PP_h(\vvalue_{k,h+1}-\vvalue_{h+1}^{\pi^k})\big](s_h^k,a_h^k)-\big(\vvalue_{k,h+1}(s_{h+1}^{k})-\vvalue_{h+1}^{\pi^k}(s_{h+1}^{k})\big)\Big)\notag\\
  &\qquad + \sum_{k=1}^K\sum_{h'=h}^{H} 2C(1+\chi)\cdot \min\big(\hat \beta_{k} D_{\cF_h}(z; z_{[k - 1],h}, \bar\sigma_{[k - 1],h}),1\big)\notag\\
  &\leq \sum_{k=1}^K\sum_{h'=h}^{H}2C(1+\chi)\cdot \min\big(\hat \beta_{k} D_{\cF_h}(z; z_{[k - 1],h}, \bar\sigma_{[k - 1],h}),1\big)\notag\\&\qquad +2\sqrt{\sum_{k = 1}^K \sum_{h' = h}^H [\VV_h (\vvalue_{k,h+1}-\vvalue_{h+1}^{\pi^k})](s_h^k, a_h^k) \log(2K^2 H /\delta)} + 2\sqrt{\log(2K^2 H /\delta)} + 2 \log(2K^2 H /\delta)\notag\\
  &\leq \sum_{h'=h}^H 2C(1+\chi)(1+\hat \beta_{k}\gamma^2)  \dim_{\alpha, K}(\cF_{h'}) + \sum_{h'=h}^H 4C(1+\chi) \hat \beta_{k} \sqrt{\dim_{\alpha, K}(\cF_{h'})} \sqrt{\sum_{k=1}^K (\sigma_{k,h'}^2+\alpha^2)}\notag\\
  &\qquad +2\sqrt{\sum_{k = 1}^K \sum_{h' = h}^H [\VV_h (\vvalue_{k,h+1}-\vvalue_{h+1}^{\pi^k})](s_h^k, a_h^k) \log(2K^2 H /\delta)} + 2\sqrt{\log(2K^2 H /\delta)} + 2 \log(2K^2 H /\delta)\notag\\
  &\leq  2CH(1+\chi)(1+\hat \beta_{k}\gamma^2)  \dim_{\alpha, K}(\cF_{h}) + 4C(1+\chi) \hat \beta_{k} \sqrt{\sum_{h'=h}^H \dim_{\alpha, K}(\cF_{h'})} \sqrt{\sum_{k=1}^K \sum_{h'=h}^H (\sigma_{k,h'}^2+\alpha^2)}\notag\\
  &\qquad +2\sqrt{\sum_{k = 1}^K \sum_{h' = h}^H [\VV_h (\vvalue_{k,h+1}-\vvalue_{h+1}^{\pi^k})](s_h^k, a_h^k) \log(2K^2 H /\delta)} + 2\sqrt{\log(2K^2 H /\delta)} + 2 \log(2K^2 H /\delta)\notag\\
  &\leq 2CH(1+\chi)(1+\hat \beta_{k}\gamma^2)  \dim_{\alpha, K}(\cF_{h}) + 4C(1+\chi) \hat \beta_{k} \sqrt{\dim_{\alpha, K}(\cF)} \sqrt{H \sum_{k=1}^K \sum_{h=1}^H (\sigma_{k,h}^2+\alpha^2)}\notag\\
  &\qquad +2\sqrt{\sum_{k = 1}^K \sum_{h' = h}^H [\VV_h (\vvalue_{k,h+1}-\vvalue_{h+1}^{\pi^k})](s_h^k, a_h^k) \log(2K^2 H /\delta)} + 2\sqrt{\log(2K^2 H /\delta)} + 2 \log(2K^2 H /\delta)
  ,\label{eq:0016}
\end{align}
where the first inequality holds due to \eqref{eq:0015}, the second inequality holds due to event $\cE_1$, the third inequality holds due to Lemma \ref{lemma:sum-bonus}, the fourth inequality holds due to Cauchy-Schwartz inequality and the last inequality holds due to $\sum_{h'=h}^H \dim_{\alpha, K}(\cF_{h'})\leq \sum_{h'=1}^H \dim_{\alpha, K}(\cF_{h'})=H\dim_{\alpha, K}(\cF)$. Furthermore, taking a summation of \eqref{eq:0016}, we have
\begin{align}
   &\sum_{k=1}^K\sum_{h=1}^H \big[\PP_h(\vvalue_{k,h+1}-\vvalue_{h+1}^{\pi^k})\big](s_h^k,a_h^k)\notag\\
   &=\sum_{k=1}^K\sum_{h=1}^H\big(\vvalue_{k,h+1}(s_{h+1}^{k})-\vvalue_{h+1}^{\pi^k}(s_{h+1}^{k})\big)\notag\\
   &\qquad+\sum_{k=1}^K\sum_{h=1}^{H}\Big(\big[\PP_h(\vvalue_{k,h+1}-\vvalue_{h+1}^{\pi^k})\big](s_h^k,a_h^k)-\big(\vvalue_{k,h+1}(s_{h+1}^{k})-\vvalue_{h+1}^{\pi^k}(s_{h+1}^{k})\big)\Big)\notag\\
   &\leq \sum_{k=1}^K\sum_{h=1}^H\big(\vvalue_{k,h+1}(s_{h+1}^{k})-\vvalue_{h+1}^{\pi^k}(s_{h+1}^{k})\big)\notag \\&\qquad+2\sqrt{\sum_{k = 1}^K \sum_{h = 1}^H [\VV_h (\vvalue_{k,h+1}-\vvalue_{h+1}^{\pi^k})](s_h^k, a_h^k) \log(2K^2 H /\delta)} + 2\sqrt{\log(2K^2 H /\delta)} + 2 \log(2K^2 H /\delta)\notag\\
   &\leq 2CH^2(1+\chi)(1+\hat \beta_{k}\gamma^2) \dim_{\alpha, K}(\cF) + 4CH(1+\chi) \hat \beta_{k} \sqrt{\dim_{\alpha, K}(\cF_h)} \sqrt{H \sum_{k=1}^K \sum_{h=1}^H (\sigma_{k,h}^2+\alpha^2)}\notag\\
  &\qquad + H \left(2\sqrt{\sum_{k = 1}^K \sum_{h = 1}^H [\VV_h (\vvalue_{k,h+1}-\vvalue_{h+1}^{\pi^k})](s_h^k, a_h^k) \log(2K^2 H /\delta)} + 2\sqrt{\log(2K^2 H /\delta)} + 2 \log(2K^2 H /\delta)\right),\notag
\end{align}
where the first inequality holds due to event $\cE_1$ and the second inequality holds due to \eqref{eq:0016}. Thus, we complete the proof of Lemma \ref{lemma:transition}.
\end{proof}

\subsubsection{Proof of Lemma \ref{lemma:transition-OP}}
\begin{proof}[Proof of Lemma \ref{lemma:transition-OP}]
   Similar to the proof of Lemma \ref{lemma:transition}, for each stage $h\in[H]$ and episode $k\in[K]$, the gap between $\vvalue_{k,h}(s_h^k)$ and $\check\vvalue_{k,h}(s_h^k)$ can be decomposed as:
   \begin{align}
   &\vvalue_{k,h}(s_h^k)-\check\vvalue_{k,h}(s_h^k) \notag\\
  &\leq \qvalue_{k,h}(s_h^k,a_h^k)-\check \qvalue_{k,h}(s_h^k,a_h^k)\notag\\
&\leq\min\Big(\hat f_{k_{\text{last}}, h}(s, a)+b_{k_{\text{last}}, h}(s, a),1\Big)-\cT_h\vvalue_{k,h+1}(s_h^k,a_h^k)\notag\\
&\qquad - \max\Big(\check f_{k_{\text{last}}, h}(s, a)-b_{k_{\text{last}}, h}(s, a),0\Big)+\cT_h\check\vvalue_{k,h+1}(s_h^k,a_h^k)\notag\\
&\qquad + \cT_h\vvalue_{k,h+1}(s_h^k,a_h^k)-\cT_h\check \vvalue_{k,h+1}(s_h^k,a_h^k)\notag\\
  &\leq \big[\PP_h(\vvalue_{k,h+1}-\check\vvalue_{k,h+1})\big](s_h^k,a_h^k)+2\cdot \min\big(\hat \beta_{k_{\text{last}}} D_{\cF_h}(z; z_{[k_{\text{last}} - 1],h}, \bar\sigma_{[k_{\text{last}} - 1],h}),1\big)\notag\\
  &\qquad + 2\cdot \min \big(b_{k_{\text{last}}, h}(s_h^k, a_h^k),1\big)
 \notag\\
 &\leq \big[\PP_h(\vvalue_{k,h+1}-\check\vvalue_{k,h+1})\big](s_h^k,a_h^k)+4C\cdot \min\big(\hat \beta_{k_{\text{last}}} D_{\cF_h}(z; z_{[k_{\text{last}} - 1],h}, \bar\sigma_{[k_{\text{last}} - 1],h}),1\big)\notag\\
 &\leq \big[\PP_h(\vvalue_{k,h+1}-\check\vvalue_{k,h+1})\big](s_h^k,a_h^k)+4C(1+\chi)\cdot  \min\big(\hat \beta_{k} D_{\cF_h}(z; z_{[k - 1],h}, \bar\sigma_{[k - 1],h}),1\big)\notag\\
  &=\vvalue_{k,h+1}(s_{h+1}^{k})-\check\vvalue_{k,h+1}(s_{h+1}^{k})+\big[\PP_h(\vvalue_{k,h+1}-\check\vvalue_{k,h+1})\big](s_h^k,a_h^k)-\big(\vvalue_{k,h+1}(s_{h+1}^{k})-\check\vvalue_{k,h+1}(s_{h+1}^{k})\big) \notag\\
  &\qquad +4C(1+\chi)\cdot \min\big(\hat \beta_{k} D_{\cF_h}(z; z_{[k - 1],h}, \bar\sigma_{[k - 1],h}),1\big),\label{eq:0017}
\end{align}
where the first and second inequalities hold due to the definition of $\check\vvalue_{k,h}(s_h^k)$ and $\vvalue_{k,h}(s_h^k)$, the third 
inequality holds due to Lemma \ref{lemma:bern-error} with $\hat {\beta}_{k_{\text{last}}}=\check \beta_{k_{\text{last}}}$, the fourth inequality holds due to $b_{k_{\text{last}}, h}(s_h^k, a_h^k)\leq C \cdot  D_{\cF_h}(z; z_{[k_{\text{last}} - 1],h}, \bar\sigma_{[k_{\text{last}} - 1],h})$ and the last inequality holds due to Lemma \ref{lemma:double}. Taking a summation of \eqref{eq:0017} over all episode $k\in[K]$ and stage $h' \ge h$, we have
\begin{align}
  &\sum_{k=1}^K\big(\vvalue_{k,h}(s_h^k)-\check\vvalue_{k,h}(s_h^k)\big)\notag\notag\\
  &\leq \sum_{k=1}^K\sum_{h'=h}^{H}\Big(\big[\PP_h(\vvalue_{k,h+1}-\check\vvalue_{k,h+1})\big](s_h^k,a_h^k)-\big(\vvalue_{k,h+1}(s_{h+1}^{k})-\check\vvalue_{k,h+1}(s_{h+1}^{k})\big)\Big)\notag\\
  &\qquad + \sum_{k=1}^K\sum_{h'=h}^{H} 4C(1+\chi)\cdot \min\big(\hat \beta_{k} D_{\cF_h}(z; z_{[k - 1],h}, \bar\sigma_{[k - 1],h}),1\big)\notag\\
  &\leq \sum_{k=1}^K\sum_{h'=h}^{H} 4C(1+\chi)\cdot \min\big(\hat \beta_{k} D_{\cF_h}(z; z_{[k - 1],h}, \bar\sigma_{[k - 1],h}),1\big)\notag \\&\qquad +2\sqrt{\sum_{k = 1}^K \sum_{h' = h}^H [\VV_h (\vvalue_{k,h+1}-\check V_{k, h + 1})](s_h^k, a_h^k) \log(2K^2 H /\delta)} + 2\sqrt{\log(2K^2 H /\delta)} + 2 \log(2K^2 H /\delta)\notag\\ 
  &\leq \sum_{h'=h}^H 4C(1+\chi)(1+\hat \beta_{k}\gamma^2)  \dim_{\alpha, K}(\cF_{h'}) + \sum_{h'=h}^H 8C(1+\chi) \hat \beta_{k} \sqrt{\dim_{\alpha, K}(\cF_{h'})} \sqrt{\sum_{k=1}^K (\sigma_{k,h'}^2+\alpha^2)}\notag\\
  &\qquad +2\sqrt{\sum_{k = 1}^K \sum_{h' = h}^H [\VV_h (\vvalue_{k,h+1}-\check V_{k, h + 1})](s_h^k, a_h^k) \log(2K^2 H /\delta)} + 2\sqrt{\log(2K^2 H /\delta)} + 2 \log(2K^2 H /\delta)\notag\\
  &\leq  4CH(1+\chi)(1+\hat \beta_{k}\gamma^2)  \dim_{\alpha, K}(\cF_{h}) + 8C(1+\chi) \hat \beta_{k} \sqrt{\sum_{h'=h}^H \dim_{\alpha, K}(\cF_{h'})} \sqrt{\sum_{k=1}^K \sum_{h'=h}^H (\sigma_{k,h'}^2+\alpha^2)}\notag\\
  &\qquad +2\sqrt{\sum_{k = 1}^K \sum_{h' = h}^H [\VV_h (\vvalue_{k,h+1}-\check V_{k, h + 1})](s_h^k, a_h^k) \log(2K^2 H /\delta)} + 2\sqrt{\log(2K^2 H /\delta)} + 2 \log(2K^2 H /\delta)\notag\\
  &\leq 4CH(1+\chi)(1+\hat \beta_{k}\gamma^2)  \dim_{\alpha, K}(\cF_{h}) + 8C(1+\chi) \hat \beta_{k} \sqrt{\dim_{\alpha, K}(\cF)} \sqrt{H \sum_{k=1}^K \sum_{h=1}^H (\sigma_{k,h}^2+\alpha^2)}\notag\\
  &\qquad +2\sqrt{\sum_{k = 1}^K \sum_{h' = h}^H [\VV_h (\vvalue_{k,h+1}-\check V_{k, h + 1})](s_h^k, a_h^k) \log(2K^2 H /\delta)} + 2\sqrt{\log(2K^2 H /\delta)} + 2 \log(2K^2 H /\delta),\label{eq:0018}
\end{align}
where the first inequality holds due to \eqref{eq:0017}, the second inequality holds due to event $\cE_2$, the third inequality holds due to Lemma \ref{lemma:sum-bonus}, the fourth inequality holds due to Cauchy-Schwartz inequality and the last inequality holds due to $\sum_{h'=h}^H \dim_{\alpha, K}(\cF_{h'})\leq \sum_{h'=1}^H \dim_{\alpha, K}(\cF_{h'})=H\dim_{\alpha, K}(\cF)$. Furthermore, taking a summation of \eqref{eq:0018}, we have
\begin{align}
   &\sum_{k=1}^K\sum_{h=1}^H \big[\PP_h(\vvalue_{k,h+1}-\check\vvalue_{k,h+1})\big](s_h^k,a_h^k)\notag\\
   &=\sum_{k=1}^K\sum_{h=1}^H\big(\vvalue_{k,h+1}(s_{h+1}^{k})-\check\vvalue_{k,h+1}(s_{h+1}^{k})\big)\notag\\
   &\qquad+\sum_{k=1}^K\sum_{h=1}^{H}\Big(\big[\PP_h(\vvalue_{k,h+1}-\check\vvalue_{k,h+1})\big](s_h^k,a_h^k)-\big(\vvalue_{k,h+1}(s_{h+1}^{k})-\check\vvalue_{k,h+1}(s_{h+1}^{k})\big)\Big)\notag\\
   &\leq \sum_{k=1}^K\sum_{h=1}^H\big(\vvalue_{k,h+1}(s_{h+1}^{k})-\check\vvalue_{k,h+1}(s_{h+1}^{k})\big)\notag \\\qquad &+ H\left(2\sqrt{\sum_{k = 1}^K \sum_{h=1}^H [\VV_h (\vvalue_{k,h+1}-\check V_{k, h + 1})](s_h^k, a_h^k) \log(2k^2 H /\delta)} + 2\sqrt{\log(2K^2 H /\delta)} + 2 \log(2K^2 H /\delta)\right)\notag\\
   &\leq 2CH^2(1+\chi)(1+\hat \beta_{k}\gamma^2)  \dim_{\alpha, K}(\cF_{h}) + 4CH(1+\chi) \hat \beta_{k} \sqrt{\dim_{\alpha, K}(\cF)} \sqrt{H \sum_{k=1}^K \sum_{h=1}^H (\sigma_{k,h}^2+\alpha^2)}\notag\\
  &\qquad +2H\left(2\sqrt{\sum_{k = 1}^K \sum_{h=1}^H [\VV_h (\vvalue_{k,h+1}-\check V_{k, h + 1})](s_h^k, a_h^k) \log(2K^2 H /\delta)} + 2\sqrt{\log(2K^2 H /\delta)} + 2 \log(2K^2 H /\delta)\right),\notag
\end{align}
where the first inequality holds due to event $\cE_2$ and the last inequality holds due to \eqref{eq:0018}. Thus, we complete the proof of Lemma \ref{lemma:transition-OP}.
\end{proof}

\subsubsection{Proof of Lemma \ref{lemma:total-variance}}
\begin{proof}[Proof of Lemma \ref{lemma:total-variance}]
According to the definition of estimated variance $\sigma_{k,h}$, the summation of variance can be decomposed as following:
   \begin{align}
       \sum_{k=1}^K\sum_{h=1}^H\sigma_{k,h}^{2} &=\sum_{k=1}^K\sum_{h=1}^H[\bar{\VV}_{k,h}\vvalue_{k,h+1}](s_h^k,a_h^k)+E_{k,h}+F_{k,h}\notag\\
   &=\underbrace{\sum_{k=1}^K\sum_{h=1}^H \big([\bar{\VV}_{k,h}\vvalue_{k,h+1}](s_h^k,a_h^k)-[{\VV}_{h}\vvalue_{k,h+1}](s_h^k,a_h^k)\big)}_{I_1}+\underbrace{\sum_{k=1}^K\sum_{h=1}^HE_{k,h}}_{I_2}+\underbrace{\sum_{k=1}^K\sum_{h=1}^H F_{k,h}}_{I_3}\notag\\
   &\qquad +\underbrace{\sum_{k=1}^K\sum_{h=1}^H \big([{\VV}_{h}\vvalue_{k,h+1}](s_h^k,a_h^k)-[{\VV}_{h}\vvalue^{\pi^k}_{h+1}](s_h^k,a_h^k)\big)}_{I_4}+\underbrace{\sum_{k=1}^K\sum_{h=1}^H [{\VV}_{h}\vvalue^{\pi^k}_{h+1}](s_h^k,a_h^k)}_{I_5}.\label{eq:0019}
   \end{align}
   For the term $I_1$, it can be upper bounded by
   \begin{align}
       I_1= \sum_{k=1}^K\sum_{h=1}^H \big([\bar{\VV}_{k,h}\vvalue_{k,h+1}](s_h^k,a_h^k)-[{\VV}_{h}\vvalue_{k,h+1}](s_h^k,a_h^k)\big) \leq \sum_{k=1}^K\sum_{h=1}^H E_{k,h},\label{eq:0020}
   \end{align}
   where the inequality holds due to Lemma \ref{lemma:variance-estimator-E}.

   For the second term $I_2=\sum_{k=1}^K\sum_{h=1}^H E_{k,h}$, we have
   \begin{align}
       \sum_{k=1}^K\sum_{h=1}^H E_{k,h}&=\sum_{k=1}^K\sum_{h=1}^H (2L{\beta}_k+\tilde \beta_k) \min\Big(1, D_{\cF_h}(z; z_{[k - 1],h}, \bar\sigma_{[k - 1],h})\Big)\notag\\
       &\leq \sum_{h=1}^H (2L{\beta}_k+\tilde \beta_k) \times (1+\gamma^2) \dim_{\alpha, K}(\cF_h) \notag\\
       &\qquad + \sum_{h=1}^H (2L{\beta}_k+\tilde \beta_k) \times 2  \sqrt{\dim_{\alpha, K}(\cF_h)} \sqrt{\sum_{k=1}^K (\sigma_{k,h}^2+\alpha^2)}\notag\\
       &\leq (2L{\beta}_k+\tilde \beta_k)H \times (1+\gamma^2) \dim_{\alpha, K}(\cF)\notag\\
       &\qquad + (2L{\beta}_k+\tilde \beta_k) \times 2  \sqrt{ \sum_{h=1}^H \dim_{\alpha, K}(\cF_h)} \sqrt{\sum_{k=1}^K  \sum_{h=1}^H(\sigma_{k,h}^2+\alpha^2)}\notag\\
       &= (2L{\beta}_k+\tilde \beta_k)H \times (1+\gamma^2) \dim_{\alpha, K}(\cF)\notag\\
       &\qquad + (2L{\beta}_k+\tilde \beta_k) \times 2  \sqrt{\dim_{\alpha, K}(\cF)} \sqrt{H\sum_{k=1}^K  \sum_{h=1}^H(\sigma_{k,h}^2+\alpha^2)}
       ,\label{eq:0021}
   \end{align}
   where the inequality holds due to Lemma \ref{lemma:sum-bonus} and the second inequality holds due to Cauchy-Schwartz inequality.

   For the term $I_3$, we have
   \begin{align}
       I_3&= \sum_{k=1}^K\sum_{h=1}^H F_{k,h}\notag\\
       &=\left(\log \cN_\cF( \epsilon) +  \log \cN_{\epsilon}(K)\right)\notag\\
   &\qquad \times  \sum_{k=1}^K\sum_{h=1}^H  \min\Big(1,2 \hat{f}_{k,h}(s_h^k,a_h^k) - 2 \check{f}_{k,h}(s_h^k,a_h^k) + 4\beta_k D_{\cF_h}(z; z_{[k - 1],h}, \bar\sigma_{[k - 1],h})\Big)\notag\\
   &\leq \left(\log \cN_\cF( \epsilon) +  \log \cN_{\epsilon}(K)\right)\notag\\
   &\qquad \times  \sum_{k=1}^K\sum_{h=1}^H  \min\Big(1,2 \cT_h V_{k,h+1}(s_h^k,a_h^k) - 2 \cT_h \check V_{k,h+1}(s_h^k,a_h^k)(s_h^k,a_h^k) + 8\beta_k D_{\cF_h}(z; z_{[k - 1],h}, \bar\sigma_{[k - 1],h})\Big)\notag\\
   &\leq \left(\log \cN_\cF( \epsilon) +  \log \cN_{\epsilon}(K)\right) \times \sum_{k=1}^K\sum_{h=1}^H \big[\PP_h(\vvalue_{k,h+1}-\check{\vvalue}_{k,h+1})\big](s_h^k,a_h^k) \notag\\
   &\qquad + \left(\log \cN_\cF( \epsilon) +  \log \cN_{\epsilon}(K)\right) \times \sum_{k=1}^K\sum_{h=1}^H \min \big(1,8\beta_k D_{\cF_h}(z; z_{[k - 1],h}, \bar\sigma_{[k - 1],h})\big)\notag\\
   &\leq \left(\log \cN_\cF( \epsilon) +  \log \cN_{\epsilon}(K)\right) \times \sum_{k=1}^K\sum_{h=1}^H \big[\PP_h(\vvalue_{k,h+1}-\check{\vvalue}_{k,h+1})\big](s_h^k,a_h^k) \notag\\
   &\qquad + \left(\log \cN_\cF( \epsilon) +  \log \cN_{\epsilon}(K)\right) \times (1+8\beta_k\gamma^2) H\dim_{\alpha, K}(\cF) \notag\\
   &\qquad + (\left(\log \cN_\cF( \epsilon) +  \log \cN_{\epsilon}(K)\right) \times 16 \beta_k \sqrt{\dim_{\alpha, K}(\cF)} \sqrt{H\sum_{k=1}^K  \sum_{h=1}^H (\sigma_{k,h}^2+\alpha^2)}\label{eq:0022.-1}
   \end{align}
   where the first inequality holds due to Lemma \ref{lemma:hoff-error}, the second
inequality holds due to $V_{k,h+1}(\cdot)\ge V_{h+1}^*(\cdot)\ge \check V_{k,h+1}(\cdot)$, the third inequality holds due to Lemma \ref{lemma:sum-bonus} with Cauchy-Schwartz inequality.

By Lemma \ref{lemma:transition-OP}, \begin{align}
&\sum_{k = 1}^K \sum_{h = 1}^H \big[\PP_h(\vvalue_{k,h+1}-\check{\vvalue}_{k,h+1})\big](s_h^k,a_h^k) \le 4CH^2(1+\chi)(1+\hat \beta_{k}\gamma^2)  \dim_{\alpha, K}(\cF_{h}) \notag\\&\qquad  + 8CH(1+\chi) \hat \beta_{k} \sqrt{\dim_{\alpha, K}(\cF)} \sqrt{H \sum_{k=1}^K \sum_{h=1}^H (\sigma_{k,h}^2+\alpha^2)}\notag\\
  &\qquad +2H\left(2\sqrt{\sum_{k = 1}^K \sum_{h=1}^H [\VV_h (\vvalue_{k,h+1}-\check V_{k, h + 1})](s_h^k, a_h^k) \log(2K^2 H /\delta)} + 2\sqrt{\log(2K^2 H /\delta)} + 2 \log(2K^2 H /\delta)\right) \notag
  \\&\le 4CH^2(1+\chi)(1+\hat \beta_{k}\gamma^2)  \dim_{\alpha, K}(\cF_{h}) + 8CH(1+\chi) \hat \beta_{k} \sqrt{\dim_{\alpha, K}(\cF)} \sqrt{H \sum_{k=1}^K \sum_{h=1}^H (\sigma_{k,h}^2+\alpha^2)} \notag \\&\qquad + \tilde{O}\bigg(H \sqrt{\sum_{k = 1}^K \sum_{h = 1}^H \big[\PP_h(\vvalue_{k,h+1}-\check{\vvalue}_{k,h+1})\big](s_h^k,a_h^k)}\bigg), \label{eq:0022.5}
\end{align}
where the last inequality follows from Lemma \ref{lemma:opt-pess}. 

Notice that for each variable $x$, $x\leq a\sqrt{x}+b$ implies $x\leq a^2+2b$, from \eqref{eq:0022.5}, we further have \begin{align} 
&\sum_{k = 1}^K \sum_{h = 1}^H \big[\PP_h(\vvalue_{k,h+1}-\check{\vvalue}_{k,h+1})\big](s_h^k,a_h^k) \notag \\& \ \le \tilde{O}\left(CH^2(1+\chi)(1+\hat \beta_{k}\gamma^2)  \dim_{\alpha, K}(\cF_{h}) + CH(1+\chi) \hat \beta_{k} \sqrt{\dim_{\alpha, K}(\cF)} \sqrt{H \sum_{k=1}^K \sum_{h=1}^H (\sigma_{k,h}^2+\alpha^2)}\right). \label{eq:0022.6}
\end{align}
Substituting \eqref{eq:0022.6} into \eqref{eq:0022.-1}, we obtain the following upper bound for $I_3$, \begin{align} 
I_3 &\le \tilde{O}\left(\left(\log \cN_\cF( \epsilon) +  \log \cN_{\epsilon}(K)\right) H \hat \beta_{k} \sqrt{\dim_{\alpha, K}(\cF)} \sqrt{H \sum_{k=1}^K \sum_{h=1}^H (\sigma_{k,h}^2+\alpha^2)}\right) \notag
\\&\quad + \tilde{O} \left(\left(\log \cN_\cF( \epsilon) +  \log \cN_{\epsilon}(K)\right) \cdot H^2(1+\hat \beta_{k}\gamma^2)  \dim_{\alpha, K}(\cF_{h})\right). \label{eq:0022}
\end{align}

   For the term $I_4$, we have
\begin{small}
   \begin{align}
       I_4&=\sum_{k=1}^K\sum_{h=1}^H \big([{\VV}_{h}\vvalue_{k,h+1}](s_h^k,a_h^k)-[{\VV}_{h}\vvalue^{\pi^k}_{h+1}](s_h^k,a_h^k)\big)\notag\\
   &=\sum_{k=1}^K\sum_{h=1}^H \Big([\PP_h \vvalue_{k,h+1}^2](s_h^k,a_h^k)-\big([\PP_h \vvalue_{k,h+1}](s_h^k,a_h^k)\big)^2
   -[\PP_h (\vvalue^{\pi^k}_{h+1})^2](s_h^k,a_h^k)+\big([\PP_h \vvalue^{\pi^k}_{h+1}](s_h^k,a_h^k)\big)^2\Big)\notag\\
   &\leq \sum_{k=1}^K\sum_{h=1}^H\big([\PP_h \vvalue_{k,h+1}^2](s_h^k,a_h^k)-[\PP_h (\vvalue^{\pi^k}_{h+1})^2](s_h^k,a_h^k)\big)\notag\\
   &\leq 2 \sum_{k=1}^K\sum_{h=1}^H\big([\PP_h \vvalue_{k,h+1}](s_h^k,a_h^k)-[\PP_h \vvalue^{\pi^k}_{h+1}](s_h^k,a_h^k)\big)\notag\\
   &\leq 8CH^2(1+\chi)(1+\hat \beta_{k}\gamma^2)  \dim_{\alpha, K}(\cF) + 16CH(1+\chi) \hat \beta_{k} \sqrt{\dim_{\alpha, K}(\cF)} \sqrt{H \sum_{k=1}^K \sum_{h=1}^H (\sigma_{k,h}^2+\alpha^2)}\notag\\
  &\qquad +2H \left(2\sqrt{\sum_{k = 1}^K \sum_{h = 1}^H [\VV_h (\vvalue_{k,h+1}-\vvalue_{h+1}^{\pi^k})](s_h^k, a_h^k) \log(2K^2 H /\delta)} + 2\sqrt{\log(2K^2 H /\delta)} + 2 \log(2K^2 H /\delta)\right), \notag
  \\&\le \tilde{O}\left(H^2(1+\hat \beta_{k}\gamma^2)  \dim_{\alpha, K}(\cF) + H \hat \beta_{k} \sqrt{\dim_{\alpha, K}(\cF)} \sqrt{H \sum_{k=1}^K \sum_{h=1}^H (\sigma_{k,h}^2+\alpha^2)}\right)\label{eq:0023}
   \end{align}
\end{small}
where the first inequality holds due to $V_{k,h+1}(\cdot)\ge V_{h+1}^*(\cdot)\ge V_{h+1}^{\pi^k}(\cdot)$, the second inequality holds due to $0\leq  V_{h+1}^*(\cdot), V_{h+1}^{\pi^k}(\cdot) \leq 1$, the third inequality holds due to Lemma \ref{lemma:transition}, and the last inequality follows from Lemma \ref{lemma:opt-pess} and the fact that $x\leq a\sqrt{x}+b$ implies $x\leq a^2+2b$. 
   
   For the term $I_5$, according to the definition of $\var_K$, we have
   \begin{align}
       I_5=\sum_{k=1}^K\sum_{h=1}^H [{\VV}_{h}\vvalue^{\pi^k}_{h+1}](s_h^k,a_h^k) = \var_K.\label{eq:0024}
   \end{align}
Substituting the results in \eqref{eq:0020}, \eqref{eq:0021}, \eqref{eq:0022}, \eqref{eq:0023} and \eqref{eq:0024} into \eqref{eq:0019}, we have
\begin{small}
\begin{align*}
   &\sum_{k=1}^K \sum_{h=1}^H \sigma_{k,h}^2\notag\\
   &=I_1+I_2+I_3+I_4+I_5\notag\\
   &\leq \tilde{O} \bigg((4L{\beta}_k+2\tilde \beta_k)H \times (1+\gamma^2) \dim_{\alpha, K}(\cF)\notag\\
       &\qquad + (4L{\beta}_k+2\tilde \beta_k) \times 2  \sqrt{\dim_{\alpha, K}(\cF)} \sqrt{H\sum_{k=1}^K  \sum_{h=1}^H(\sigma_{k,h}^2+\alpha^2)} + \var_K\notag\\
     &\qquad +  8CH^2(1+\chi)(1+\hat \beta_{k}\gamma^2)  \dim_{\alpha, K}(\cF) + 16CH(1+\chi) \hat \beta_{k} \sqrt{\dim_{\alpha, K}(\cF)} \sqrt{H \sum_{k=1}^K \sum_{h=1}^H (\sigma_{k,h}^2+\alpha^2)}\notag\\
  &\qquad + \left(\log \cN_\cF( \epsilon) +  \log \cN_{\epsilon}(K)\right) H \hat \beta_{k} \sqrt{\dim_{\alpha, K}(\cF)} \sqrt{H \sum_{k=1}^K \sum_{h=1}^H (\sigma_{k,h}^2+\alpha^2)} \notag
  \\&\qquad + \left(\log \cN_\cF( \epsilon) +  \log \cN_{\epsilon}(K)\right) \cdot H^2(1+\hat \beta_{k}\gamma^2)  \dim_{\alpha, K}(\cF_{h})\bigg) \notag\\
   &\leq \var_K +\left(\log \cN_\cF( \epsilon) +  \log \cN_{\epsilon}(K)\right) \times  \tilde O\big((1+\gamma^2)(\beta_k+ H \hat\beta_k+\tilde \beta_k) H\dim_{\alpha, K}(\cF)\big)\notag\\
   &\qquad + \left(\log \cN_\cF( \epsilon) +  \log \cN_{\epsilon}(K)\right) \times \tilde O \left((\beta_k+H \hat\beta_k+\tilde \beta_k)\sqrt{\dim_{\alpha, K}(\cF)} \sqrt{H\sum_{k=1}^K  \sum_{h=1}^H (\sigma_{k,h}^2+\alpha^2)}\right)
\end{align*}
Notice that for each variable $x$, $x\leq a\sqrt{x}+b$ implies $x\leq a^2+2b$. With this fact, we have
   \begin{align*}
     \sum_{k=1}^K \sum_{h=1}^H \sigma_{k,h}^2&\leq  \left(\log \cN_\cF( \epsilon) +  \log \cN_{\epsilon}(K)\right) \times \tilde O\big((1+\gamma^2)(\beta_k+ H \hat\beta_k+\tilde \beta_k) H\dim_{\alpha, K}(\cF)\big)\notag\\
     &\qquad +\left(\log \cN_\cF( \epsilon) +  \log \cN_{\epsilon}(K)\right)^2 \times \tilde O\big((\beta_k+ H \hat\beta_k+\tilde \beta_k)^2 H\dim_{\alpha, K}(\cF)\big)\notag\\
     &\qquad + \tilde O(\var_K+KH \alpha^2).
   \end{align*}
\end{small}
   Thus, we complete the proof of Lemma \ref{lemma:total-variance}.
\end{proof}

\section{Covering Number Argument}
\subsection{Rare Switching}
Based on the policy-updating criterion, the following lemma provides a upper bound of the switching cost.
\begin{lemma} \label{lemma:switch}
   The number of episodes when the algorithm updates the value function is at most $O\left(\dim_{\alpha, K}(\cF) \cdot  H\right)$. 
\end{lemma}

\begin{proof} 
   According to line 9, the policy is updated at episode $k$ only when there exists a stage $h \in [H]$ such that \begin{align*}
       \sum_{i \in [k_{last}, k - 1]} \frac{1}{\bar \sigma_{i, h}^2} D_{\cF_h}^2(z_{i, h}; z_{[k_{last} - 1], h}, \bar \sigma_{[k_{last} - 1], h}) \ge \chi / C. 
   \end{align*}
   and \begin{align} 
       \sum_{i\in [k_{last}, k - 2]} \frac{1}{\bar \sigma_{i, h}^2} D_\cF^2(z_{i, h}; z_{[k_{last} - 1], h}, \bar \sigma_{[k_{last} - 1], h}) < \chi. \label{eq:stab-sensitivity:0}
   \end{align}

   Then the following inequality holds, \begin{align} 
       &\sup_{f_1, f_2 \in \cF_h} \frac{\sum_{i \in [1, k - 2]} \frac{1}{\bar \sigma_{i, h}^2}(f_1(z_{i, h}) - f_2(z_{i, h}))^2 + \lambda}{\sum_{i \in [1, k_{last} - 1]} \frac{1}{\bar \sigma_{i, h}^2}(f_1(z_{i, h}) - f_2(z_{i, h}))^2 + \lambda} \\&= 1 + \sup_{f_1, f_2 \in \cF_h} \frac{\sum_{i \in [k_{last}, k - 2]} \frac{1}{\bar \sigma_{i, h}^2}(f_1(z_{i, h}) - f_2(z_{i, h}))^2}{\sum_{i \in [1, k_{last} - 1]} \frac{1}{\bar \sigma_{i, h}^2}(f_1(z_{i, h}) - f_2(z_{i, h}))^2 + \lambda} \notag
       \\&\le 1 + \sum_{i \in [k_{last}, k - 2]} \frac{1}{\bar \sigma_{i, h}^2} D_{\cF_h}^2(z_{i, h}; z_{[k_{last} - 1], h}, \bar \sigma_{[k_{last} - 1], h}) \notag
       \\&\le 1 + \chi, \label{eq:stab-sensitivity}
   \end{align}
   where the first inequality holds due to the definition of $D_{\cF_h}$ (Definition \ref{def:ged}), the second inequality follows from \eqref{eq:stab-sensitivity:0}. 
   
   \eqref{eq:stab-sensitivity} further gives a lower bound for the summation \begin{align*} 
       &\sum_{i \in [k_{last} , k - 1]} \frac{1}{\bar \sigma_{i, h}^2} D_\cF^2(z_{i, h}; z_{[i - 1], h}, \bar \sigma_{[i - 1], h}) \\&\ge \frac{1}{1 + \chi} \sum_{i \in [k_{last}, k - 1]} \frac{1}{\bar \sigma_{i, h}^2} D_\cF^2(z_{i, h}; z_{[k_{last} - 1], h}, \bar \sigma_{[k_{last} - 1], h}) \\&\ge \frac{\chi / C}{1 + \chi}. 
   \end{align*}
Note that $\frac{\chi/ C}{1 + \chi} \le 1$, we also have \begin{align*} 
       \sum_{i \in [k_{last}, k - 1]} \min \left\{1, \frac{1}{\bar \sigma_{i, h}^2} D_\cF^2(z_{i, h}; z_{[i - 1], h}, \bar \sigma_{[i - 1], h})\right\} \ge \frac{\chi / C}{1 + \chi}. 
   \end{align*}
   Then we have an upper bound and lower bound for the following summation: 
   \begin{align*} 
       l_K \cdot \frac{\chi / C}{1 + \chi} \le \sum_{k = 1}^K \sum_{h = 1}^H \min \left\{1, \frac{1}{\bar \sigma_{k, h}^2} D_\cF^2(z_{k, h}; z_{[k - 1], h}, \bar \sigma_{[k - 1], h})\right\} \le \dim_{\alpha, K}(\cF) \cdot H.
   \end{align*}
Therefore, the number of policy switching $l_K$ is of order $O(\dim_{\alpha, K}(\cF) \cdot H)$. 

\end{proof}

\begin{lemma}[Stability of uncertainty under rare switching strategy]\label{lemma:double}
If the policy is not updated at episode $k$, the uncertainty of all state-action pair $z = (s, a) \in \cS \times \cA$ and stage $h \in [H]$ satisfies the following stability property: \begin{align*} 
   D_{\cF_h}^2(z; z_{[k - 1], h}, \bar \sigma_{[k - 1], h}) \ge \frac{1}{1 + \chi} D_{\cF_h}^2(z; z_{[k_{last} - 1], h}, \bar\sigma_{[k_{last} - 1], h}). 
\end{align*}
\end{lemma}

\begin{proof}
   Due to the definition of $k_{last}$ in Algorithm \ref{algorithm1}, we have \begin{align*} 
       \sum_{i\in [k_{last}, k - 1]} \frac{1}{\bar \sigma_{i, h}^2} D_\cF^2(z_{i, h}; z_{[k_{last} - 1], h}, \bar \sigma_{[k_{last} - 1], h}) < \chi.
   \end{align*}
   As is shown in \eqref{eq:stab-sensitivity}, here we also have \begin{align*} 
   \sup_{f_1, f_2 \in \cF_h} \frac{\sum_{i \in [1, k - 1]} \frac{1}{\bar \sigma_{i, h}^2}(f_1(z_{i, h}) - f_2(z_{i, h}))^2 + \lambda}{\sum_{i \in [1, k_{last} - 1]} \frac{1}{\bar \sigma_{i, h}^2}(f_1(z_{i, h}) - f_2(z_{i, h}))^2 + \lambda} \le  1 + \chi. 
   \end{align*}
From the definition of $D_{\cF_h}$, \begin{align*} 
       D_{\cF_h}^2(z; z_{[k - 1], h}, \bar \sigma_{[k - 1], h}) \ge \frac{1}{1 + \chi} D_{\cF_h}^2(z; z_{[k_{last}-1], h}, \bar\sigma_{[k_{last}-1], h}). 
   \end{align*}
The proof is then completed due to the arbitrariness of $h$. 
\end{proof}

\subsection{Value Function Class and Its Covering Number}

The optimistic value functions at episode $k$ and stage $h \in [H]$ in our construction belong to the following function class: \begin{align} 
   \cV_{k, h} = \left\{V \bigg| \max_{a \in \cA} \min_{1 \le i \le l_k + 1} \min\left(1, f_i(\cdot, a) + \beta \cdot b(\cdot, a)\right)\right\}, \label{def:opt_value_class}
\end{align} where $l_k$ is the number of updated policies as defined in Algorithm \ref{algorithm1}, $f_i \in \cF_h$ and $b \in \cB$. 

Similarly, we also define the following pessimistic value function classes for all $k \ge 1$: 
\begin{align} 
   \check \cV_{k, h} = \left\{V \bigg| \max_{a \in \cA} \max_{1 \le i \le l_k + 1} \max\left(0, f_i(\cdot, a) - \beta \cdot b(\cdot, a)\right)\right\}, \label{def:pes_value_class}
\end{align}

\begin{lemma}[$\epsilon$-covering number of optimistic value function classes] \label{lemma:opt-cover}
For optimistic value function class $\cV_{k, h}$ defined in \eqref{def:opt_value_class}, we define the distance between two value functions $V_1$ and $V_2$ as $\|V_1 - V_2\|_\infty := \max_{s \in \cS} |V_1(s) - V_2(s)|$. Then the $\epsilon$-covering number with respect to the distance function can be upper bounded by \begin{align} 
   \cN_\epsilon(k):=  [\cN_{\cF}(\epsilon / 2) \cdot \cN(\cB, \epsilon / 2\beta)]^{l_k + 1} .\label{eq:v-cover}
\end{align}
\end{lemma}

\begin{proof} 
   By the definition of $\cN(\cF, \epsilon)$, there exists an $\epsilon / 2$-net of $\cF$, denoted by $\cC(\cF, \epsilon / 2)$, such that for any $f \in \cF$, we can find $f' \in \cC(\cF, \epsilon / 2)$ such that $\|f - f'\|_\infty \le \epsilon / 2$. 
   Also, there exists an $\epsilon / 2\beta$-net of $\cB$, $\cC(\cB, \epsilon / 2\beta)$. 

   Then we consider the following subset of $\cV_k$, \begin{align*}
       \cV^c = \left\{V \bigg| \max_{a \in \cA} \min_{1 \le i \le l_{k} + 1} \min\left(1, f_i(\cdot, a) + \beta \cdot b_i(\cdot, a)\right), f_i \in \cC(\cF_h, \epsilon / 2), b_i \in \cC(\cB, \epsilon / 2 \beta)\right\}. 
   \end{align*}
   Consider an arbitrary $V \in \cV$ where $V = \max_{a \in \cA} \min_{1 \le i \le l_k + 1} \min(1, f_i(\cdot, a) + \beta \cdot b_i(\cdot, a))$. For each $f_i$, there exists $f_i^c \in \cC(\cF, \epsilon / 2)$ such that $\|f_i - f_i^c\|_\infty \le \epsilon / 2$. There also exists $b^c \in \cC(\cB, \epsilon / 2 \beta)$ such that $\|b - b^c\|_\infty\le \epsilon / 2\beta$. Let $V^c = \max_{a \in \cA} \min_{1 \le i \le l_k + 1} \min(1, f_i^c(\cdot, a) + \beta \cdot b^c(\cdot, a)) \in \cV^c$. It is then straightforward to check that $\|V - V^c\|_\infty \le \epsilon / 2 + \beta \cdot \epsilon / 2\beta = \epsilon$. 

   By direct calculation, we have $|\cV^c| = [\cN(\cF_h, \epsilon / 2) \cdot \cN(\cB, \epsilon / 2\beta)]^{l_k + 1}$. 

\end{proof}

\begin{lemma}[$\epsilon$-covering number of pessimistic value function classes] \label{lemma:pes-cover}
For pessimistic value function class $\check \cV_{k, h}$ defined in \eqref{def:pes_value_class}, we define the distance between two value functions $V_1$ and $V_2$ as $\|V_1 - V_2\|_\infty := \max_{s \in \cS} |V_1(s) - V_2(s)|$. Then the $\epsilon$-covering number of $\check \cV_k$ with respect to the distance function can be upper bounded by 
   $\cN_\epsilon(k)$ defined in \eqref{eq:v-cover}. 
\end{lemma}
\begin{proof} 
   The proof is nearly the same as that of Lemma \ref{lemma:opt-cover}. 
\end{proof}

\section{Auxiliary Lemmas}

\begin{lemma}[Azuma-Hoeffding inequality]\label{lemma:azuma}
Let $\{x_i\}_{i=1}^n$ be a martingale difference sequence with respect to a filtration $\{\cG_{i}\}$ satisfying $|x_i| \leq M$ for some constant $M$, $x_i$ is $\cG_{i+1}$-measurable, $\EE[x_i|\cG_i] = 0$. Then for any $0<\delta<1$, with probability at least $1-\delta$, we have 
\begin{align}
   \sum_{i=1}^n x_i\leq M\sqrt{2n \log (1/\delta)}.\notag
\end{align} 
\end{lemma}

\begin{lemma}[Corollary 2, \citealt{agarwal2022vo}] \label{lemma:freedman-variant}
   Let $M>0,V>v>0$ be constants, and $\{x_i\}_{i\in[t]}$ be stochastic process adapted to a filtration $\{\cH_i\}_{i\in[t]}$. Suppose $\EE[x_i|\cH_{i-1}]=0$,  $|x_i|\le M$ and $\sum_{i\in[t]}\EE[x_i^2|\cH_{i-1}]\le V^2$ almost surely. Then for any $\delta,\epsilon>0$, let $\iota = \sqrt{\log\frac{\left(2\log(V/v)+2\right)\cdot\left(\log(M/m)+2\right)}{\delta}}$ we have
   \[
       \PP\left(\sum_{i\in[t]}x_i>\iota\sqrt{2\left(2\sum_{i\in[t]}\EE[x_i^2|\cH_{i-1}]+ v^2\right)}+\frac{2}{3}\iota^2\left(2\max_{i\in[t]}|x_i|+m\right) \right)\le \delta.
   \]
\end{lemma}

\begin{lemma} [Lemma 7, \citealt{russo2014learning}] \label{lemma:concentration-measure}
Consider random variables $\left(Z_{n} | n\in\mathbb{N}\right)$ adapted
to the filtration $\left(\mathcal{H}_{n}:\, n=0,1,...\right)$. Assume
$\mathbb{E}\left[\exp\left\{ \lambda Z_{i}\right\} \right]$ is finite
for all $\lambda$. Define the conditional mean 
$\mu_{i}=\mathbb{E}\left[Z_{i}\mid\mathcal{H}_{i-1}\right]$.
We define the conditional cumulant generating function of the centered
random variable $\left[Z_{i}-\mu_{i}\right]$ by 
$\psi_{i}\left(\lambda\right)=\log\mathbb{E}\left[\exp\left(\lambda\left[Z_{i}-\mu_{i}\right]\right)\mid\mathcal{H}_{i-1}\right]$. 
   For all $x\geq0$ and $\lambda\geq0$, \[ \mathbb{P}\left(\sum_{1}^{n}\lambda Z_{i}\leq x+\sum_{1}^{n}\left[\lambda\mu_{i}+\psi_{i}\left(\lambda\right)\right]\,\,\,\,\mbox{\ensuremath{\forall}}n\in\mathbb{N}\right)\geq1-e^{-x}. \] 
\end{lemma}

\begin{lemma}[Self-normalized bound for scalar-valued martingales] \label{lemma:hoeffding-variant}
   Consider random variables $\left(v_{n} | n\in\mathbb{N}\right)$ adapted
   to the filtration $\left(\mathcal{H}_{n}:\, n=0,1,...\right)$. Let $\{\eta_i\}_{i = 1}^ \infty$ be a sequence of real-valued random variables which is $\cH_{i + 1}$-measurable and is conditionally $\sigma$-sub-Gaussian. Then for an arbitrarily chosen $\lambda > 0$, for any $\delta > 0$, with probability at least $1 - \delta$, it holds that \begin{align*} 
       \sum_{i = 1}^n \epsilon_i v_i \le \frac{\lambda \sigma^2}{2} \cdot \sum_{i = 1}^n v_i^2 + \log(1 / \delta) / \lambda \ \ \ \ \ \ \ \forall n \in \NN. 
   \end{align*}
\end{lemma}

\begin{lemma}[Lemma 10, \citealt{zhang2020almost}] \label{lemma:freedman} Let $(M_n)_{n \geq 0}$ be a martingale such that $M_0 = 0$ and $|M_n - M_{n-1}| \leq c$ for some $c > 0$ and any $n \geq 1$. Let $\text{Var}_n = \sum_{k=1}^{n} E[(M_k - M_{k-1})^2 | \mathcal{F}_{k-1}]$ for $n \geq 0$, where $\mathcal{F}_k = \sigma(M_1, M_2, \dots, M_k)$. Then for any positive integer $n$ and any $\varepsilon, p > 0$, we have that
\[
\mathbb{P}\left( |M_n| \geq 2\sqrt{\text{Var}_n \log\left(\frac{1}{p}\right)} + 2 \sqrt{\varepsilon \log\left(\frac{1}{p}\right)} + 2c \log\left(\frac{1}{p}\right) \right) \leq \left(2n c^2 / \varepsilon + 2 \right) p
\]
\end{lemma}

\end{document}